\documentclass{article}

\PassOptionsToPackage{numbers}{natbib}
\usepackage[main, final]{style/neurips_2026}

\usepackage[utf8]{inputenc} 
\usepackage[T1]{fontenc}    
\usepackage{hyperref}       
\usepackage{url}            
\usepackage{booktabs}       
\usepackage{amsfonts}       
\usepackage{nicefrac}       
\usepackage{microtype}      
\usepackage{xcolor}         

\usepackage{style/custom}

\title{
An Asymptotic Theory of Chain-of-Thought in In-Context Learning}

\author{%
  Kaito~Takanami \\
  Department of Physics, Graduate School of Science, The University of Tokyo\\
  Tokyo, Japan \\
  John A. Paulson School of Engineering and Applied Sciences, Harvard University\\
  Cambridge, MA, USA \\
  \texttt{takanami255@g.ecc.u-tokyo.ac.jp} \\
  \And
  Cengiz Pehlevan \\
   John A. Paulson School of Engineering and Applied Sciences, Harvard University\\
   Kempner Institute for the Study of Natural and Artificial Intelligence, Harvard University\\
   Center for Brain Science, Harvard University \\
   Cambridge, MA, USA 
}

\begin{document}

\maketitle

\begin{abstract}
  Chain-of-thought (CoT) reasoning has become a widely used mechanism for eliciting multi-step reasoning in large language models by generating intermediate reasoning steps at inference time. 
Yet the scaling behavior of generalization with CoT depth remains poorly understood.
To address this question, we study a theoretically solvable model of CoT for in-context weight prediction in linear regression, where test-time reasoning is represented as an iterative refinement of the weight-parameter estimate. 
Using tools from random matrix theory under high-dimensional asymptotics, we derive an exact formula for the generalization error as a function of reasoning depth, pretraining data amount, and context length. 
Our analysis reveals a sharp phase transition separating exponential and polynomial improvement, saturation, and overthinking, and characterizes how the optimal reasoning depth scales. 
We further show that deeper reasoning is most effective with sufficiently rich pretraining and in-context information, whereas limited pretraining or context makes longer reasoning prone to error amplification or saturation.
We also validate these predictions through experiments on fully learned linear attention and softmax attention models. 
Our results provide a unified theoretical account of how test-time CoT depth affects generalization.
\end{abstract}

\section{Introduction}
\label{section:introduction}

Chain-of-thought (CoT) reasoning refers to a test-time procedure in which a model generates intermediate reasoning steps before producing its final answer~\citep{wei2022chain}. 
In large language models (LLMs), CoT has become a widely used prompting strategy because it often improves performance on tasks that require multi-step reasoning beyond direct one-shot prediction~\citep{wei2022chain, NEURIPS2022_8bb0d291, snell2025scaling}. 
From a practical perspective, an important advantage of CoT is that it enables performance gains through test-time computation, which is increasingly valuable as continued scaling through pretraining alone faces growing constraints in both data and training resources~\citep{villalobos2024position, cottier2024rising}.

Despite its empirical effectiveness, how test-time CoT depth affects generalization remains poorly understood. 
In particular, recent empirical studies have shown that increasing the number of reasoning steps does not lead to uniformly monotone gains: depending on the setting, deeper reasoning may improve performance, saturate, or even amplify errors through overthinking~\citep{yang2025towards,liu2025mind}.
A central open question is what governs these qualitatively different regimes, and in particular how they depend on the properties of the data seen during pretraining and on the structure of the information available at test time. 
While empirical studies can reveal these phenomena, they cannot by themselves identify the principles that govern them, which motivates the study of a theoretically tractable model in which test-time reasoning dynamics can be analyzed explicitly and related to generalization.

To address these questions, we study a theoretically tractable model of CoT in an in-context learning (ICL) problem for linear regression, where the task is to predict the underlying regression weight from a sequence of contextual examples. 
This setting provides a minimal framework that still captures the essential structure needed to analyze the effect of test-time reasoning depth~\citep{huang2025transformers,javanmard2026understanding,javanmard2026theoreticalperspectivesdataquality}. 
During pretraining, the model directly predicts the regression weight from the context. At inference time, CoT is modeled as a sequence of intermediate updates that progressively refine the weight estimate before the final prediction. 
This formulation enables us to derive a closed-form expression for the generalization error in the high-dimensional limit and to analyze test-time reasoning dynamics explicitly.

Our main results are summarized as follows:
\begin{itemize}
  \item We introduce an in-context weight prediction model for pretraining and test-time CoT, and derive an exact characterization of its test-time error dynamics in a controlled high-dimensional limit, where the task-parameter dimension $D$, number of pretraining tasks $M$, and number of in-context examples $L$ go to infinity simultaneously with fixed ratios~(Section~\ref{section:result_asympt}).
  \item We identify a phase-transition structure in test-time CoT across data regimes, separating four regimes: an exponential-improvement regime, a polynomial-improvement regime, a saturation regime, and an overthinking regime in which deeper reasoning amplifies error~(Section~\ref{section:result_phase_transition}).
  \item The phase transition further reveals how pretraining task diversity and in-context examples jointly govern test-time scaling. Pretraining task diversity determines the stability of iterative refinement and the optimal reasoning depth, while in-context examples control the error decay rate and the asymptotic information limit at test time~(Sections~\ref{subsection:exponential_polynomial_scaling_regime}--\ref{subsection:saturation_regime}).
  \item We validate these theoretical predictions in nonlinear softmax-attention models and find that the same qualitative behavior persists beyond the solvable linear setting~(Section~\ref{section:experiments}).
\end{itemize}

These results provide unified theoretical insights into how test-time CoT depth affects generalization, and yield principled understandings for when deeper reasoning helps, saturates, 
or becomes harmful\footnote{The code for reproducing the results is available at \url{https://github.com/taka255/cot_asymptotics}.}.

\paragraph{Impact statement.}
This work is theoretical and studies simple attention models in a synthetic setting. 
Its potential positive impact is a better understanding of test-time reasoning, and we do not identify direct negative societal impacts.

\section{Related Works}

\paragraph{CoT and test-time reasoning in LLMs.}
Recent work has brought increasing attention to test-time scaling in LLMs. 
A central reason is that additional inference-time computation can improve reasoning capabilities that are not easily unlocked by scaling pretraining alone~\citep{wei2022chain, snell2025scaling, muennighoff2025s1}. 
This practical importance is reflected in recent frontier reasoning-oriented systems, which have adopted test-time reasoning as a core design principle rather than a purely prompt-level heuristic~\citep{OpenAIOA,guo2025deepseek}. 
At the same time, empirical studies have shown that increasing reasoning depth is not uniformly beneficial: longer reasoning can yield diminishing returns and saturation~\citep{wang2025scalingscalingexploringtesttime}, or even harm performance through overthinking~\citep{su2025underthinkingoverthinkingempiricalstudy, yang2025towards,liu2025mind, hassid2026dont}. 
However, these works are largely empirical and do not provide a general account of what governs the transition between improvement, saturation, and harmful overthinking. 
Moreover, much of the existing literature emphasizes problem difficulty, while paying comparatively less attention to the role of data properties. 
More broadly, theoretical understanding of test-time scaling remains relatively limited~\citep{halder2025demystifying,javanmard2026understanding,schaeffer2025large, huang2025best, levi2025simple,levi2026learning}. 
Motivated by this gap, we theoretically investigate how pretraining quality and in-context information govern these regimes and their scaling laws.\looseness=-1

\paragraph{Theoretical studies of CoT and ICL.}
Theoretical studies of in-context learning have long used linear regression as a tractable benchmark for understanding how transformers can implement learning algorithms~\citep{garg2022what, akyurek2023what, von2023transformers, zhang2025training}. 
One line of development has pushed this framework further through asymptotic analyses that make exact characterization possible, revealing richer phenomena such as double descent and phase transitions associated with the emergence of genuine in-context generalization beyond memorization~\citep{lu2025asymptotic,letey2026pretraintest, nguyen2025differential, bordelon2026theory, takanami2026learning}. 
A separate line of work has connected this broader ICL setting to CoT through in-context weight prediction models for linear regression, where transformers iteratively refine task-parameter estimates without learning~\citep{huang2025transformers,javanmard2026understanding,javanmard2026theoreticalperspectivesdataquality}. 
In particular, \citep{javanmard2026understanding} theoretically showed that increasing test-time compute can reduce the amount of contextual information required during training, but can also hurt performance when the training data fails to represent the input directions that are important for the downstream task.
Our work builds on both lines of research in the finite-data regime, where overthinking can occur even under identical training and inference distributions.
This allows us to view the qualitatively different effects of test-time reasoning depth as phase transitions, and to identify when deeper reasoning improves generalization, when it saturates, and when it becomes harmful through overthinking. 
We further show how these transitions are shaped by pretraining quality and in-context information.

\section{Model}
\label{section:model}
In this section, we introduce the model used in our analysis. 
We study a variant of in-context linear regression model~\citep{garg2022what}.
While standard in-context linear regression aims to predict the output for a new query input given a prompt of labeled examples, 
our objective is to infer the regression weight itself, as proposed in \citep{huang2025transformers,javanmard2026theoreticalperspectivesdataquality, javanmard2026understanding}.
During pretraining, the model learns an update rule from prompts of contextual examples. 
At inference time, CoT is generated by iteratively applying this learned update rule to a new task.

\paragraph{Training data.}
We consider \(M\) linear regression tasks indexed by \(\mu = 1,\dots,M\).
For each task \(\mu\), the training examples \(\{(\vb{x}_{\mu,l}, y_{\mu,l})\}_{l=1}^L\) are generated as
\begin{align}
    \vb{x}_{\mu,l} \sim \mathcal{N}\ab(\vb{0}, \frac{I_D}{D}), \quad
    y_{\mu,l} = \vb{w}_\mu^\top \vb{x}_{\mu,l} + \epsilon_{\mu,l}, \quad
    \epsilon_{\mu,l} \sim \mathcal{N}(0,\sigma_\mu^2),
\end{align}
where $L$ is the length of the context, and \(\vb{w}_\mu \in \mathbb{R}^D\) denotes the task-specific ground-truth parameter vector, sampled from a Gaussian normal distribution.

\paragraph{Prompt construction.}
For each task \(\mu\), we form a prompt matrix from the labeled examples together with a dedicated slot that stores the current estimate of the task parameter.
Let
\begin{align}
    \mathbf{X}_\mu = [\vb{x}_{\mu,1}, \dots, \vb{x}_{\mu,L}] \in \mathbb{R}^{D \times L},
    \qquad
    \vb{y}_\mu = [y_{\mu,1}, \dots, y_{\mu,L}]^\top \in \mathbb{R}^{L}.
\end{align}
We define the embedding matrix as
\begin{align}
    \mathbf{E}_\mu
    =
    \begin{bmatrix}
        \mathbf{X}_\mu & \vb{0}_{D} \\
        \vb{y}_\mu^\top & 0 \\
        0_{D\times L} & \hat{\vb{w}}_{\mu}^{\text{init}} \\
        0_{1\times L} & 1
    \end{bmatrix}
    \in \mathbb{R}^{(2D+2)\times (L+1)}.
\end{align}
The first $L$ columns encode the labeled examples, while the last column functions as a slot for the predicted task parameter. 
It is initialized with \(\hat{\vb{w}}_{\mu}^{\text{init}}\) and updated during reasoning, unless otherwise specified, we take \(\hat{\vb{w}}_{\mu}^{\text{init}}=\vb{0}\).
The final row serves as a marker that distinguishes the example columns, which take value \(0\), from the parameter-estimate column, which takes value \(1\).

\paragraph{Architecture.}
We consider a single-layer linear self-attention model with residual connection. For an input matrix
\(\mathbf{E} \in \mathbb{R}^{(2D+2)\times T}\), the layer is defined as
\begin{align}
    \text{\textbf{Atten}}(\mathbf{E}; \mathbf{W}_K, \mathbf{W}_Q, \mathbf{W}_V)
    = \mathbf{E} + \frac{1}{L} (\mathbf{W}_V \mathbf{E})(\mathbf{W}_K \mathbf{E})^\top (\mathbf{W}_Q \mathbf{E}),
\end{align}
where \(\mathbf{W}_K, \mathbf{W}_Q, \mathbf{W}_V \in \mathbb{R}^{(2D+2)\times(2D+2)}\).
Defining \(\mathbf{W} := \mathbf{W}_K^\top \mathbf{W}_Q\), \(\mathbf{V} := \mathbf{W}_V\), and \(\theta := (\mathbf{W},\mathbf{V})\), we equivalently write
\begin{align}
    \textbf{Atten}(\mathbf{E};\theta)
    =
    \mathbf{E} + \frac{1}{L} \mathbf{V} \mathbf{E} \mathbf{E}^\top \mathbf{W} \mathbf{E}.
\end{align}
The readout is taken from the weight slot of the final column, which is given by
\begin{align}
    \textbf{Read}(\mathbf{E}) = \mathbf{E}_{(D+2):(2D+1),T} \in \mathbb{R}^{D}.
\end{align}
The resulting estimate of the task parameter is thus given by \(\textbf{Read}(\textbf{Atten}(\mathbf{E};\theta))\).

\paragraph{Simplified parameterization.}
To obtain a tractable model, we retain only the blocks that directly control the update of the parameter-estimate token.
Accordingly, we use the block parameterization
\begin{align}
\mathbf{W}=
\begin{pmatrix}
  0_{D\times D} & 0_{D\times 1} & I_{D\times D} & 0_{D\times 1}\\
  0_{1\times D} & 0_{1\times 1} & 0_{1\times D} & -1\\
  0_{D\times D} & 0_{D\times 1} & 0_{D\times D} & 0_{D\times 1}\\
  0_{1\times D} & 0_{1\times 1} & 0_{1\times D} & 0_{1\times 1}
\end{pmatrix},
\qquad
\mathbf{V}=
\begin{pmatrix}
  0_{D\times D} & 0_{D\times 1} & 0_{D\times D} & 0_{D\times 1}\\
  0_{1\times D} & 0_{1\times 1} & 0_{1\times D} & 0_{1\times 1}\\
  -D \mathbf{A} & 0_{D\times 1} & 0_{D\times D} & 0_{D\times 1}\\
  0_{1\times D} & 0_{1\times 1} & 0_{1\times D} & 0_{1\times 1}
\end{pmatrix}.
\end{align}
Here \(\mathbf{A}\in \mathbb{R}^{D\times D}\) is the only learnable parameter; all blocks of \(\mathbf{W}\) and all other blocks of \(\mathbf{V}\) are fixed.
In both \(\mathbf{W}\) and \(\mathbf{V}\), the row and column indices are partitioned into four consecutive blocks of sizes \(D,1,D,1\).
This restriction is motivated by prior gradient-flow analyses of related in-context weight-prediction settings~\citep{huang2025transformers,javanmard2026understanding}. 
These works show that, under suitable initializations, gradient flow converges to solutions with the same sparse block structure.
Consistent with this motivation, our full-parameter training experiments suggest that these fixed blocks are learned close to the prescribed values or remain negligible, while the retained block \(\mathbf{A}\) governs the nontrivial update of the parameter estimate at inference time.
See Appendix~\ref{appendix_justification_of_parameter_simplification} for empirical support for this simplification.

Under this parameterization, the attention layer maps the embedding matrix to
\begin{align}
    \textbf{Read}(\textbf{Atten}(\mathbf{E}_\mu; \theta))
    =
    \hat{\vb{w}}_{\mu}^{\text{init}}
    -
    \frac{D}{L}\mathbf{A}
    \left(
      \mathbf{X}_\mu \mathbf{X}_\mu^\top \hat{\vb{w}}_{\mu}^{\text{init}} - \mathbf{X}_\mu \vb{y}_\mu
    \right). \label{eq:readout}
\end{align}
Thus, the readout is an updated estimate of the task parameter obtained by a linear transformation of the empirical gradient. 
Detailed calculations and discussions are given in Appendix~\ref{appendix:readout}.

\paragraph{Training objective.}
We learn \(\mathbf{A}\) by ridge-regularized empirical risk minimization so that the readout matches the ground-truth task parameter.
Under the above simplification, the empirical objective is
\begin{align}
  \mathcal{L}(\mathbf{A}) &:=
  \frac{1}{M}\sum_{\mu=1}^M
  \ab\|
    {\vb{w}}_{\mu} -
    \textbf{Read}(\textbf{Atten}(\mathbf{E}_\mu; \theta))
  \|^2
  + \frac{\lambda}{2}\|\mathbf{A}\|_F^2, \label{eq:loss}
\end{align}
where \(\lambda \ge 0\) is a ridge regularization parameter. We denote the minimizer of Eq.~\eqref{eq:loss} by \(\mathbf{A}^*\).

\paragraph{CoT in inference phase.}
At inference time, we consider a new linear regression task generated from the same distribution as in pre-training.
Specifically, the task parameter \(\vb{w} \in \mathbb{R}^D\) denotes the ground-truth parameter vector, and the examples \(\{(\vb{x}_l,y_l)\}_{l=1}^L\) are generated without noise as
$\vb{x}_l \sim \mathcal{N}(\vb{0}, I_D/D)$ and $y_l = \vb{w}^\top \vb{x}_l$ , for \(l=1,\dots,L\).
Let \(\mathbf{X} = [\vb{x}_1,\dots,\vb{x}_L] \in \mathbb{R}^{D\times L}\) and \(\vb{y} = [y_1,\dots,y_L]^\top \in \mathbb{R}^L\).
We then generate a CoT trajectory by repeatedly constructing a prompt from the observed examples and the current sequence of parameter estimates, and applying the learned attention layer.
At step \(t\), we define
\(\hat{\mathbf{W}}_t = [\hat{\vb{w}}_0,\hat{\vb{w}}_1,\dots,\hat{\vb{w}}_t] \in \mathbb{R}^{D\times (t+1)}\),
and construct the prompt
\begin{align}
  \mathbf{P}_t
  =
  \begin{bmatrix}
      \mathbf{X} & 0_{D\times (t+1)} \\
      \vb{y}^\top & 0_{1\times (t+1)} \\
      0_{D\times L} & \hat{\mathbf{W}}_t \\
      0_{1\times L} & \mathbf{1}_{1\times (t+1)}
  \end{bmatrix}
  \in \mathbb{R}^{(2D+2)\times (L+t+1)}.
\end{align}
Starting from \(\hat{\vb{w}}_0=\vb{0}\), we define the inference-time CoT trajectory by \(\hat{\vb{w}}_{t+1} := \textbf{Read}(\textbf{Atten}(\mathbf{P}_t;\theta^*))\),
where \(\theta^*\) denotes the parameter pair \((\mathbf{W},\mathbf{V})\) with \(\mathbf{A}=\mathbf{A}^*\).
Here, $t$ can be interpreted as the amount of test-time computation, since the computation required at inference is proportional to $t$.

At each step, the new parameter-estimate token is explicitly generated and appended to the prompt, where it conditions the next step.
This captures the sequential generation of and conditioning on intermediate tokens central to CoT, while representing each token as a parameter estimate rather than a natural-language reasoning step.
Our formulation is therefore related to scratchpad reasoning~\citep{nye2022show} and recurrent-depth or looped Transformers~\citep{geiping2025scaling, prairie2026parcaescalinglawsstable, kohli2026loopthinkgeneralize}, but retains the intermediate estimates as explicit tokens in the context rather than only updating hidden states.

\paragraph{Evaluation.}
To study how the CoT depth at inference time \(t\) affects the final prediction performance, we evaluate the parameter estimation error on a new task.
We then define the mean squared error (MSE) by $\mathcal{E}_{t} := \mathbb{E} [\| \vb{w} - \hat{\vb{w}}_t \|^2 ] / D$,
where the expectation is taken over the pre-training and inference data distributions.

\paragraph{High-dimensional limit.}
We consider the high-dimensional limit in which $D,L,M \to \infty$ with the ratios 
\begin{equation}
    \frac{L}{D} \to \alpha, \quad  \frac{M}{D} \to \tau,
\end{equation}
where $\alpha,\tau \in (0,\infty)$ are fixed constants.
Here, $\alpha$ represents the number of in-context examples per task (context length), and $\tau$ represents the amount of task diversity.
The advantage of this regime is that the generalization error can be characterized quantitatively in terms of only a few macroscopic system parameters, with sample-to-sample fluctuations becoming asymptotically negligible.

\section{Precise characterization of test-time CoT dynamics}
\label{section:result_asympt}

Our goal is to characterize the generalization error at a finite CoT depth \(t\) in the high-dimensional limit. 
The key observation is that the test-time CoT iteration induces a linear recursion for the estimation error (see Appendix \ref{appendix:main_result_reduction}), which can be written as
\begin{align}
  \vb{w}-\hat{\vb{w}}_t
  =
  \left(\mathbf{I}-\frac{1}{\alpha}\mathbf{A}^* \mathbf{S}\right)^t \vb{w},
\end{align}
where \(\mathbf{S}=\mathbf{X}\mathbf{X}^\top\) is the inference covariance matrix and \(\mathbf{A}^*\) is the learned update matrix.

Operationally, this means that CoT behaves like a learned iterative algorithm: pretraining determines the update matrix \(\mathbf{A}\), and inference applies this update repeatedly to the prompt examples so as to progressively refine the estimate of the underlying task parameter~\citep{javanmard2026understanding}. 
This representation implies that the MSE depends only on the matrix product generated by the CoT dynamics. 
In particular, the MSE is given by
\begin{align}
  \mathcal E_t
  = \frac{1}{D}\E\Tr\left[
    \left(\mathbf{I}-\frac{1}{\alpha}\mathbf{A}^*\mathbf{S}\right)^t
    \left(\mathbf{I}-\frac{1}{\alpha}\mathbf{S}{\mathbf{A}^*}^\top\right)^t
  \right] .
  \label{eq:generalization_trace_maintext}
\end{align}
Thus, evaluating the error reduces to understanding a family of mixed moments of \(\mathbf{A}^*\mathbf{S}\) and \(\mathbf{S}{\mathbf{A}^*}^\top\).
Since these two matrices generally do not commute, their mixed moments cannot be recovered from a one-point resolvent, which motivates the use of a two-point generating function.
To organize these moments, we introduce the two-point generating function
\begin{align}
  F(u,v)
  := \frac{1}{D}\E\Tr\ab[
    (\mathbf{I}-u\mathbf{A}^*\mathbf{S})^{-1}(\mathbf{I}-v\mathbf{S}{\mathbf{A}^*}^\top)^{-1}] .
\end{align}
A coefficient comparison shows that \(\mathcal E_t\) is obtained from the derivatives of \(F(u,v)\) at \(u=v=0\). 
Therefore, the problem of evaluating the generalization error is reduced to computing the large-\(D\) limit of the scalar function \(F(u,v)\).

In the proportional high-dimensional limit, the generating function \(F(u,v)\) with randomness converges to a deterministic limit, which we denote again by \(F(u,v)\). 
Combined with the reduction above, this implies that the finite-depth generalization error is asymptotically determined by the derivatives of this limiting function at \(u=v=0\). 
The following result makes this characterization explicit and, in the case \(\lambda=0\), gives a closed-form expression for \(F(u,v)\).
\begin{restatable}{result}{resultmain}
  \label{result:main}
  The generalization error at CoT depth $t$ is asymptotically given by
  \begin{align}
    \mathcal E_{t} &= 
      \sum_{p=0}^t\sum_{q=0}^t
      \binom{t}{p}\binom{t}{q}
      \Bigl(-\frac1\alpha\Bigr)^{p+q}
      \frac{1}{p!\,q!}\,
      \pdv[p]{}{u}\pdv[q]{}{v} F(u,v)\Big|_{u=v=0}.
  \end{align}
  Here, especially when $\lambda = 0$, $F(u,v)$ is explicitly given by
  \begin{align}
    F(u,v) &=
    \frac{\alpha\,g(u)\,g(v)}
    {\alpha-(g(u)-1)(g(v)-1)
    \ab[1+\frac{1+\sigma^2}{\alpha \ab(\tau-1)} (\alpha+g(u)+g(v) - 1)]}, 
  \end{align}
  where 
  $g(u)$ is the resolvent function of Wishart matrix, which is given by
  \begin{align}
    g(u) = \frac{c+u(1-\alpha)-\sqrt{\bigl(c+u(1-\alpha)\bigr)^2-4cu}}{2u}.
  \end{align}
  with $c=1+ (1+\sigma^2) / \alpha$.
\end{restatable}

Technically, the idea to reduce the finite-time error to a two-point correlation function is recently discussed in~\citep{Atanasov2025-wx} for the case of stochastic gradient dynamics.
The present setting is more delicate because the propagation operator \(\mathbf{A}^*\) is itself learned and generally non-symmetric, so the resulting two-point object does not reduce to the simpler forms analyzed in prior work.
To evaluate this two-point generating function, we combine the generating-function reduction with additional linearization~\citep{helton2018applications} and multi-source cavity techniques~\citep{clark2025simplified} in random matrix theory.
The detailed derivation, including the case \(\lambda>0\), is given in Appendix~\ref{appendix:main_result}.
Also, its validity is supported by the good agreement with numerical experiments reported in Appendix~\ref{appendix:theory_vs_experiment}.

\section{Phase transition of test-time scaling law}
\label{section:result_phase_transition}
In this section, we extract the asymptotic test-time scaling law from the finite-depth characterization of the generalization error using Result~\ref{result:main}. 
As \(t\to \infty\), the behavior of \(\mathcal E_t\) is controlled by two quantities: the limiting error \(\mathcal E_\infty\) and the exponential rate \(\Lambda(\alpha,\tau,\sigma^2)\). 
Their interaction yields a sharp phase transition, separating regimes of exponential improvement, polynomial improvement, information-limited saturation, and overthinking. 
More precisely, we obtain the following asymptotic test-time scaling law.

\begin{restatable}{result}{resultscalinganalysis}
  \label{result:scaling_analysis}
  Assume $\tau>1$.
  Asymptotically as \(t\to\infty\), the test-time scaling law of the generalization error obeys
  \begin{align}
    \mathcal E_t-\mathcal E_\infty
    \sim
    K\, t^{-1/2}\,\Lambda(\alpha,\tau,\sigma^2)^t,
  \end{align}
  where $K$ is a positive constant and the limiting value $\mathcal E_\infty$ is given by
  \begin{align}
    \mathcal E_\infty
    =
    \begin{cases}
      \displaystyle
      \frac{(1-\alpha)(\tau-1)}{\tau-2-\sigma^2},
      & 0<\alpha<1 \text{ and } \tau \geq \tau_c(\alpha,\sigma^2) \\
      0, & \text{otherwise},
    \end{cases}
  \end{align}
  and the exponential rate satisfies
  \begin{align}
    \Lambda(\alpha,\tau,\sigma^2)
    \begin{cases}
      >1, & \tau<\tau_c(\alpha,\sigma^2),\\
      =1, & \tau=\tau_c(\alpha,\sigma^2),\\
      <1, & \tau>\tau_c(\alpha,\sigma^2).
    \end{cases}
  \end{align}
  Here, the critical value $\tau_c(\alpha,\sigma^2)$ is given by
  \begin{align}
    \tau_c(\alpha,\sigma^2) &=
    1+
    \frac{(1+\sigma^2)(\alpha+1+2\sigma^2-\sqrt{\Delta})
    (2\alpha+2+2\sigma^2-\sqrt{\Delta})}
    {2\alpha\sqrt{\Delta}},
  \end{align}
  with $\Delta := (\alpha+1+2\sigma^2)^2-4\alpha$.
\end{restatable}

The detailed derivation is given in Appendix~\ref{appendix:scaling_analysis}.
The scaling law is derived under a condition on the relevant singularity that remains to be justified rigorously through multivariate singularity analysis, and agrees closely with numerical results in Appendix~\ref{appendix:theory_vs_experiment}.

Result~\ref{result:scaling_analysis} shows that the large-\(t\) behavior of test-time CoT is controlled by two quantities: 
the exponential factor \(\Lambda(\alpha,\tau,\sigma^2)\), which determines whether the finite-depth error grows or decays, and the limiting error floor \(\mathcal E_\infty\), which determines whether perfect asymptotic recovery is possible. 
Their combination yields four qualitatively distinct asymptotic regimes. When \(\Lambda>1\), deeper reasoning amplifies the error, giving the overthinking regime. At the critical boundary \(\Lambda=1\), the exponential improvement disappears and the error decays only polynomially in depth. 
When \(\Lambda<1\), the error decays exponentially fast; this corresponds to the exponential-improvement regime if \(\mathcal E_\infty=0\), and to the saturation regime if \(\mathcal E_\infty>0\), where the dynamics remains stable but converges to a nonzero error floor.

Figure~\ref{fig:phase_diagram}A summarizes this phase structure. It shows that the effect of increasing test-time CoT depth is not governed by a single universal scaling law, but instead changes qualitatively across data regimes. 
Figure~\ref{fig:phase_diagram}B and Figure~\ref{fig:phase_diagram}C show representative trajectories of \(\mathcal E_t\) across these regimes, illustrating how the finite-depth dynamics transition between error amplification, polynomial decay, exponential decay to zero, and exponential decay toward a nonzero limit.
We now discuss the structure and behavior of these regimes in detail.

\begin{figure}[h]
    \centering
    \includegraphics[width=1.0\textwidth]{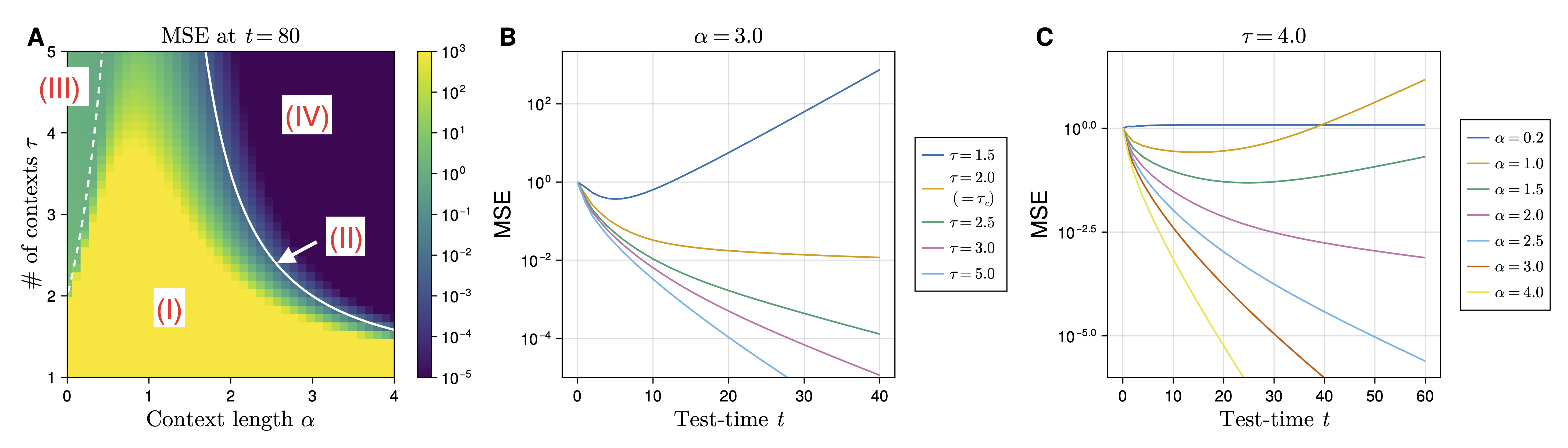}
    \caption{
      \textbf{Phase diagram of the test-time scaling law and representative error dynamics.}
      (A) Heatmap of theoretical prediction of the generalization error (MSE) at a fixed test-time depth \(t=80\). 
      The diagram is divided into four regimes:
      \textbf{(I)} the \emph{overthinking regime}, \(\tau<\tau_c(\alpha,\sigma^2)\), where long test-time CoT amplifies the error;
      \textbf{(II)} the \emph{polynomial-decay regime}, $\alpha>1$ and \(\tau=\tau_c(\alpha,\sigma^2)\), where the error decreases only polynomially;
      \textbf{(III)} the \emph{saturation regime}, \(0<\alpha<1\) and \(\tau>\tau_c(\alpha,\sigma^2)\), where the dynamics is stable but the error converges to a nonzero limit;
      and
      \textbf{(IV)} the \emph{exponential-decay regime}, \(\alpha>1\) and \(\tau>\tau_c(\alpha,\sigma^2)\), where the error decays exponentially to zero.
      (B,C) Theoretical prediction of the generalization error \(\mathcal{E}_t\) dynamics for (B) fixed $\alpha$ and (C) fixed $\tau$.
      (A-C) Parameters: $\lambda= 10^{-5}, \sigma^2=0.01$.
      }
    \label{fig:phase_diagram}
\end{figure}

\subsection{Exponential/polynomial scaling regime} 
\label{subsection:exponential_polynomial_scaling_regime}
In the exponential/polynomial scaling regime, repeated test-time refinement remains beneficial and the generalization error vanishes asymptotically (\(\mathcal E_\infty=0\)). 
When \(\tau > \tau_c(\alpha, \sigma^2)\), the error decays exponentially fast  (\(\Lambda(\alpha,\tau,\sigma^2) < 1\)), whereas at the critical boundary \(\tau=\tau_c(\alpha,\sigma^2)\), the exponential factor disappears and only polynomial decay remains  (\(\Lambda(\alpha,\tau,\sigma^2)=1\)).

In this regime, repeated refinement remains stable, so increasing the CoT depth consistently reduces the generalization error rather than amplifying it. 
Moreover, the error converges to zero asymptotically, implying that perfect recovery is achievable in the large-depth limit. 
The following theorem gives a representative asymptotic scaling law in this favorable setting.

\begin{restatable}{theorem}{theoremexponentialdecay}
  \label{theorem:exponential_decay}
    Assume the population-risk regime \(\tau \to \infty\), ridgeless learning \(\lambda = 0\), and \(\alpha>1\).
    Then, as \(t \to \infty\),
    \begin{align}
      \mathcal E_t &\asymp t^{-3/2}
      \ab(\frac{2\sqrt{\alpha}+\sigma^2}{\alpha+1+\sigma^2})^{2t} .
    \end{align}
    Especially, in the context-rich limit $\alpha \gg \sigma^2+1$, $\mathcal E_t \asymp t^{-3/2} (4 / \alpha)^{t}$.
\end{restatable}
The proof is given in Appendix~\ref{appendix:exponential_decay}.

The simplified form Theorem~\ref{theorem:exponential_decay} clarifies that the decay rate is controlled by \(\alpha\), namely the effective amount of in-context examples. 
Larger \(\alpha\) leads to faster exponential improvement, meaning that each refinement step becomes more informative when the context is richer.

The main implication is that, sufficiently accurate pretraining makes additional test-time compute genuinely useful for improving performance. 
Moreover, increasing the number of in-context examples in this regime further improves the efficiency of test-time scaling by accelerating the decay rate. 
In this sense, pretraining enables effective test-time scaling, while in-context information determines how rapidly its benefit is realized.

\subsection{Overthinking regime}
\label{subsection:overthinking_regime}
In the overthinking regime, the effect of increasing test-time CoT depth is non-monotone. 
At small depths, additional reasoning can still improve performance, but beyond a certain point it amplifies error instead of reducing it. 
This behavior is governed by task diversity \(\tau\), which determines the quality of pretraining. 
More specifically, when \(\tau\) is small, the learned operator \(\mathbf{A}^*\) retains systematic error, which is then amplified by repeated test-time refinement. 
As a result, the overthinking regime gives rise to a nontrivial optimal stopping depth. 
The following result characterizes its scaling near the phase boundary, together with the asymptotic error achieved at the optimum.

\begin{restatable}{result}{resultoptimaldepth}
  \label{result:optimal_depth}
  Assume that $(\alpha, \tau, \sigma^2)$ is in the overthinking regime.
  Let $t^* := \argmin_{t\in\mathbb Z_{\ge 0}} \mathcal E_t$.
  Then, as \(\tau\to \tau_c(\alpha,\sigma^2)\) from below,
  \begin{align}
    t^*
    =C_t(\alpha,\sigma^2)\ab(\tau_c(\alpha,\sigma^2)-\tau)^{-1}
    \bigl(1+o(1)\bigr),
  \end{align}
  for some positive function \(C_t(\alpha,\sigma^2)\).
  Moreover, the generalization error at $t^*$ satisfies
  \begin{align}
    \mathcal E_{t^*}
    =
    C_E(\alpha,\sigma^2)\,(t^*)^{-1/2}
    \bigl(1+o(1)\bigr),
    \label{eq:Estar_via_tstar}
  \end{align}
  for some positive function \(C_E(\alpha,\sigma^2)\).
\end{restatable}

The detailed derivation is given in Appendix~\ref{appendix:optimal_depth}.

Result~\ref{result:optimal_depth} shows that, in the overthinking regime, the benefit of test-time reasoning is intrinsically finite-depth. 
As \(\tau\) approaches \(\tau_c(\alpha,\sigma^2)\) from below, the optimal stopping depth \(t^*\) diverges, meaning that deeper reasoning remains useful for increasingly longer horizons. 
At the same time, the best achievable error decreases only polynomially, as \((t^*)^{-1/2}\), rather than exponentially. 
Thus, improving pretraining extends the useful range of test-time CoT, but does not immediately restore the efficient exponential scaling.

This behavior reflects the origin of overthinking. 
When \(\tau\) is small, pretraining task diversity is insufficient, and the learned update rule \(\mathbf{A}^*\) is systematically imperfect. 
Repeated test-time refinement then amplifies this residual imperfection, so deeper reasoning eventually becomes harmful. 
The bottleneck is therefore not test-time compute itself, but the quality of the learned rule produced by pretraining.

The broader implication is that test-time scaling is conditional on pretraining quality. 
In the overthinking regime, increasing reasoning depth does not provide an unconditional path to better performance, because its benefit is confined to a finite useful horizon before error amplification takes over. 
Deeper reasoning therefore becomes reliably effective only when pretraining is strong enough to delay the onset of overthinking and sustain useful refinement over longer horizons.

\subsection{Saturation regime}
\label{subsection:saturation_regime}
In the saturation regime, repeated test-time refinement remains stable and converges (\(\Lambda(\alpha,\tau,\sigma^2)<1\)), but the generalization error does not vanish asymptotically (\(\mathcal E_\infty>0\)). 
This nonzero limit reflects an information bottleneck at test time. 
Even when pretraining is sufficiently accurate, too few in-context examples leave the prompt itself incomplete, so some information required for prediction is simply unavailable at inference time. 
It cannot reconstruct information that is absent from the context.

This mechanism also explains a seemingly counterintuitive phenomenon: increasing the number of in-context examples does not always lead to better long-horizon performance.
As shown in Fig.~\ref{fig:phase_diagram}C, a setting with larger $\alpha$ may perform better at small depths, yet be overtaken at sufficiently large depths by the saturation regime because long-horizon behavior is governed by stability of the test-time refinement dynamics. 
This qualitative prediction is consistent with \citep{ge2025innatereasoningenoughincontext}, which reports that adding more in-context CoT  examples does not monotonically improve performance.

These results imply that, in this regime, performance is limited mainly by the amount of in-context information rather than by test-time compute. 
Therefore, increasing the CoT depth alone yields diminishing returns, while enriching prompt examples is the more effective way to improve performance. 
This identifies the saturation regime as an information-limited phase of test-time scaling.

\section{Experiments}
\label{section:experiments}
In the preceding sections, we theoretically analyzed how the effect of test-time CoT depends on the number of pretraining tasks and in-context examples using a simplified linear attention model. 
Here, we ask if the same qualitative predictions continue to hold in more expressive settings. 
To this end, we consider two extensions: a full linear attention model and a softmax attention model.
As in the theoretical setting, \(L\) denotes the number of in-context examples per task and \(M\) denotes the number of training tasks. 
In both models, we use the same basic prompt format as in Section~\ref{section:model}, initialize \(\vb{w}_\mu^{\text{init}}\) as a random Gaussian vector, and read out the prediction from the parameter block of the final token. 
In both extensions, the attention mechanism is parameterized by query, key, and value matrices.

Because these expanded models permit more general token interactions than the simplified theory, we explicitly control which source tokens are visible to the final parameter-estimate token through an attention mask. 
Concretely, at each step, we treat the final parameter-estimate token as the query token and allow it to attend to the \(L\) in-context example tokens and all previously generated estimate tokens, while masking out the query token itself. 
This mask makes CoT an iterative update process anchored to the observed examples, while allowing previously generated estimates to be explicitly reused as intermediate states.
These details are given in Appendix~\ref{appendix:experimental_details}.

\begin{figure}[h]
    \centering
    \includegraphics[width=0.8\textwidth]{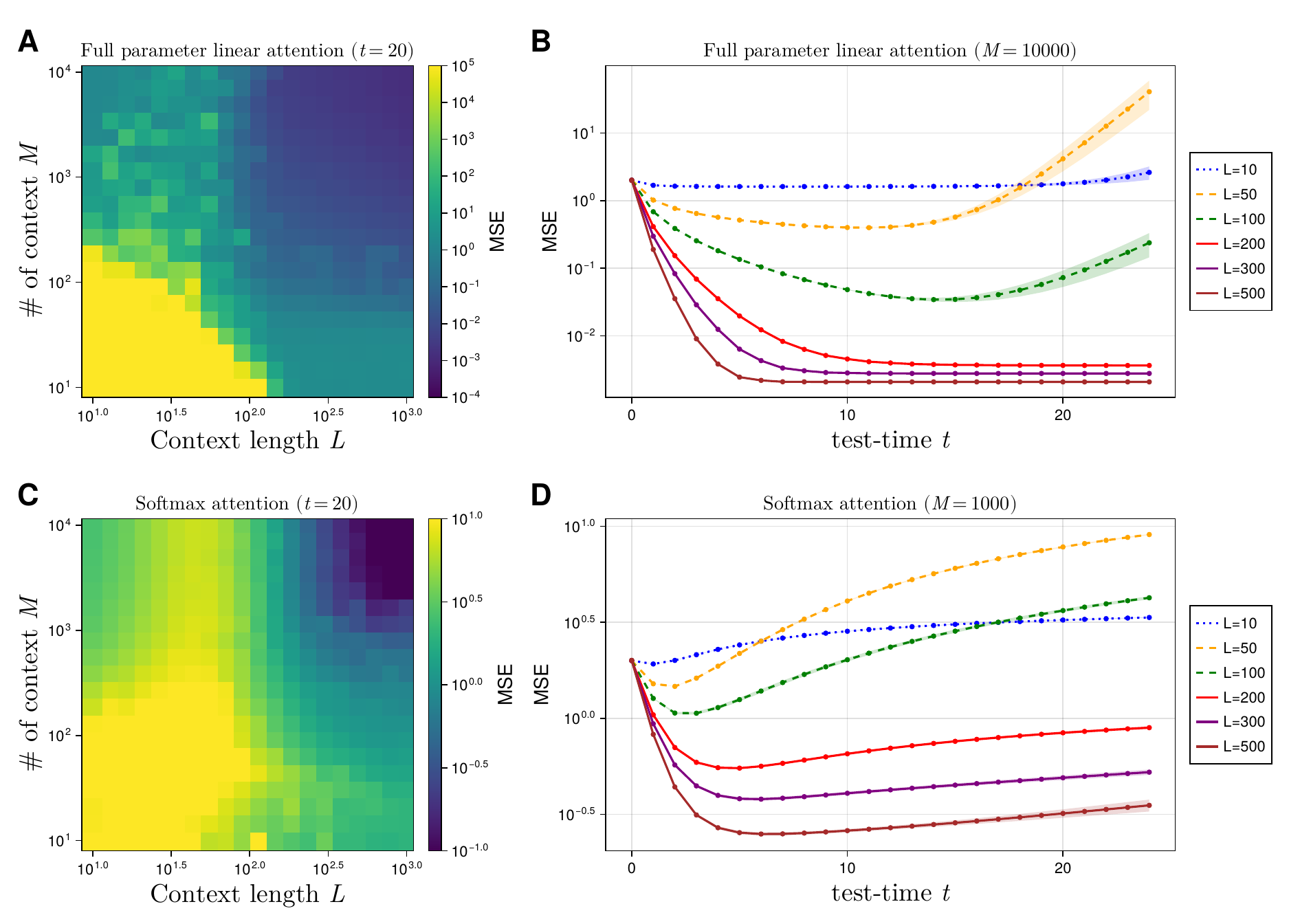}
    \caption{
      \textbf{CoT experiments in the fully learned linear attention and softmax attention models.}
      (A, C) Phase diagrams of the test-time generalization error at \(t=20\) as a function of the context length \(L\) and the number of training tasks \(M\), for (A) the fully learned linear attention model and (C) the softmax attention model, respectively.
      (B, D) Test-time generalization error as a function of reasoning depth in the (B) fully learned linear attention model and (D) the softmax attention model, respectively. 
      The solid curves correspond to a decay regime, the dashed curves to an overthinking regime, and the dotted curves to a saturation regime.
      Parameters: (A-D) \(\lambda=10^{-5}\), \(\sigma^2=0\), and \(D=50\). 
      Error bars represent the standard error of the mean over 5 trials per point.
    }
    \label{fig:experiments}
\end{figure}

Figures~\ref{fig:experiments}A and~C show the phase diagrams for long test-time CoT ($t=20$) in the fully learned linear attention model and the softmax attention model, respectively.
These phase diagrams suggest that the broad qualitative picture from the theory holds in more expressive models.
In particular, when $M$ is small, the overthinking regime appears more prominently, whereas as $M$ increases, the behavior separates into saturation and decay regimes depending on the number of in-context examples $L$.

Figures~\ref{fig:experiments}B and~D examine the corresponding test-time dynamics at representative values of $M$.
The results also reproduce the three characteristic phases predicted by the theory at a qualitative level.
The solid curves exhibit a decay regime; in the fully learned linear attention model, this decay is close to the exponential behavior predicted by the theory.
In the softmax attention model, however, the decrease appears more gradual over the displayed range, suggesting that while the same qualitative phase structure persists, the precise decay rate depends on the attention parameterization.
A possible reason is that softmax normalization constrains the effective update size at each step, leading to a more gradual decrease over the observed range.
The dashed curves show an overthinking regime, where the error initially decreases but eventually increases as the number of reasoning steps increases.
The dotted curves correspond to a saturation regime, in which the error improves little after an initial period and then remains nearly constant over a long range of reasoning steps.

Overall, the behavior of these regimes is consistent with the theoretical predictions on how $L$ and $M$ control the effectiveness of test-time CoT.
First, the decay regime appears consistent with the prediction in Section~\ref{subsection:exponential_polynomial_scaling_regime}: increasing the number of in-context examples accelerates the rate of error decay.
Second, the overthinking regime is broadly consistent with the prediction in Section~\ref{subsection:overthinking_regime}, in the sense that pretraining on more tasks makes longer test-time CoT beneficial over a wider range of reasoning steps.
Finally, in the long-CoT limit, the overthinking regime can yield worse performance than the saturation regime with fewer in-context examples, which is also in line with the prediction in Section~\ref{subsection:saturation_regime}.

\section{Conclusion}
\label{section:conclusion}
We studied the effect of test-time chain-of-thought depth on generalization in a solvable model of in-context learning for linear regression.
Using a two-point generating function approach combined with random matrix theory, we derived an exact asymptotic formula for the test-time generalization error dynamics.
This formula revealed a sharp phase transition separating four qualitatively distinct regimes of test-time scaling.
When pretraining is sufficient, deeper reasoning reduces the error exponentially at a rate controlled by the in-context examples.
When pretraining is insufficient, reasoning beyond a finite optimal depth amplifies error, and the useful range of test-time compute grows only as pretraining improves.
When in-context information is limited, reasoning remains stable but saturates at a nonzero error floor, since it cannot recover information absent from the context.
Our results provide a unified theoretical account of empirical test-time scaling phenomena, including improvement, saturation and overthinking, and clarify that the primary bottleneck shifts from pretraining quality to context information to test-time compute across these regimes.
These findings not only lay a theoretical foundation for future research on test-time scaling, but also offer practical insights for the design of pretraining strategies and test-time reasoning in LLMs.

\paragraph{Limitations.}
While our analysis provides an insightful unified theoretical framework for understanding the test-time CoT dynamics, 
it is restricted to a linear in-context learner in which CoT is formalized as iterative weight refinements.
Extending the theory beyond isotropic linear regression and to richer CoT structures and realistic architectures, together with broader empirical validation, is an important direction for future work.

\begin{ack}
We thank Mary Letey and Alex Atanasov for helpful discussions.
K.T. acknowledges support from the 2025 Young Researchers Overseas Dispatch Program,
DAIKIN Advanced Interdisciplinary Research (AIR)-Vision Mobility Grant.
K.T. was also supported by JST BOOST NAIS Grant Number JPMJBS2418.
C.P. is supported by an NSF CAREER Award (IIS-2239780), DARPA grants DIAL-FP-038 and AIQ-HR00112520041, the Simons Collaboration on the Physics of Learning and Neural Computation, and the William F. Milton Fund from Harvard University. 
This work has been made possible in part by a gift from the Chan Zuckerberg Initiative Foundation to establish the Kempner Institute for the Study of Natural and Artificial Intelligence.
The authors declare no competing interests.
\end{ack}

\bibliography{ref}
\bibliographystyle{unsrtnat}

\newpage
\appendix
\begin{center}
\huge{\textbf{Appendix}}    
\end{center}

In this appendix, we present a systematic asymptotic derivation of the high-dimensional theory underlying our results.
Our analysis is based on standard cavity and resolvent methods for random matrices and high-dimensional random systems.
The derivation is not fully rigorous in the mathematical sense: several steps rely on self-averaging, leave-one-out replacements, and deterministic-equivalent substitutions that we do not justify here in complete detail.
However, these manipulations are standard in the random matrix literature, and they lead to a closed system of equations that is internally consistent.
Moreover, the resulting predictions are in excellent quantitative agreement with our numerical experiments, which provides strong evidence that the theory derived below correctly captures the asymptotic behavior of the model.

\section{Justification of parameter simplification}
\label{appendix_justification_of_parameter_simplification}

This section explains the simplified parameterization in Section~\ref{section:model} and gives empirical support for this choice.
The main text fixes all blocks of \(\mathbf{W}\) and all non-\(\mathbf{A}\) blocks of \(\mathbf{V}\), leaving only \(\mathbf{A}\) as a learnable parameter.
This reduction is motivated by the structure learned by the full-parameter model.

Only the third and fourth block columns of \(\mathbf{W}\) and the third block row of \(\mathbf{V}\) can affect the readout.
With block sizes \((D,1,D,1)\), write
\begin{align}
  \mathbf{W}
  &=
  \begin{pmatrix}
    \star & \star & W_{13} & W_{14} \\
    \star & \star & W_{23} & W_{24} \\
    \star & \star & W_{33} & W_{34} \\
    \star & \star & W_{43} & W_{44}
  \end{pmatrix},
  &
  \mathbf{V}
  &=
  \begin{pmatrix}
    \star & \star & \star & \star \\
    \star & \star & \star & \star \\
    V_{31} & V_{32} & V_{33} & V_{34} \\
    \star & \star & \star & \star
  \end{pmatrix},
\end{align}
where \(\star\) is any conformable block that does not affect the readout.
Rotational equivariance of the isotropic task distribution motivates setting \(W_{14}=W_{23}=V_{32}=V_{34}=0\) and \(W_{13}=cI_D\).
Requiring the ground-truth estimate to be a fixed point when \(\vb{y}=\mathbf{X}^{\top}\vb{w}\) gives \(W_{24}=-c\).
We absorb the common scale \(c\) into \(V_{31}\) and set \(c=1\).
We also set \(V_{33}=0\) to remove a history-only update that could change the estimate after the context residual vanishes.

\begin{figure}[h]
    \centering
    \includegraphics[width=0.8\linewidth]{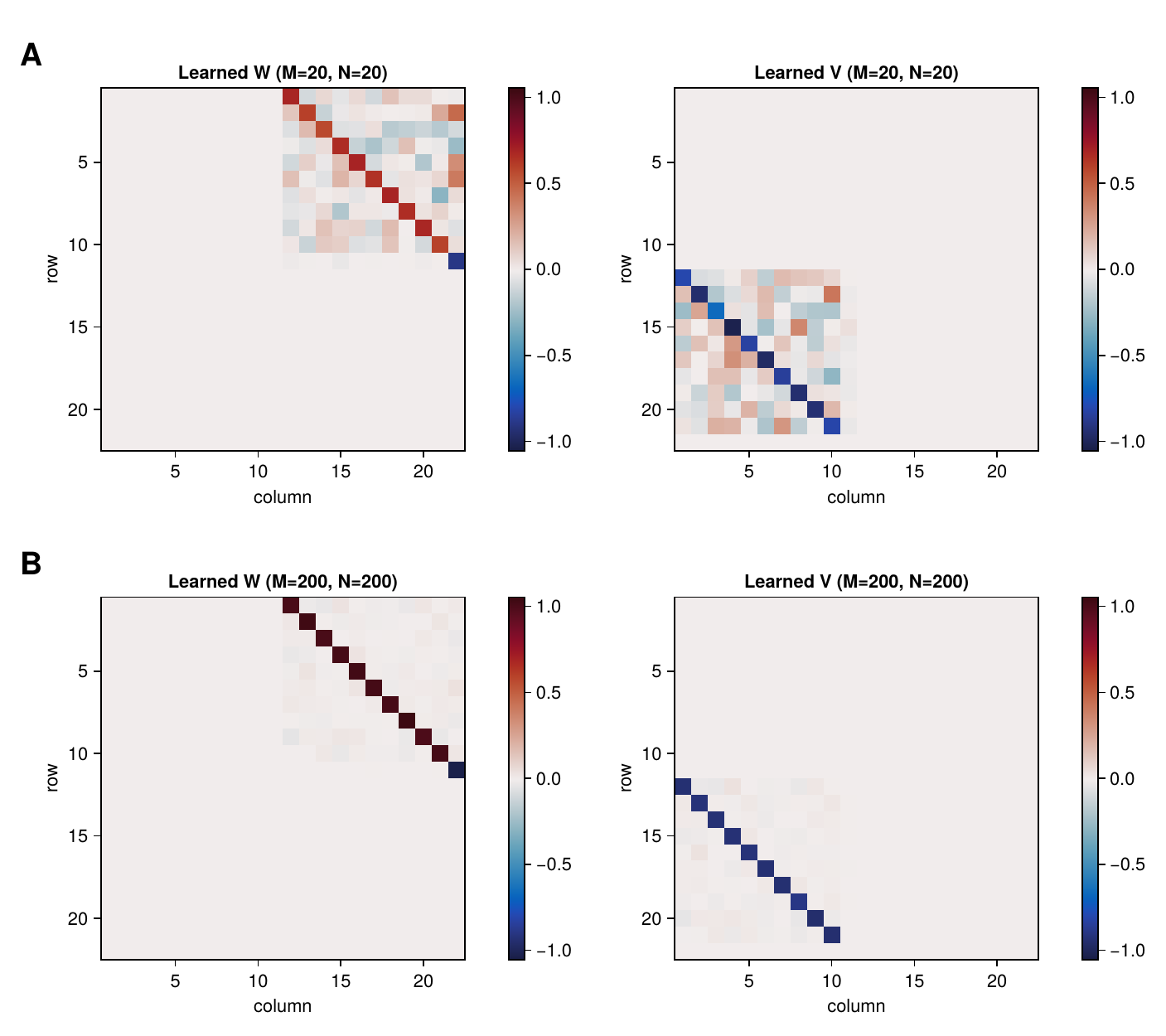}
    \caption{
    Heatmaps of the pretrained full-parameter matrices \(\mathbf{W}\) and \(\mathbf{V}\).
    Parameters: (A,B) $\lambda = 10^{-3}, \sigma = 0.1$ (A) \(M=N=20\), (B) \(M=N=200\).}
    \label{fig:placeholder}
\end{figure}

We train the full model, where all entries of \(\mathbf{W}\) and \(\mathbf{V}\) are learnable, and show the learned matrices in Figure~\ref{fig:placeholder}.
The results show that the learned \(\mathbf{W}\) develops the prescribed routing structure: the block mapping the parameter-estimate coordinates to the feature coordinates becomes close to an identity map, while the marker coordinate learns the corresponding sign structure.
Similarly, the learned \(\mathbf{V}\) concentrates its mass on the block that maps feature coordinates into the parameter-estimate coordinates, which is precisely the block represented by \(\mathbf{A}\) in the simplified model.
The remaining blocks stay close to zero.

This behavior becomes clearer as the number of training tasks and examples increases.
Thus, the simplified parameterization should not be viewed as imposing an arbitrary constraint. 
Rather, for analytical tractability, it fixes part of the model to the parameter structure that is naturally approached under full-parameter training, while keeping the essential component \(\mathbf{A}\), which governs the test-time update of the parameter estimate, learnable.
The detailed setup of the full-parameter model is given in Appendix~\ref{appendix:experimental_details}.

\section{Derivation of the readout formula}

\label{appendix:readout}

In this appendix, we derive the readout formula in Eq.~\eqref{eq:readout}.
First, recall that, for task $\mu$, the prompt embedding is
\begin{align}
    \mathbf{E}_\mu
    =
    \begin{bmatrix}
        \mathbf{X}_\mu & \vb{0}_{D} \\
        \vb{y}_\mu^\top & 0 \\
        0_{D\times L} & \hat{\vb{w}}_{\mu}^{\mathrm{init}} \\
        0_{1\times L} & 1
    \end{bmatrix}
    \in \mathbb{R}^{(2D+2)\times(L+1)},
\end{align}
where
\begin{align}
    \mathbf{X}_\mu \in \mathbb{R}^{D\times L},
    \qquad
    \vb{y}_\mu \in \mathbb{R}^{L},
    \qquad
    \hat{\vb{w}}_{\mu}^{\mathrm{init}} \in \mathbb{R}^{D}.
\end{align}
We partition all matrices according to the block structure $(D,1,D,1)$, corresponding respectively to the feature block, the response block, the parameter-estimate block, and the marker block.
The attention layer is
\begin{align}
    \textbf{Atten}(\mathbf{E};\theta)
    =
    \mathbf{E} + \frac{1}{L} \mathbf{V} \mathbf{E} \mathbf{E}^\top \mathbf{W} \mathbf{E},
\end{align}
and the readout extracts the parameter-estimate block from the final column:
\begin{align}
    \textbf{Read}(\mathbf{E})
    :=
    \mathbf{E}_{(D+2):(2D+1),(L+1)}
    \in \mathbb{R}^{D}.
\end{align}

To isolate the mechanism relevant to the parameter-estimate update, we retain only the blocks that are used in the parameterization of the main text and set the other potentially contributing blocks to zero. Specifically, we consider
\begin{align}
\mathbf{W}=
\begin{pmatrix}
  0 & 0 & W_{13} & 0\\
  0 & 0 & 0 & W_{24}\\
  0 & 0 & 0 & 0\\
  0 & 0 & 0 & 0
\end{pmatrix},
\qquad
\mathbf{V}=
\begin{pmatrix}
  0 & 0 & 0 & 0\\
  0 & 0 & 0 & 0\\
  V_{31} & 0 & 0 & 0\\
  0 & 0 & 0 & 0
\end{pmatrix},
\end{align}
where $V_{31}\in\mathbb{R}^{D\times D}$, $W_{13}\in\mathbb{R}^{D\times D}$, and $W_{24}\in\mathbb{R}$.

Since the first $L$ columns of \(\mathbf{E}_\mu\) have nonzero entries only in the first two block rows, while the last column has nonzero entries only in the third and fourth block rows, the only part of \(\mathbf{W}\) that affects the last column through the readout-relevant channel is
\begin{align}
    \left(\mathbf{W} \mathbf{E}_\mu\right)_{:,(L+1)}
    =
    \begin{bmatrix}
        W_{13}\hat{\vb{w}}_{\mu}^{\mathrm{init}} \\
        W_{24} \\
        0 \\
        0
    \end{bmatrix}.
\end{align}
The lower two blocks are immaterial in what follows, because they do not contribute to the first block of \(\mathbf{E}_\mu \mathbf{E}_\mu^\top \mathbf{W} \mathbf{E}_\mu\), which is the only part later selected by \(V_{31}\).
Multiplying by \(\mathbf{E}_\mu^\top\) from the left, we only need the contribution of the first two block rows of \(\left(\mathbf{W} \mathbf{E}_\mu\right)_{:,(L+1)}\). This gives
\begin{align}
    \left(\mathbf{E}_\mu^\top \mathbf{W} \mathbf{E}_\mu\right)_{:,(L+1)}
    &=
    \mathbf{E}_\mu^\top
    \begin{bmatrix}
        W_{13}\hat{\vb{w}}_{\mu}^{\mathrm{init}} \\
        W_{24} \\
        0 \\
        0
    \end{bmatrix} \\
    &=
    \begin{bmatrix}
        \mathbf{X}_\mu^\top W_{13}\hat{\vb{w}}_{\mu}^{\mathrm{init}} + \vb{y}_\mu W_{24} \\
        0
    \end{bmatrix}.
\end{align}
Multiplying by \(\mathbf{E}_\mu\) from the left, the first block of the last column becomes
\begin{align}
    \left(\mathbf{E}_\mu \mathbf{E}_\mu^\top \mathbf{W} \mathbf{E}_\mu\right)_{1,(L+1)}
    =
    \mathbf{X}_\mu \left(\mathbf{X}_\mu^\top W_{13}\hat{\vb{w}}_{\mu}^{\mathrm{init}} + \vb{y}_\mu W_{24}\right).
\end{align}
Since only the first block is needed by \(V_{31}\), this is the only term we keep.
Because only the \((3,1)\) block of \(\mathbf{V}\) contributes to the third row block, the update of the parameter-estimate slot is
\begin{align}
    \textbf{Read}(\textbf{Atten}(\mathbf{E}_\mu;\theta))
    =
    \hat{\vb{w}}_{\mu}^{\mathrm{init}}
    +
    \frac{1}{L}
    V_{31}
    \mathbf{X}_\mu
    \left(
        \mathbf{X}_\mu^\top W_{13}\hat{\vb{w}}_{\mu}^{\mathrm{init}}
        +
        \vb{y}_\mu W_{24}
    \right).
\end{align}
This is the general readout formula induced by the three relevant blocks $V_{31}$, $W_{13}$, and $W_{24}$.

Here, we specialize to the parameterization:
\begin{align}
    V_{31}=-D\mathbf{A},
    \qquad
    W_{13}=I_D,
    \qquad
    W_{24}=-1,
\end{align}
with \(\mathbf{A}\in\mathbb{R}^{D\times D}\). Substituting these values into the previous expression gives
\begin{align}
    \textbf{Read}(\textbf{Atten}(\mathbf{E}_\mu;\theta))
    &=
    \hat{\vb{w}}_{\mu}^{\mathrm{init}}
    -
    \frac{D}{L}
    \mathbf{A}
    \mathbf{X}_\mu
    \left(
        \mathbf{X}_\mu^\top \hat{\vb{w}}_{\mu}^{\mathrm{init}}
        -
        \vb{y}_\mu
    \right) \\
    &=
    \hat{\vb{w}}_{\mu}^{\mathrm{init}}
    -
    \frac{D}{L}
    \mathbf{A}
    \left(
        \mathbf{X}_\mu \mathbf{X}_\mu^\top \hat{\vb{w}}_{\mu}^{\mathrm{init}}
        -
        \mathbf{X}_\mu \vb{y}_\mu
    \right),
\end{align}
which is Eq.~\eqref{eq:readout}.
The choice $W_{13}=I_D$ means that the current estimate $\hat{\vb{w}}_{\mu}^{\mathrm{init}}$ is passed to the interaction term without distortion. 
Likewise, $W_{24}=-1$ attaches the response vector with the opposite sign, so that the quantity inside the parentheses becomes
\begin{align}
    \mathbf{X}_\mu^\top \hat{\vb{w}}_{\mu}^{\mathrm{init}} - \vb{y}_\mu,
\end{align}
which is exactly the vector of prediction residuals on the context examples. 
Therefore, with these two choices, the attention layer forms the standard regression error signal before applying the learned linear transform \(\mathbf{A}\).

The remaining block \(V_{31}=-D\mathbf{A}\) determines how this error signal is converted into an update of the parameter-estimate slot. 
The factor \(\mathbf{X}_\mu\) maps residuals back to parameter space, while \(\mathbf{A}\) acts as a learned preconditioner. 
In this sense, \(W_{13}=I_D\) and \(W_{24}=-1\) specify the canonical gradient-like structure, and the learnable content of the update is concentrated in \(\mathbf{A}\).

\paragraph{Connection to gradient descent.}
Under this specialization, the term
\begin{align}
    \mathbf{X}_\mu \mathbf{X}_\mu^\top \hat{\vb{w}}_{\mu}^{\mathrm{init}} - \mathbf{X}_\mu \vb{y}_\mu
\end{align}
is proportional to the empirical gradient of the squared loss
\begin{align}
    \mathcal{L}_\mu(\vb{w})
    :=
    \frac{1}{2L}\lVert \mathbf{X}_\mu^\top \vb{w} - \vb{y}_\mu\rVert_2^2,
\end{align}
since
\begin{align}
    \nabla \mathcal{L}_\mu(\vb{w})
    =
    \frac{1}{L}
    \left(
        \mathbf{X}_\mu \mathbf{X}_\mu^\top \vb{w} - \mathbf{X}_\mu \vb{y}_\mu
    \right).
\end{align}
Hence,
\begin{align}
    \textbf{Read}(\textbf{Atten}(\mathbf{E}_\mu;\theta))
    =
    \hat{\vb{w}}_{\mu}^{\mathrm{init}}
    -
    D \mathbf{A} \nabla \mathcal{L}_\mu(\hat{\vb{w}}_{\mu}^{\mathrm{init}}),
\end{align}
so the attention layer performs a preconditioned gradient step on the task parameter.

\section{Derivation of Result~\ref{result:main}}
\label{appendix:main_result}

In this appendix, we derive Result~\ref{result:main}, which characterizes the asymptotic generalization error.
The key object is the learned matrix \(\mathbf{A}^*=\mathbf{B}\mathbf{C}^{-1}\), whose randomness is inherited from the pre-training tasks.
Since the inference error depends on both the random matrix \(\mathbf{A}^*\) and the inference-time sample covariance \(\mathbf{S}=\mathbf{X}\mathbf{X}^\top\), the main challenge is to control their coupled effect in the high-dimensional limit.

Our derivation proceeds in four steps.
First, in Appendix~\ref{appendix:main_result_reduction}, we rewrite the generalization error as a finite linear combination of derivatives of a two-variable generating function \(F(u,v)\).
This step reduces the original prediction problem to the evaluation of a resolvent-type quantity associated with the inference dynamics.
Next, in Appendix~\ref{appendix:main_result_extended_matrix}, we introduce an extended block matrix \(\mathcal M(u,v)\) whose resolvent contains \(F(u,v)\) as the normalized trace of a single block.
This reformulation is essential because it removes the inverse \(\mathbf{C}^{-1}\) appearing in \(\mathbf{A}^*=\mathbf{B}\mathbf{C}^{-1}\), thereby converting the problem into a more tractable linear resolvent analysis.

The next task is to characterize the resolvent of \(\mathcal M(u,v)\) through a finite collection of scalar order parameters.
To this end, in Appendix~\ref{appendix:main_result_order_parameters}, we define the relevant normalized traces and derive the exact algebraic identities that follow from \(\mathbf{G}\mathcal M(u,v)=\mathbf{I}\), where \(\mathbf{G}=\mathcal M(u,v)^{-1}\).
These identities alone do not close, because they involve mixed traces with insertions of \(\mathbf{S}\), \(\mathbf{B}\), \(\mathbf{B}^\top\), and \(\mathbf{C}\).
We therefore derive cavity equations for these mixed quantities: Appendix~\ref{appendix:main_result_cavity_s} treats the \(\mathbf{S}\)-dependent terms by a leave-one-sample-out argument, while Appendices~\ref{appendix:main_result_cavity_b}, \ref{appendix:main_result_cavity_beta}, and \ref{appendix:main_result_cavity_c} treat the \(\mathbf{B}\)-, \(\mathbf{B}^\top\)-, and \(\mathbf{C}\)-dependent terms by leave-one-task-out arguments.
Combining these relations yields, in Appendix~\ref{appendix:main_result_closed_system}, a closed deterministic self-consistent system for the primary order parameters.

Finally, in Appendix~\ref{appendix:main_result_lambda_zero}, we solve this closed system explicitly in the limit \(\lambda\to0\).
The solution is expressed in terms of the scalar resolvent \(g(u)\) of a Wishart-type matrix, and in particular yields a closed-form expression for the block \(m_{13}(u,v)\).
Substituting this expression into the derivative formula obtained in Appendix~\ref{appendix:main_result_reduction} gives the claimed representation of the asymptotic generalization error.
In this way, the proof isolates the effect of pre-training through the finite-dimensional order-parameter system and makes the dependence of the final error on \(\alpha\), \(\tau\), and \(\sigma^2\) fully explicit.

We begin by restating the result.

\resultmain*

We now derive this formula in a sequence of reductions.

\subsection{Preliminaries}
For clarity, we first introduce the notation used throughout the appendix.
For each pre-training task \(\mu=1,\dots,M\), define
\begin{align}
  \vb{v}_\mu
  :=
  \frac{1}{\alpha}\mathbf{X}_\mu \vb{y}_\mu
  \in \mathbb R^D,
  \qquad
  \alpha=\frac{L}{D},
  \qquad
  \tau=\frac{M}{D},
\end{align}
and recall that, for \(k=0\), the empirical minimizer \(\mathbf{A}^*\) is given by
\begin{align}
  \mathbf{A}^*
  =
  \left[
    \frac{1}{M}\sum_{\mu=1}^M \vb{w}_\mu \vb{v}_\mu^\top
  \right]
  \left[
    \frac{1}{M}\sum_{\mu=1}^M \vb{v}_\mu \vb{v}_\mu^\top
    + \frac{\lambda}{2}\mathbf{I}
  \right]^{-1}.
\end{align}
Accordingly, throughout this appendix we write
\begin{align}
  \mathbf{B}:=\frac{1}{M}\sum_{\mu=1}^M \vb{w}_\mu \vb{v}_\mu^\top,
  \qquad
  \mathbf{C}:=\frac{1}{M}\sum_{\mu=1}^M \vb{v}_\mu \vb{v}_\mu^\top+\frac{\lambda}{2}\mathbf{I},
  \qquad
  \mathbf{A}:=\mathbf{B}\mathbf{C}^{-1},
\end{align}
so that \(\mathbf{A}=\mathbf{A}^*\).
For the inference-time task, we define
\begin{align}
  \mathbf{S}:=\mathbf{X}\mathbf{X}^\top=\sum_{\ell=1}^L \vb{x}_\ell \vb{x}_\ell^\top \in \mathbb R^{D\times D}.
\end{align}

In the large-\(D\) analysis below, we repeatedly use the effective second-order statistics of the pre-training pair \((\vb{w}_\mu,\vb{v}_\mu)\).
Since
\(
 \vb{v}_\mu=\alpha^{-1}\mathbf{X}_\mu\vb{y}_\mu
\)
and
\(
 \vb{y}_\mu=\mathbf{X}_\mu^\top \vb{w}_\mu+\vb{\epsilon}_\mu
\),
the isotropy of the Gaussian design implies
\begin{align}
  \frac{1}{D}\E\|\vb{w}_\mu\|^2 &= 1,
  \qquad
  \frac{1}{D}\E\!\left[\vb{w}_\mu^\top \vb{v}_\mu\right] = 1,
  \qquad
  \frac{1}{D}\E\|\vb{v}_\mu\|^2 \to c,
  \label{eq:appendix_wv_second_moments}
\end{align}
where
\begin{align}
  c:=1+\frac{1+\sigma^2}{\alpha}.
\end{align}
Equivalently, at the matrix level,
\begin{align}
  \E[\vb{w}_\mu \vb{w}_\mu^\top] = I_D,
  \qquad
  \E[\vb{w}_\mu \vb{v}_\mu^\top] = I_D,
  \qquad
  \frac{1}{D}\Tr \E[\vb{v}_\mu \vb{v}_\mu^\top] \to c.
  \label{eq:appendix_wv_covariances}
\end{align}
Therefore, whenever \(Q^{[\mu]}\in \mathbb R^{D\times D}\) is independent of \((\vb{w}_\mu,\vb{v}_\mu)\) and has bounded normalized trace norm, the corresponding quadratic forms self-average as
\begin{align}
  \frac{1}{D}\vb{w}_\mu^\top Q^{[\mu]}\vb{w}_\mu
  &\;\asymp\;
  \frac{1}{D}\Tr Q^{[\mu]},
  \nonumber\\
  \frac{1}{D}\vb{v}_\mu^\top Q^{[\mu]}\vb{w}_\mu
  &\;\asymp\;
  \frac{1}{D}\Tr Q^{[\mu]},
  \nonumber\\
  \frac{1}{D}\vb{v}_\mu^\top Q^{[\mu]}\vb{v}_\mu
  &\;\asymp\;
  c\,\frac{1}{D}\Tr Q^{[\mu]}.
  \label{eq:appendix_wv_quadratic_forms}
\end{align}
Later, \(Q^{[\mu]}\) will typically be a block of a leave-one-task-out resolvent \(\mathbf{G}^{[\mu]}\), so \eqref{eq:appendix_wv_quadratic_forms} is precisely the input used when replacing inner products such as
\(
\vb{v}_\mu^\top \mathbf{G}_{ab}^{[\mu]}\vb{w}_\mu
\)
and
\(
\vb{v}_\mu^\top \mathbf{G}_{ab}^{[\mu]}\vb{v}_\mu
\)
by normalized traces.

\subsection{Reduction of the generalization error to a generating function}
\label{appendix:main_result_reduction}

We first show that the generalization error can be expressed in terms of a two-variable generating function.
This representation is the starting point of the random matrix analysis.

When \(k=0\), the pre-training prompt contains only the initial estimate \(\hat{\vb{w}}_{\mu,0}=\vb{0}\), and the learned matrix \(\mathbf{A}=\mathbf{A}^*\) is used at inference time to update the parameter estimate by
\begin{align}
  \hat{\vb{w}}_{t+1}
  =
  \hat{\vb{w}}_t
  -
  \frac{1}{\alpha}
  \mathbf{A}\left(\mathbf{S}\hat{\vb{w}}_t - \mathbf{X}\vb{y}\right),
  \qquad
  \hat{\vb{w}}_0=\vb{0},
\end{align}
where \(\mathbf{S}=\mathbf{X}\mathbf{X}^\top\) and \(\vb{y}=\mathbf{X}^\top \vb{w}\) in the noiseless inference task.
Since \(\mathbf{X}\vb{y}=\mathbf{X}\mathbf{X}^\top \vb{w}=\mathbf{S}\vb{w}\), this recursion becomes
\begin{align}
  \hat{\vb{w}}_{t+1}
  =
  \hat{\vb{w}}_t
  -
  \frac{1}{\alpha}\mathbf{A} \mathbf{S}(\hat{\vb{w}}_t-\vb{w}).
\end{align}
Therefore, the estimation error \(\vb{e}_t:=\vb{w}-\hat{\vb{w}}_t\) satisfies the linear recursion
\begin{align}
  \vb{e}_{t+1}
  =
  \left(\mathbf{I}-\frac{1}{\alpha}\mathbf{A}\mathbf{S}\right)\vb{e}_t,
  \qquad
  \vb{e}_0=\vb{w}.
\end{align}
Iterating this relation yields
\begin{align}
  \vb{w}-\hat{\vb{w}}_t
  =
  \left(\mathbf{I}-\frac{1}{\alpha}\mathbf{A}\mathbf{S}\right)^t \vb{w}.
\end{align}

Using this identity, the generalization error can be written as
\begin{align}
  \mathcal E_t
  :=
  \frac{1}{D}\E\bigl[\|\vb{w}-\hat{\vb{w}}_t\|^2\bigr]
  =
  \frac{1}{D}\E\left[
    \vb{w}^\top
    \left(\mathbf{I}-\frac{1}{\alpha}\mathbf{S}\mathbf{A}^\top\right)^t
    \left(\mathbf{I}-\frac{1}{\alpha}\mathbf{A}\mathbf{S}\right)^t
    \vb{w}
  \right].
\end{align}
Since \(\vb{w}\sim \mathcal N(\vb{0},\mathbf{I}_D)\) is independent of the inference design matrix \(\mathbf{X}\), we may average over \(\vb{w}\) first and obtain
\begin{align}
  \mathcal E_t
  =
  \frac{1}{D}\E\Tr\left[
    \left(\mathbf{I}-\frac{1}{\alpha}\mathbf{A}\mathbf{S}\right)^t
    \left(\mathbf{I}-\frac{1}{\alpha}\mathbf{S}\mathbf{A}^\top\right)^t
  \right].
  \label{eq:appendix_generalization_trace}
\end{align}
Here and below, \(\E\) denotes the expectation over both the pre-training randomness defining \(\mathbf{A}\) and the inference-time randomness defining \(\mathbf{S}\).

To extract this quantity, we introduce the two-variable generating function
\begin{align}
  F(u,v)
  :=
  \frac{1}{D}\E\Tr\Bigl[
    (\mathbf{I}-u\mathbf{A}\mathbf{S})^{-1}(\mathbf{I}-v\mathbf{S}\mathbf{A}^\top)^{-1}
  \Bigr].
  \label{eq:appendix_F_definition}
\end{align}
Expanding both resolvents around \(u=v=0\), we have the absolutely formal power series
\begin{align}
  F(u,v) =
  \sum_{p,q\ge 0}
  u^p v^q\,
  \frac{1}{D}\E\Tr\bigl[(\mathbf{A}\mathbf{S})^p(\mathbf{S}\mathbf{A}^\top)^q\bigr].
  \label{eq:appendix_F_series}
\end{align}
On the other hand, expanding the powers in \eqref{eq:appendix_generalization_trace} by the binomial formula gives
\begin{align}
  \mathcal E_t
  &=
  \frac{1}{D}\E\Tr\left[
    \sum_{p=0}^t \binom{t}{p}\left(-\frac{1}{\alpha}\right)^p (\mathbf{A}\mathbf{S})^p
    \sum_{q=0}^t \binom{t}{q}\left(-\frac{1}{\alpha}\right)^q (\mathbf{S}\mathbf{A}^\top)^q
  \right] \notag\\
  &=
  \sum_{p=0}^t\sum_{q=0}^t
  \binom{t}{p}\binom{t}{q}
  \left(-\frac{1}{\alpha}\right)^{p+q}
  \frac{1}{D}\E\Tr\bigl[(\mathbf{A}\mathbf{S})^p(\mathbf{S}\mathbf{A}^\top)^q\bigr].
\end{align}
Comparing this with \eqref{eq:appendix_F_series}, we conclude that
\begin{align}
  \mathcal E_t
  =
  \sum_{p=0}^t\sum_{q=0}^t
  \binom{t}{p}\binom{t}{q}
  \left(-\frac{1}{\alpha}\right)^{p+q}
  \frac{1}{p!\,q!}
  \partial_u^p\partial_v^q F(u,v)\Big|_{u=v=0}.
  \label{eq:appendix_generalization_from_F}
\end{align}

Equation~\eqref{eq:appendix_generalization_from_F} shows that the asymptotic evaluation of \(\mathcal E_t\) reduces to the determination of the scalar function \(F(u,v)\).
The remaining task is therefore to compute the large-\(D\) limit of \(F(u,v)\).

\subsection{Extended matrix representation of the resolvent}
\label{appendix:main_result_extended_matrix}

We next rewrite \(F(u,v)\) as a block trace of the resolvent of a larger matrix.
This step is crucial because the matrix \(\mathbf{A}=\mathbf{B}\mathbf{C}^{-1}\) contains the inverse \(\mathbf{C}^{-1}\), which is inconvenient to manipulate directly.
The extended matrix formulation removes this inverse by embedding \(\mathbf{A}\) into a linear block system.

Recall that
\begin{align}
  \mathbf{A}=\mathbf{B}\mathbf{C}^{-1},
  \qquad
  \mathbf{A}^\top=\mathbf{C}^{-1}\mathbf{B}^\top,
\end{align}
where
\begin{align}
  \mathbf{B}=\frac{1}{M}\sum_{\mu=1}^M \vb{w}_\mu \vb{v}_\mu^\top,
  \qquad
  \mathbf{C}=\frac{1}{M}\sum_{\mu=1}^M \vb{v}_\mu \vb{v}_\mu^\top+\frac{\lambda}{2}\mathbf{I}.
\end{align}
We introduce the extended block matrix
\begin{align}
  \mathcal M(u,v)
  :=
  \begin{pmatrix}
    \mathbf{I} & -u\mathbf{B} & -\mathbf{I} & 0 \\
    -\mathbf{S} & \mathbf{C} & 0 & 0 \\
    0 & 0 & \mathbf{I} & -v\mathbf{S} \\
    0 & 0 & -\mathbf{B}^\top & \mathbf{C}
  \end{pmatrix}
  \in \mathbb R^{4D\times 4D},
\end{align}
and denote its inverse by
\begin{align}
  \mathbf{G}(u,v):=\mathcal M(u,v)^{-1}.
\end{align}
We write \(\mathbf{G}\) in \(D\times D\) block form as
\begin{align}
  \mathbf{G}=
  \begin{pmatrix}
    \mathbf{G}_{11} & \mathbf{G}_{12} & \mathbf{G}_{13} & \mathbf{G}_{14} \\
    \mathbf{G}_{21} & \mathbf{G}_{22} & \mathbf{G}_{23} & \mathbf{G}_{24} \\
    0 & 0 & \mathbf{G}_{33} & \mathbf{G}_{34} \\
    0 & 0 & \mathbf{G}_{43} & \mathbf{G}_{44}
  \end{pmatrix}.
\end{align}
The lower-left zero block follows from the upper block-triangular structure of \(\mathcal M(u,v)\).

To identify \(\mathbf{G}_{13}\), it is convenient to decompose \(\mathcal M(u,v)\) as
\begin{align}
  \mathcal M(u,v)
  =
  \begin{pmatrix}
    M_{11} & M_{12} \\
    0 & M_{22}
  \end{pmatrix},
\end{align}
where
\begin{align}
  M_{11}
  :=
  \begin{pmatrix}
    \mathbf{I} & -u\mathbf{B} \\
    -\mathbf{S} & \mathbf{C}
  \end{pmatrix},
  \qquad
  M_{12}
  :=
  \begin{pmatrix}
    -\mathbf{I} & 0 \\
    0 & 0
  \end{pmatrix},
  \qquad
  M_{22}
  :=
  \begin{pmatrix}
    \mathbf{I} & -v\mathbf{S} \\
    -\mathbf{B}^\top & \mathbf{C}
  \end{pmatrix}.
\end{align}
Since \(\mathcal M(u,v)\) is upper block triangular, its inverse is
\begin{align}
  \mathbf{G}
  =
  \mathcal M(u,v)^{-1}
  =
  \begin{pmatrix}
    M_{11}^{-1} & -M_{11}^{-1}M_{12}M_{22}^{-1} \\
    0 & M_{22}^{-1}
  \end{pmatrix}.
\end{align}
Hence the \((1,3)\)-block of \(\mathbf{G}\) is given by
\begin{align}
  \mathbf{G}_{13}
  =
  \bigl(-M_{11}^{-1}M_{12}M_{22}^{-1}\bigr)_{11}
  =
  \bigl(M_{11}^{-1}\bigr)_{11}\bigl(M_{22}^{-1}\bigr)_{11}.
  \label{eq:appendix_G13_factorization}
\end{align}

We now evaluate these two factors by the Schur complement formula.
For \(M_{11}\), the \((1,1)\)-block of the inverse is
\begin{align}
  \bigl(M_{11}^{-1}\bigr)_{11}
  &=
  \Bigl(\mathbf{I}-(-u\mathbf{B})\mathbf{C}^{-1}(-\mathbf{S})\Bigr)^{-1}
  =
  (\mathbf{I}-u\mathbf{B}\mathbf{C}^{-1}\mathbf{S})^{-1}
  =
  (\mathbf{I}-u\mathbf{A}\mathbf{S})^{-1}.
\end{align}
Similarly, for \(M_{22}\), we obtain
\begin{align}
  \bigl(M_{22}^{-1}\bigr)_{11}
  &=
  \Bigl(\mathbf{I}-(-v\mathbf{S})\mathbf{C}^{-1}(-\mathbf{B}^\top)\Bigr)^{-1}
  =
  (\mathbf{I}-v\mathbf{S}\mathbf{C}^{-1}\mathbf{B}^\top)^{-1}
  =
  (\mathbf{I}-v\mathbf{S}\mathbf{A}^\top)^{-1}.
\end{align}
Substituting these identities into \eqref{eq:appendix_G13_factorization}, we arrive at
\begin{align}
  \mathbf{G}_{13}
  =
  (\mathbf{I}-u\mathbf{A}\mathbf{S})^{-1}(\mathbf{I}-v\mathbf{S}\mathbf{A}^\top)^{-1}.
  \label{eq:appendix_G13_identification}
\end{align}

Taking the normalized trace and expectation, we obtain the desired representation of \(F(u,v)\):
\begin{align}
  F(u,v)
  =
  \frac{1}{D}\E\Tr(\mathbf{G}_{13}).
  \label{eq:appendix_F_as_G13}
\end{align}
Thus, the original problem has been reduced to the analysis of one block of the resolvent of the extended matrix \(\mathcal M(u,v)\).

In the remainder of the proof, we study the large-\(D\) behavior of \(\mathbf{G}=\mathcal M(u,v)^{-1}\) through a finite set of normalized traces, which we call order parameters.
The key point is that the block structure of \(\mathcal M(u,v)\) allows these order parameters to satisfy a closed deterministic system in the high-dimensional limit.
Once that system is solved, \eqref{eq:appendix_F_as_G13} yields \(F(u,v)\), and then \eqref{eq:appendix_generalization_from_F} gives the asymptotic generalization error.


Similarly, considering the average error when \(\mu\) is chosen uniformly at random from the set of trained tasks is given by
\begin{align}
  \mathcal E_t(\mathbf{A}) &=
  \frac{1}{D} \mathbb{E}_X \, \Tr \ab(
  \ab(\mathbf{I}-\frac{D}{L}\mathbf{A}\mathbf{X}\mathbf{X}^\top\ab)^t
  \vb{w}_\mu \vb{w}_\mu^\top
  \ab(\mathbf{I}-\frac{D}{L}\mathbf{X}\mathbf{X}^\top \mathbf{A}^\top\ab)^t
  ) \\
  &= \frac{1}{D} \mathbb{E}_X \, \Tr \ab(
    \ab(\mathbf{I}-\frac{D}{L}\mathbf{A}\mathbf{X}\mathbf{X}^\top\ab)^t
    \ab(\frac{1}{M} \sum_\mu \vb{w}_\mu \vb{w}_\mu^\top)
    \ab(\mathbf{I}-\frac{D}{L}\mathbf{X}\mathbf{X}^\top \mathbf{A}^\top\ab)^t
    ) \\
  &= \sum_{p=0}^t\sum_{q=0}^t
    \binom{t}{p}\binom{t}{q}
    \ab(-\frac1\alpha)^{p+q}
    \frac{1}{p!\,q!}\,
    \partial_u^p\partial_v^q \tilde{F}(u,v)\Big|_{u=v=0}
\end{align}
where 
\begin{align}
  \tilde{F}(u,v) =&
  \frac1D \mathbb E_X\,\Tr \ab[(\mathbf{I}-u\mathbf{A}\mathbf{S})^{-1}
  \ab(\frac{1}{M} \sum_\mu \vb{w}_\mu \vb{w}_\mu^\top)
  (\mathbf{I}-v\mathbf{S}\mathbf{A}^\top)^{-1}\ab] \\
  =& \frac1D\mathbb E_X\,\operatorname{tr}\ab(\mathbf{G}_{13} \ab(\frac{1}{M} \sum_\mu \vb{w}_\mu \vb{w}_\mu^\top)).
\end{align}

\subsection{Order parameters and basic identities from \(\mathbf{G}\mathcal M(u,v)=\mathbf{I}\)}
\label{appendix:main_result_order_parameters}

We now introduce the normalized traces that will be used to characterize the resolvent \(\mathbf{G}=\mathcal M(u,v)^{-1}\).
Our aim is to reduce the full matrix-valued problem to a finite set of scalar quantities that remain deterministic in the high-dimensional limit.
These quantities are chosen so that the block matrix identity
\begin{align}
  \mathbf{G}\mathcal M(u,v)=\mathbf{I}_{4D}
\end{align}
closes into a finite system once combined with the cavity relations derived in the following subsections.

\paragraph{Order parameters.}
Recall that \(\mathbf{G}\) is written in \(D\times D\) blocks as
\begin{align}
  \mathbf{G}=
  \begin{pmatrix}
    \mathbf{G}_{11} & \mathbf{G}_{12} & \mathbf{G}_{13} & \mathbf{G}_{14} \\
    \mathbf{G}_{21} & \mathbf{G}_{22} & \mathbf{G}_{23} & \mathbf{G}_{24} \\
    0 & 0 & \mathbf{G}_{33} & \mathbf{G}_{34} \\
    0 & 0 & \mathbf{G}_{43} & \mathbf{G}_{44}
  \end{pmatrix}.
\end{align}
We first define the normalized traces of these blocks:
\begin{align}
  &m_{11}:=\frac{1}{D}\Tr(\mathbf{G}_{11}),
  \qquad
  m_{12}:=\frac{1}{D}\Tr(\mathbf{G}_{12}),
  \qquad
  m_{13}:=\frac{1}{D}\Tr(\mathbf{G}_{13}),
  \qquad
  m_{14}:=\frac{1}{D}\Tr(\mathbf{G}_{14}), \notag\\
  &m_{21}:=\frac{1}{D}\Tr(\mathbf{G}_{21}),
  \qquad
  m_{22}:=\frac{1}{D}\Tr(\mathbf{G}_{22}),
  \qquad
  m_{23}:=\frac{1}{D}\Tr(\mathbf{G}_{23}),
  \qquad
  m_{24}:=\frac{1}{D}\Tr(\mathbf{G}_{24}), \notag\\
  &m_{33}:=\frac{1}{D}\Tr(\mathbf{G}_{33}),
  \qquad
  m_{34}:=\frac{1}{D}\Tr(\mathbf{G}_{34}),
  \qquad
  m_{43}:=\frac{1}{D}\Tr(\mathbf{G}_{43}),
  \qquad
  m_{44}:=\frac{1}{D}\Tr(\mathbf{G}_{44}).
  \label{eq:appendix_def_m}
\end{align}

Next, since the block equations generated by \(\mathbf{G}\mathcal M=\mathbf{I}\) involve the matrices \(\mathbf{S}\), \(\mathbf{B}\), \(\mathbf{B}^\top\), and \(\mathbf{C}\), we also introduce mixed normalized traces in which these matrices are inserted.
For the terms involving the inference covariance \(\mathbf{S}\), we define
\begin{align}
  s_{12}&:=\frac{1}{D}\Tr(\mathbf{G}_{12}\mathbf{S}),
  &
  s_{13}&:=\frac{1}{D}\Tr(\mathbf{G}_{13}\mathbf{S}),
  &
  s_{22}&:=\frac{1}{D}\Tr(\mathbf{G}_{22}\mathbf{S}), \notag\\
  s_{23}&:=\frac{1}{D}\Tr(\mathbf{G}_{23}\mathbf{S}),
  &
  s_{33}&:=\frac{1}{D}\Tr(\mathbf{G}_{33}\mathbf{S}),
  &
  s_{43}&:=\frac{1}{D}\Tr(\mathbf{G}_{43}\mathbf{S}).
  \label{eq:appendix_def_s}
\end{align}
For the terms involving the task-dependent matrix \(\mathbf{B}\), we define
\begin{align}
  b_{11}:=\frac{1}{D}\Tr(\mathbf{G}_{11}\mathbf{B}),
  \qquad
  b_{21}:=\frac{1}{D}\Tr(\mathbf{G}_{21}\mathbf{B}),
  \label{eq:appendix_def_b}
\end{align}
and for the terms involving \(\mathbf{B}^\top\), we define
\begin{align}
  \beta_{14}&:=\frac{1}{D}\Tr(\mathbf{G}_{14}\mathbf{B}^\top),
  \qquad
  \beta_{24}:=\frac{1}{D}\Tr(\mathbf{G}_{24}\mathbf{B}^\top), \notag\\
  \beta_{34}&:=\frac{1}{D}\Tr(\mathbf{G}_{34}\mathbf{B}^\top),
  \qquad
  \beta_{44}:=\frac{1}{D}\Tr(\mathbf{G}_{44}\mathbf{B}^\top).
  \label{eq:appendix_def_beta}
\end{align}
Finally, for the terms involving \(\mathbf{C}\), we define
\begin{align}
  c_{12}&:=\frac{1}{D}\Tr(\mathbf{G}_{12}\mathbf{C}),
  \qquad
  c_{14}:=\frac{1}{D}\Tr(\mathbf{G}_{14}\mathbf{C}), \notag\\
  c_{22}&:=\frac{1}{D}\Tr(\mathbf{G}_{22}\mathbf{C}),
  \qquad
  c_{24}:=\frac{1}{D}\Tr(\mathbf{G}_{24}\mathbf{C}), \notag\\
  c_{34}&:=\frac{1}{D}\Tr(\mathbf{G}_{34}\mathbf{C}),
  \qquad
  c_{44}:=\frac{1}{D}\Tr(\mathbf{G}_{44}\mathbf{C}).
  \label{eq:appendix_def_c}
\end{align}

The quantity of ultimate interest is \(m_{13}\), because by \eqref{eq:appendix_F_as_G13} we have
\begin{align}
  F(u,v)=\frac{1}{D}\E\Tr(\mathbf{G}_{13})=\E[m_{13}].
\end{align}
The remaining order parameters are auxiliary variables introduced only to obtain a closed system.

\paragraph{Block equations from \(\mathbf{G}\mathcal M=\mathbf{I}\).}
We now derive the basic algebraic relations satisfied by these order parameters.
Multiplying \(\mathbf{G}\) and \(\mathcal M(u,v)\), we obtain
\begin{align}
  &\mathbf{G}\mathcal M(u,v) \notag\\
  &=
  \begin{pmatrix}
    \mathbf{G}_{11} & \mathbf{G}_{12} & \mathbf{G}_{13} & \mathbf{G}_{14} \\
    \mathbf{G}_{21} & \mathbf{G}_{22} & \mathbf{G}_{23} & \mathbf{G}_{24} \\
    0 & 0 & \mathbf{G}_{33} & \mathbf{G}_{34} \\
    0 & 0 & \mathbf{G}_{43} & \mathbf{G}_{44}
  \end{pmatrix}
  \begin{pmatrix}
    \mathbf{I} & -u\mathbf{B} & -\mathbf{I} & 0 \\
    -\mathbf{S} & \mathbf{C} & 0 & 0 \\
    0 & 0 & \mathbf{I} & -v\mathbf{S} \\
    0 & 0 & -\mathbf{B}^\top & \mathbf{C}
  \end{pmatrix} \notag\\
  &=
  \begin{pmatrix}
    \mathbf{G}_{11}-\mathbf{G}_{12}\mathbf{S}
    &
    -u\mathbf{G}_{11}\mathbf{B}+\mathbf{G}_{12}\mathbf{C}
    &
    -\mathbf{G}_{11}+\mathbf{G}_{13}-\mathbf{G}_{14}\mathbf{B}^\top
    &
    -v\mathbf{G}_{13}\mathbf{S}+\mathbf{G}_{14}\mathbf{C}
    \\
    \mathbf{G}_{21}-\mathbf{G}_{22}\mathbf{S}
    &
    -u\mathbf{G}_{21}\mathbf{B}+\mathbf{G}_{22}\mathbf{C}
    &
    -\mathbf{G}_{21}+\mathbf{G}_{23}-\mathbf{G}_{24}\mathbf{B}^\top
    &
    -v\mathbf{G}_{23}\mathbf{S}+\mathbf{G}_{24}\mathbf{C}
    \\
    0
    &
    0
    &
    \mathbf{G}_{33}-\mathbf{G}_{34}\mathbf{B}^\top
    &
    -v\mathbf{G}_{33}\mathbf{S}+\mathbf{G}_{34}\mathbf{C}
    \\
    0
    &
    0
    &
    \mathbf{G}_{43}-\mathbf{G}_{44}\mathbf{B}^\top
    &
    -v\mathbf{G}_{43}\mathbf{S}+\mathbf{G}_{44}\mathbf{C}
  \end{pmatrix}.
\end{align}
Since \(\mathbf{G}\mathcal M(u,v)=\mathbf{I}_{4D}\), each block must match the corresponding block of the identity matrix.
Taking the normalized trace of each nontrivial block yields the following scalar relations.

From the \((1,1)\), \((1,2)\), \((1,3)\), and \((1,4)\) blocks, we obtain
\begin{align}
  m_{11}-s_{12} &= 1,
  \label{eq:appendix_basic_111}\\
  -u\,b_{11}+c_{12} &= 0,
  \label{eq:appendix_basic_112}\\
  -m_{11}+m_{13}-\beta_{14} &= 0,
  \label{eq:appendix_basic_113}\\
  -v\,s_{13}+c_{14} &= 0.
  \label{eq:appendix_basic_114}
\end{align}
From the \((2,1)\), \((2,2)\), \((2,3)\), and \((2,4)\) blocks, we obtain
\begin{align}
  m_{21}-s_{22} &= 0,
  \label{eq:appendix_basic_221}\\
  -u\,b_{21}+c_{22} &= 1,
  \label{eq:appendix_basic_222}\\
  -m_{21}+m_{23}-\beta_{24} &= 0,
  \label{eq:appendix_basic_223}\\
  -v\,s_{23}+c_{24} &= 0.
  \label{eq:appendix_basic_224}
\end{align}
From the \((3,3)\), \((3,4)\), \((4,3)\), and \((4,4)\) blocks, we obtain
\begin{align}
  m_{33}-\beta_{34} &= 1,
  \label{eq:appendix_basic_333}\\
  -v\,s_{33}+c_{34} &= 0,
  \label{eq:appendix_basic_334}\\
  m_{43}-\beta_{44} &= 0,
  \label{eq:appendix_basic_443}\\
  -v\,s_{43}+c_{44} &= 1.
  \label{eq:appendix_basic_444}
\end{align}

Equations
\eqref{eq:appendix_basic_111}--\eqref{eq:appendix_basic_444}
are exact identities that hold for every realization of the randomness.
They express the primary traces \(m_{ab}\) in terms of the mixed traces involving \(\mathbf{S}\), \(\mathbf{B}\), \(\mathbf{B}^\top\), and \(\mathbf{C}\).
At this stage the system is not yet closed, because the quantities \(s_{ab}\), \(b_{ab}\), \(\beta_{ab}\), and \(c_{ab}\) still depend on the full matrix structure of \(\mathbf{G}\).

The role of the cavity method is precisely to close this system.
In the next subsection, we analyze the dependence on the inference samples \(\{\vb{x}_\ell\}_{\ell=1}^L\) and express the \(s\)-variables in terms of the \(m\)-variables alone.
Subsequently, by removing one pre-training task at a time, we derive analogous relations for \(b\), \(\beta\), and \(c\).
Once these relations are combined with
\eqref{eq:appendix_basic_111}--\eqref{eq:appendix_basic_444},
the order parameters satisfy a closed deterministic system in the high-dimensional limit.


\subsection{Cavity equations for the \(S\)-dependent order parameters}
\label{appendix:main_result_cavity_s}

We next derive closed equations for the quantities
\(
s_{12},s_{13},s_{22},s_{23},s_{33},s_{43}
\),
which contain the inference-time sample covariance
\(
\mathbf{S}=\sum_{\ell=1}^L \vb{x}_\ell \vb{x}_\ell^\top
\).
The key idea is to remove one inference sample \(\vb{x}_\ell\) from \(\mathbf{S}\), compare the full resolvent with the leave-one-out resolvent, and then use concentration of quadratic forms.
Since the randomness of \(\mathbf{S}\) comes from the independent Gaussian vectors \(\{\vb{x}_\ell\}_{\ell=1}^L\), this leave-one-out argument closes the \(s\)-variables in terms of the basic traces \(m_{ab}\).

\paragraph{Leave-one-out decomposition.}
Fix \(\ell\in\{1,\dots,L\}\), and define the leave-one-out sample covariance
\begin{align}
  \mathbf{S}^{[\ell]}
  :=
  \sum_{j\ne \ell}\vb{x}_j\vb{x}_j^\top
  =
  \mathbf{S}-\vb{x}_\ell\vb{x}_\ell^\top.
\end{align}
Correspondingly, we introduce the leave-one-out extended matrix
\begin{align}
  \mathcal M^{[\ell]}(u,v)
  :=
  \begin{pmatrix}
    \mathbf{I} & -u\mathbf{B} & -\mathbf{I} & 0 \\
    -\mathbf{S}^{[\ell]} & \mathbf{C} & 0 & 0 \\
    0 & 0 & \mathbf{I} & -v\mathbf{S}^{[\ell]} \\
    0 & 0 & -\mathbf{B}^\top & \mathbf{C}
  \end{pmatrix},
\end{align}
and its inverse
\begin{align}
  \mathbf{G}^{[\ell]} := \bigl(\mathcal M^{[\ell]}(u,v)\bigr)^{-1}.
\end{align}
Then the full matrix can be written as
\begin{align}
  \mathcal M(u,v)=\mathcal M^{[\ell]}(u,v)+\Delta_\ell,
\end{align}
where
\begin{align}
  \Delta_\ell
  =
  \begin{pmatrix}
    0 & 0 & 0 & 0 \\
    -\vb{x}_\ell\vb{x}_\ell^\top & 0 & 0 & 0 \\
    0 & 0 & 0 & -v\,\vb{x}_\ell\vb{x}_\ell^\top \\
    0 & 0 & 0 & 0
  \end{pmatrix}.
  \label{eq:appendix_delta_l}
\end{align}
This perturbation has rank \(2\), and it is convenient to factorize it as
\begin{align}
  \Delta_\ell = U_\ell V_\ell^\top,
\end{align}
with
\begin{align}
  U_\ell
  :=
  \begin{pmatrix}
    0 & 0 \\
    -\vb{x}_\ell & 0 \\
    0 & -\sqrt{v}\,\vb{x}_\ell \\
    0 & 0
  \end{pmatrix}
  \in \mathbb R^{4D\times 2},
  \qquad
  V_\ell
  :=
  \begin{pmatrix}
    \vb{x}_\ell & 0 \\
    0 & 0 \\
    0 & 0 \\
    0 & \sqrt{v}\,\vb{x}_\ell
  \end{pmatrix}
  \in \mathbb R^{4D\times 2}.
  \label{eq:appendix_UV_l}
\end{align}

\paragraph{Resolvent identity.}
Applying the Woodbury formula to
\(
\mathcal M=\mathcal M^{[\ell]}+U_\ell V_\ell^\top
\),
we obtain
\begin{align}
  \mathbf{G}
  =
  \mathbf{G}^{[\ell]}
  -
  \mathbf{G}^{[\ell]}U_\ell
  \Bigl(\mathbf{I}_2+V_\ell^\top \mathbf{G}^{[\ell]}U_\ell\Bigr)^{-1}
  V_\ell^\top \mathbf{G}^{[\ell]}.
  \label{eq:appendix_woodbury_l}
\end{align}
We denote the \(2\times2\) matrix in the middle by
\begin{align}
  \mathbf{K}_\ell
  :=
  V_\ell^\top \mathbf{G}^{[\ell]}U_\ell.
\end{align}
Using the explicit forms \eqref{eq:appendix_UV_l}, a direct block computation gives
\begin{align}
  \mathbf{K}_\ell
  =
  \begin{pmatrix}
    -\,\vb{x}_\ell^\top \mathbf{G}_{12}^{[\ell]}\vb{x}_\ell
    &
    -\,\sqrt v\,\vb{x}_\ell^\top \mathbf{G}_{13}^{[\ell]}\vb{x}_\ell
    \\[2mm]
    0
    &
    -\,v\,\vb{x}_\ell^\top \mathbf{G}_{43}^{[\ell]}\vb{x}_\ell
  \end{pmatrix}.
  \label{eq:appendix_K_l_exact}
\end{align}

The advantage of the leave-one-out construction is that \(\vb{x}_\ell\) is independent of \(\mathbf{G}^{[\ell]}\), because \(\mathbf{G}^{[\ell]}\) depends only on \(\{\vb{x}_j\}_{j\ne \ell}\) and on the pre-training randomness.
Since \(\vb{x}_\ell\sim\mathcal N(\vb{0},I_D/D)\), standard concentration of quadratic forms implies that, for any matrix \(Q^{[\ell]}\) independent of \(\vb{x}_\ell\) with bounded normalized trace norm,
\begin{align}
  \vb{x}_\ell^\top Q^{[\ell]}\vb{x}_\ell
  -
  \frac{1}{D}\Tr(Q^{[\ell]})
  \;\longrightarrow\;0
\end{align}
in probability as \(D\to\infty\).
Applying this to each entry of \eqref{eq:appendix_K_l_exact}, and using the fact that removing a single sample does not affect the normalized traces at leading order, we obtain the deterministic equivalent
\begin{align}
  \mathbf{K}_\ell
  \;\asymp\;
  \bar K(u,v)
  :=
  \begin{pmatrix}
    -m_{12} & -\sqrt v\,m_{13} \\
    0 & -v\,m_{43}
  \end{pmatrix},
  \label{eq:appendix_K_bar}
\end{align}
where \(\asymp\) means equality up to terms vanishing in the high-dimensional limit.

\paragraph{Extracting the \(S\)-inserted traces.}
To derive the equations for the \(s\)-variables, we must compute the quadratic forms that appear when \(\mathbf{G}\) is sandwiched by \(\vb{x}_\ell\).
For this purpose, define
\begin{align}
  L_\ell^\top
  &:=
  \begin{pmatrix}
    \vb{x}_\ell^\top & 0 & 0 & 0 \\
    0 & \vb{x}_\ell^\top & 0 & 0 \\
    0 & 0 & \vb{x}_\ell^\top & 0 \\
    0 & 0 & 0 & \vb{x}_\ell^\top
  \end{pmatrix}
  \in \mathbb R^{4\times 4D},
  \\
  R_\ell
  &:=
  \begin{pmatrix}
    0 & 0 \\
    \vb{x}_\ell & 0 \\
    0 & \vb{x}_\ell \\
    0 & 0
  \end{pmatrix}
  \in \mathbb R^{4D\times 2}.
\end{align}
Then
\begin{align}
  L_\ell^\top \mathbf{G} R_\ell
  =
  \begin{pmatrix}
    \vb{x}_\ell^\top \mathbf{G}_{12}\vb{x}_\ell &
    \vb{x}_\ell^\top \mathbf{G}_{13}\vb{x}_\ell \\
    \vb{x}_\ell^\top \mathbf{G}_{22}\vb{x}_\ell &
    \vb{x}_\ell^\top \mathbf{G}_{23}\vb{x}_\ell \\
    0 &
    \vb{x}_\ell^\top \mathbf{G}_{33}\vb{x}_\ell \\
    0 &
    \vb{x}_\ell^\top \mathbf{G}_{43}\vb{x}_\ell
  \end{pmatrix}.
\end{align}
Averaging over \(\ell\), the definitions of the \(s\)-variables in \eqref{eq:appendix_def_s} yield
\begin{align}
  \frac{1}{L}\sum_{\ell=1}^L L_\ell^\top G R_\ell
  =
  \begin{pmatrix}
    s_{12} & s_{13} \\
    s_{22} & s_{23} \\
    0 & s_{33} \\
    0 & s_{43}
  \end{pmatrix}.
  \label{eq:appendix_s_matrix_average}
\end{align}
Indeed, for example,
\begin{align}
  \frac{1}{L}\sum_{\ell=1}^L \vb{x}_\ell^\top \mathbf{G}_{12}\vb{x}_\ell
  =
  \frac{1}{L}\Tr\!\left(\mathbf{G}_{12}\sum_{\ell=1}^L \vb{x}_\ell\vb{x}_\ell^\top\right)
  =
  \frac{1}{L}\Tr(\mathbf{G}_{12}\mathbf{S})
  =
  \frac{1}{\alpha}\,s_{12}. 
\end{align}
Thus one may equivalently work either with the average over \(\ell\) or directly with the normalized traces.
Below we present the final equations in terms of the normalized quantities \(s_{ab}\) defined in \eqref{eq:appendix_def_s}.

We now apply \eqref{eq:appendix_woodbury_l} between \(L_\ell^\top\) and \(R_\ell\):
\begin{align}
  L_\ell^\top \mathbf{G} R_\ell
  =
  L_\ell^\top \mathbf{G}^{[\ell]} R_\ell
  -
  L_\ell^\top \mathbf{G}^{[\ell]}U_\ell
  (\mathbf{I}_2+\mathbf{K}_\ell)^{-1}
  V_\ell^\top \mathbf{G}^{[\ell]}R_\ell.
  \label{eq:appendix_LGR_l}
\end{align}
Each term can now be evaluated by the same concentration argument as above.

\paragraph{First term.}
Since \(\vb{x}_\ell\) is independent of \(\mathbf{G}^{[\ell]}\), we have
\begin{align}
  L_\ell^\top \mathbf{G}^{[\ell]} R_\ell
  \;\asymp\;
  \frac{1}{D}
  \begin{pmatrix}
    \Tr(\mathbf{G}_{12}^{[\ell]}) & \Tr(\mathbf{G}_{13}^{[\ell]}) \\
    \Tr(\mathbf{G}_{22}^{[\ell]}) & \Tr(\mathbf{G}_{23}^{[\ell]}) \\
    0 & \Tr(\mathbf{G}_{33}^{[\ell]}) \\
    0 & \Tr(\mathbf{G}_{43}^{[\ell]})
  \end{pmatrix}
  \;\asymp\;
  \begin{pmatrix}
    m_{12} & m_{13} \\
    m_{22} & m_{23} \\
    0 & m_{33} \\
    0 & m_{43}
  \end{pmatrix}.
  \label{eq:appendix_first_term_s}
\end{align}

\paragraph{Second term.}
Using \eqref{eq:appendix_UV_l}, we compute
\begin{align}
  L_\ell^\top \mathbf{G}^{[\ell]}U_\ell
  \;\asymp\;
  \begin{pmatrix}
    -m_{12} & -\sqrt v\,m_{13} \\
    -m_{22} & -\sqrt v\,m_{23} \\
    0 & -\sqrt v\,m_{33} \\
    0 & -\sqrt v\,m_{43}
  \end{pmatrix},
\end{align}
and
\begin{align}
  V_\ell^\top \mathbf{G}^{[\ell]}R_\ell
  \;\asymp\;
  \begin{pmatrix}
    m_{12} & m_{13} \\
    0 & v\,m_{43}
  \end{pmatrix}.
\end{align}
Moreover, from \eqref{eq:appendix_K_bar},
\begin{align}
  I_2+\bar K
  =
  \begin{pmatrix}
    1-m_{12} & -\sqrt v\,m_{13} \\
    0 & 1-vm_{43}
  \end{pmatrix},
\end{align}
so its inverse is
\begin{align}
  (I_2+\bar K)^{-1}
  =
  \begin{pmatrix}
    \dfrac{1}{1-m_{12}}
    &
    \dfrac{\sqrt v\,m_{13}}{(1-m_{12})(1-vm_{43})}
    \\[3mm]
    0
    &
    \dfrac{1}{1-vm_{43}}
  \end{pmatrix}.
  \label{eq:appendix_IplusK_inv}
\end{align}
Substituting these expressions into the second term of \eqref{eq:appendix_LGR_l}, we find after a direct matrix multiplication that
\begin{align}
  &L_\ell^\top G^{[\ell]}U_\ell
  (I_2+\bar K)^{-1}
  V_\ell^\top G^{[\ell]}R_\ell
  \notag\\
  &\asymp
  \begin{pmatrix}
    -\dfrac{(m_{12})^2}{1-m_{12}}
    &
    -\dfrac{m_{12}m_{13}}{(1-m_{12})(1-vm_{43})}
    -\dfrac{v\,m_{13}m_{43}}{1-vm_{43}}
    \\[4mm]
    -\dfrac{m_{12}m_{22}}{1-m_{12}}
    &
    -\dfrac{m_{22}m_{13}}{(1-m_{12})(1-vm_{43})}
    -\dfrac{v\,m_{23}m_{43}}{1-vm_{43}}
    \\[4mm]
    0
    &
    -\dfrac{v\,m_{33}m_{43}}{1-vm_{43}}
    \\[4mm]
    0
    &
    -\dfrac{v\,(m_{43})^2}{1-vm_{43}}
  \end{pmatrix}.
  \label{eq:appendix_second_term_s}
\end{align}

\paragraph{Final equations.}
Combining \eqref{eq:appendix_LGR_l}, \eqref{eq:appendix_first_term_s}, and \eqref{eq:appendix_second_term_s}, and then summing over \(\ell\), we obtain the deterministic relations
\begin{align}
  \begin{pmatrix}
    s_{12} & s_{13} \\
    s_{22} & s_{23} \\
    0 & s_{33} \\
    0 & s_{43}
  \end{pmatrix}
  =
  \alpha
  \begin{pmatrix}
    \dfrac{m_{12}}{1-m_{12}}
    &
    \dfrac{m_{13}}{(1-m_{12})(1-vm_{43})}
    \\[4mm]
    \dfrac{m_{22}}{1-m_{12}}
    &
    \dfrac{m_{23}}{1-vm_{43}}
    +
    \dfrac{m_{22}m_{13}}{(1-m_{12})(1-vm_{43})}
    \\[4mm]
    0
    &
    \dfrac{m_{33}}{1-vm_{43}}
    \\[4mm]
    0
    &
    \dfrac{m_{43}}{1-vm_{43}}
  \end{pmatrix}.
  \label{eq:appendix_s_fixed_point_matrix}
\end{align}
Equivalently, componentwise,
\begin{align}
  s_{12}
  &= \alpha\,\frac{m_{12}}{1-m_{12}},
  \label{eq:appendix_s12}\\
  s_{13}
  &= \alpha\,\frac{m_{13}}{(1-m_{12})(1-vm_{43})},
  \label{eq:appendix_s13}\\
  s_{22}
  &= \alpha\,\frac{m_{22}}{1-m_{12}},
  \label{eq:appendix_s22}\\
  s_{23}
  &= \alpha\left(
      \frac{m_{23}}{1-vm_{43}}
      +
      \frac{m_{22}m_{13}}{(1-m_{12})(1-vm_{43})}
    \right),
  \label{eq:appendix_s23}\\
  s_{33}
  &= \alpha\,\frac{m_{33}}{1-vm_{43}},
  \label{eq:appendix_s33}\\
  s_{43}
  &= \alpha\,\frac{m_{43}}{1-vm_{43}}.
  \label{eq:appendix_s43}
\end{align}

These equations complete the treatment of the \(S\)-dependent order parameters.
They show that the six quantities \(s_{ab}\) can be expressed entirely in terms of the primary traces \(m_{ab}\).

\subsection{Cavity equations for the \(B\)-dependent order parameters}
\label{appendix:main_result_cavity_b}

We now turn to the order parameters \(b_{11}\) and \(b_{21}\), which contain the pre-training matrix
\[
\mathbf{B}=\frac1M\sum_{\mu=1}^M \vb{w}_\mu \vb{v}_\mu^\top .
\]
Unlike the \(S\)-dependent quantities treated in the previous subsection, these variables depend on the randomness of the pre-training tasks.
Accordingly, we now perform a leave-one-task-out analysis with respect to the index \(\mu\in\{1,\dots,M\}\).

The main point is that both \(\mathbf{B}\) and \(\mathbf{C}\) depend on the same task variables \((\vb{w}_\mu,\vb{v}_\mu)\), so removing one task produces a perturbation that simultaneously affects two blocks of the extended matrix.
This is why the corresponding cavity matrix has rank \(4\), rather than rank \(2\) as in the \(S\)-cavity analysis.

\paragraph{Leave-one-out decomposition.}
Fix \(\mu\in\{1,\dots,M\}\), and define
\begin{align}
  \mathbf{B}^{[\mu]}
  &:=
  \frac1M\sum_{\nu\ne \mu}\vb{w}_\nu \vb{v}_\nu^\top,
  \\
  \mathbf{C}^{[\mu]}
  &:=
  \frac1M\sum_{\nu\ne \mu}\vb{v}_\nu \vb{v}_\nu^\top+\frac{\lambda}{2}\mathbf{I}.
\end{align}
We then introduce the leave-one-out extended matrix
\begin{align}
  \mathcal M^{[\mu]}(u,v)
  :=
  \begin{pmatrix}
    \mathbf{I} & -u\mathbf{B}^{[\mu]} & -\mathbf{I} & 0 \\
    -\mathbf{S} & \mathbf{C}^{[\mu]} & 0 & 0 \\
    0 & 0 & \mathbf{I} & -v\mathbf{S} \\
    0 & 0 & -\bigl(\mathbf{B}^{[\mu]}\bigr)^\top & \mathbf{C}^{[\mu]}
  \end{pmatrix},
\end{align}
and denote its inverse by
\begin{align}
  \mathbf{G}^{[\mu]}:=\bigl(\mathcal M^{[\mu]}(u,v)\bigr)^{-1}.
\end{align}
Since
\begin{align}
  \mathbf{B}=\mathbf{B}^{[\mu]}+\frac1M \vb{w}_\mu \vb{v}_\mu^\top,
  \qquad
  \mathbf{C}=\mathbf{C}^{[\mu]}+\frac1M \vb{v}_\mu \vb{v}_\mu^\top,
\end{align}
the full matrix can be decomposed as
\begin{align}
  \mathcal M(u,v)=\mathcal M^{[\mu]}(u,v)+\Delta_\mu,
\end{align}
where
\begin{align}
  \Delta_\mu=
  \begin{pmatrix}
    0 & -\dfrac{u}{M}\vb{w}_\mu \vb{v}_\mu^\top & 0 & 0 \\
    0 & \dfrac1M \vb{v}_\mu \vb{v}_\mu^\top & 0 & 0 \\
    0 & 0 & 0 & 0 \\
    0 & 0 & -\dfrac1M \vb{v}_\mu \vb{w}_\mu^\top & \dfrac1M \vb{v}_\mu \vb{v}_\mu^\top
  \end{pmatrix}.
  \label{eq:appendix_delta_mu}
\end{align}
It is convenient to write this perturbation in factorized form
\begin{align}
  \Delta_\mu = U_\mu V_\mu^\top,
\end{align}
with
\begin{align}
  U_\mu
  :=
  \begin{pmatrix}
    -\sqrt{\dfrac{u}{M}}\,\vb{w}_\mu & 0 & 0 & 0 \\
    0 & \dfrac1{\sqrt M}\vb{v}_\mu & 0 & 0 \\
    0 & 0 & 0 & 0 \\
    0 & 0 & -\dfrac1{\sqrt M}\vb{v}_\mu & \dfrac1{\sqrt M}\vb{v}_\mu
  \end{pmatrix}
  \in \mathbb R^{4D\times 4},
\end{align}
and
\begin{align}
  V_\mu
  :=
  \begin{pmatrix}
    0 & 0 & 0 & 0 \\
    \sqrt{\dfrac{u}{M}}\,\vb{v}_\mu & \dfrac1{\sqrt M}\vb{v}_\mu & 0 & 0 \\
    0 & 0 & \dfrac1{\sqrt M}\vb{w}_\mu & 0 \\
    0 & 0 & 0 & \dfrac1{\sqrt M}\vb{v}_\mu
  \end{pmatrix}
  \in \mathbb R^{4D\times 4}.
  \label{eq:appendix_UV_mu}
\end{align}

\paragraph{Woodbury formula and the cavity matrix.}
Applying the Woodbury formula to
\(
\mathcal M=\mathcal M^{[\mu]}+U_\mu V_\mu^\top
\),
we obtain
\begin{align}
  \mathbf{G}
  =
  \mathbf{G}^{[\mu]}
  -
  \mathbf{G}^{[\mu]}U_\mu
  \Bigl(\mathbf{I}_4+V_\mu^\top \mathbf{G}^{[\mu]}U_\mu\Bigr)^{-1}
  V_\mu^\top \mathbf{G}^{[\mu]}.
  \label{eq:appendix_woodbury_mu}
\end{align}
We denote the corresponding \(4\times4\) cavity matrix by
\begin{align}
  \mathbf{J}_\mu:=V_\mu^\top \mathbf{G}^{[\mu]}U_\mu .
\end{align}
Using \eqref{eq:appendix_UV_mu}, we compute
\begin{align}
  \mathbf{J}_\mu
  =
  \begin{pmatrix}
    -\dfrac{u}{M}\vb{v}_\mu^\top \mathbf{G}_{21}^{[\mu]}\vb{w}_\mu
    &
    \dfrac{\sqrt u}{M}\vb{v}_\mu^\top \mathbf{G}_{22}^{[\mu]}\vb{v}_\mu
    &
    -\dfrac{\sqrt u}{M}\vb{v}_\mu^\top \mathbf{G}_{24}^{[\mu]}\vb{v}_\mu
    &
    \dfrac{\sqrt u}{M}\vb{v}_\mu^\top \mathbf{G}_{24}^{[\mu]}\vb{v}_\mu
    \\[3mm]
    -\dfrac{\sqrt u}{M}\vb{v}_\mu^\top \mathbf{G}_{21}^{[\mu]}\vb{w}_\mu
    &
    \dfrac1M\vb{v}_\mu^\top \mathbf{G}_{22}^{[\mu]}\vb{v}_\mu
    &
    -\dfrac1M\vb{v}_\mu^\top \mathbf{G}_{24}^{[\mu]}\vb{v}_\mu
    &
    \dfrac1M\vb{v}_\mu^\top \mathbf{G}_{24}^{[\mu]}\vb{v}_\mu
    \\[3mm]
    0 & 0 &
    -\dfrac1M\vb{w}_\mu^\top \mathbf{G}_{34}^{[\mu]}\vb{v}_\mu
    &
    \dfrac1M\vb{w}_\mu^\top \mathbf{G}_{34}^{[\mu]}\vb{v}_\mu
    \\[3mm]
    0 & 0 &
    -\dfrac1M\vb{v}_\mu^\top \mathbf{G}_{44}^{[\mu]}\vb{v}_\mu
    &
    \dfrac1M\vb{v}_\mu^\top \mathbf{G}_{44}^{[\mu]}\vb{v}_\mu
  \end{pmatrix}.
  \label{eq:appendix_J_mu_exact}
\end{align}

To simplify this expression, we use the moment relations
\eqref{eq:appendix_wv_second_moments}--\eqref{eq:appendix_wv_quadratic_forms}
from the preliminaries.
As in the previous subsection, removing a single task does not affect normalized traces at leading order.
Therefore \(\mathbf{J}_\mu\) is asymptotically equivalent to the deterministic matrix
\begin{align}
  \bar J
  :=
  \begin{pmatrix}
    -\dfrac{u}{\tau}m_{21}
    &
    \dfrac{\sqrt u}{\tau}c\,m_{22}
    &
    -\dfrac{\sqrt u}{\tau}c\,m_{24}
    &
    \dfrac{\sqrt u}{\tau}c\,m_{24}
    \\[3mm]
    -\dfrac{\sqrt u}{\tau}m_{21}
    &
    \dfrac{1}{\tau}c\,m_{22}
    &
    -\dfrac{1}{\tau}c\,m_{24}
    &
    \dfrac{1}{\tau}c\,m_{24}
    \\[3mm]
    0 & 0 &
    -\dfrac{1}{\tau}m_{34}
    &
    \dfrac{1}{\tau}m_{34}
    \\[3mm]
    0 & 0 &
    -\dfrac{c}{\tau}m_{44}
    &
    \dfrac{c}{\tau}m_{44}
  \end{pmatrix},
  \label{eq:appendix_J_bar}
\end{align}
where \(\tau=M/D\).

\paragraph{Choice of test vectors.}
We now derive equations for \(b_{11}\) and \(b_{21}\).
By definition,
\begin{align}
  b_{11}=\frac1D\Tr(\mathbf{G}_{11}\mathbf{B}),
  \qquad
  b_{21}=\frac1D\Tr(\mathbf{G}_{21}\mathbf{B}).
\end{align}
Since \(\mathbf{B}\) is a sum of rank-one terms \(\vb{w}_\nu \vb{v}_\nu^\top/M\), these quantities can be extracted by testing \(\mathbf{G}\) against the vectors \(\vb{w}_\mu\) and \(\vb{v}_\mu\).
More precisely, define
\begin{align}
  {L_\mu^b}^\top
  &:=
  \begin{pmatrix}
    \vb{v}_\mu^\top & 0 & 0 & 0 \\
    0 & \vb{v}_\mu^\top & 0 & 0
  \end{pmatrix}
  \in \mathbb R^{2\times 4D},
  \\
  R_\mu^b
  &:=
  \begin{pmatrix}
    \vb{w}_\mu \\
    0 \\
    0 \\
    0
  \end{pmatrix}
  \in \mathbb R^{4D\times 1}.
\end{align}
Then
\begin{align}
  {L_\mu^b}^\top \mathbf{G} R_\mu^b
  =
  \begin{pmatrix}
    \vb{v}_\mu^\top \mathbf{G}_{11}\vb{w}_\mu \\
    \vb{v}_\mu^\top \mathbf{G}_{21}\vb{w}_\mu
  \end{pmatrix},
\end{align}
so after normalization and averaging over \(\mu\), the limit is exactly \((b_{11},b_{21})^\top\).

Applying \eqref{eq:appendix_woodbury_mu} between \({L_\mu^b}^\top\) and \(R_\mu^b\), we obtain
\begin{align}
  {L_\mu^b}^\top \mathbf{G} R_\mu^b
  =
  {L_\mu^b}^\top \mathbf{G}^{[\mu]} R_\mu^b
  -
  \bigl({L_\mu^b}^\top \mathbf{G}^{[\mu]}U_\mu\bigr)
  (\mathbf{I}_4+\mathbf{J}_\mu)^{-1}
  \bigl(V_\mu^\top \mathbf{G}^{[\mu]}R_\mu^b\bigr).
  \label{eq:appendix_b_master}
\end{align}

\paragraph{Evaluation of each factor.}
We now evaluate the three factors on the right-hand side of \eqref{eq:appendix_b_master}.

First, by the quadratic-form concentration stated in
\eqref{eq:appendix_wv_quadratic_forms} and the independence between
\((\vb{w}_\mu,\vb{v}_\mu)\) and the leave-one-out resolvent \(\mathbf{G}^{[\mu]}\),
\begin{align}
  \frac1D {L_\mu^b}^\top \mathbf{G}^{[\mu]} R_\mu^b
  \;\asymp\;
  \begin{pmatrix}
    m_{11} \\
    m_{21}
  \end{pmatrix}.
  \label{eq:appendix_b_first}
\end{align}
Indeed, for the first component,
\(
D^{-1}\vb{v}_\mu^\top \mathbf{G}_{11}^{[\mu]}\vb{w}_\mu
\to D^{-1}\Tr(\mathbf{G}_{11}^{[\mu]})
\),
where we used \(\E[\vb{w}_\mu \vb{v}_\mu^\top]=\mathbf{I}_D\) from
\eqref{eq:appendix_wv_covariances}.
The second component is treated in the same way.

Next, a direct multiplication using \eqref{eq:appendix_UV_mu} gives
\begin{align}
  \frac1D {L_\mu^b}^\top \mathbf{G}^{[\mu]}U_\mu
  \;\asymp\;
  \frac1{\sqrt M}
  \begin{pmatrix}
    -\sqrt u\,m_{11} & c\,m_{12} & -c\,m_{14} & c\,m_{14} \\
    -\sqrt u\,m_{21} & c\,m_{22} & -c\,m_{24} & c\,m_{24}
  \end{pmatrix},
  \label{eq:appendix_b_second}
\end{align}
and similarly
\begin{align}
  \frac1D V_\mu^\top \mathbf{G}^{[\mu]}R_\mu^b
  \;\asymp\;
  \frac1{\sqrt M}
  \begin{pmatrix}
    \sqrt u\,m_{21} \\
    m_{21} \\
    0 \\
    0
  \end{pmatrix}.
  \label{eq:appendix_b_third}
\end{align}

At this point, only the first two columns of \((\mathbf{I}_4+\bar J)^{-1}\) contribute, because the vector in \eqref{eq:appendix_b_third} has vanishing third and fourth entries.
Therefore only the upper-left \(2\times2\) block matters.
From \eqref{eq:appendix_J_bar}, this block is
\begin{align}
  \begin{pmatrix}
    1-\dfrac{u}{\tau}m_{21} & \dfrac{\sqrt u}{\tau}c\,m_{22} \\
    -\dfrac{\sqrt u}{\tau}m_{21} & 1+\dfrac{c}{\tau}m_{22}
  \end{pmatrix},
\end{align}
whose determinant is
\begin{align}
  \frac{1}{\tau}\Bigl(\tau-u\,m_{21}+c\,m_{22}\Bigr).
\end{align}
We therefore introduce the shorthand
\begin{align}
  D_1:=\tau-u\,m_{21}+c\,m_{22}.
  \label{eq:appendix_D1_def}
\end{align}
A short calculation then shows that
\begin{align}
  \bigl({L_\mu^b}^\top \mathbf{G}^{[\mu]}U_\mu\bigr)
  (\mathbf{I}_4+\bar J)^{-1}
  \bigl(V_\mu^\top \mathbf{G}^{[\mu]}R_\mu^b\bigr)
  \;\asymp\;
  \frac1\tau
  \begin{pmatrix}
    \dfrac{m_{21}(c\,m_{12}-u\,m_{11})}{D_1} \\
    \dfrac{m_{21}(c\,m_{22}-u\,m_{21})}{D_1}
  \end{pmatrix}.
\end{align}
Substituting this together with \eqref{eq:appendix_b_first} into \eqref{eq:appendix_b_master} yields the closed equations for \(b_{11}\) and \(b_{21}\).

\paragraph{Final equations.}
We conclude that, in the high-dimensional limit,
\begin{align}
  b_{11}
  &=
  m_{11}
  -
  \frac{m_{21}(c\,m_{12}-u\,m_{11})}
       {\tau-u\,m_{21}+c\,m_{22}},
  \label{eq:appendix_b11}
  \\
  b_{21}
  &=
  m_{21}
  -
  \frac{m_{21}(c\,m_{22}-u\,m_{21})}
       {\tau-u\,m_{21}+c\,m_{22}}.
  \label{eq:appendix_b21}
\end{align}

These relations express the \(B\)-dependent order parameters entirely in terms of the primary traces \(m_{ab}\).
Together with the identities derived from \(G\mathcal M=I\), they already determine the \(u\)-sector once the \(C\)-dependent quantities are also closed.

\subsection{Cavity equations for the \(B^\top\)-dependent order parameters}
\label{appendix:main_result_cavity_beta}

We next derive closed equations for the quantities
\(
\beta_{14},\beta_{24},\beta_{34},\beta_{44}
\),
which contain the transpose matrix
\[
\mathbf{B}^\top=\frac1M\sum_{\mu=1}^M \vb{v}_\mu \vb{w}_\mu^\top .
\]
As in the previous subsection, the randomness enters through the pre-training tasks \(\{(\vb{w}_\mu,\vb{v}_\mu)\}_{\mu=1}^M\), and we therefore use the same leave-one-task-out decomposition.
The underlying mechanism is the same as for the \(b\)-variables, but now we must probe the third and fourth block columns of the resolvent, because \(\mathbf{B}^\top\) appears in the \((4,3)\)-block of the extended matrix \(\mathcal M(u,v)\).

\paragraph{Setup.}
We keep the leave-one-task-out notation introduced in the previous subsection.
For each fixed \(\mu\), we write
\begin{align}
  \mathcal M(u,v)=\mathcal M^{[\mu]}(u,v)+U_\mu V_\mu^\top,
\end{align}
with \(U_\mu\), \(V_\mu\), and \(\mathbf{G}^{[\mu]}=\bigl(\mathcal M^{[\mu]}(u,v)\bigr)^{-1}\) defined in
\eqref{eq:appendix_UV_mu}.
The corresponding Woodbury formula is
\begin{align}
  \mathbf{G}
  =
  \mathbf{G}^{[\mu]}
  -
  \mathbf{G}^{[\mu]}U_\mu
  \bigl(\mathbf{I}_4+\mathbf{J}_\mu\bigr)^{-1}
  V_\mu^\top \mathbf{G}^{[\mu]},
  \qquad
  \mathbf{J}_\mu:=V_\mu^\top \mathbf{G}^{[\mu]}U_\mu,
\end{align}
and, in the high-dimensional limit, \(\mathbf{J}_\mu\) can be replaced by its deterministic equivalent \(\bar J\) given in \eqref{eq:appendix_J_bar}.

Our goal is to evaluate
\begin{align}
  \beta_{14}=\frac1D\Tr(\mathbf{G}_{14}\mathbf{B}^\top),
  \qquad
  \beta_{24}=\frac1D\Tr(\mathbf{G}_{24}\mathbf{B}^\top),
  \qquad
  \beta_{34}=\frac1D\Tr(\mathbf{G}_{34}\mathbf{B}^\top),
  \qquad
  \beta_{44}=\frac1D\Tr(\mathbf{G}_{44}\mathbf{B}^\top).
\end{align}
Since \(\mathbf{B}^\top\) is a sum of rank-one terms \(\vb{v}_\mu \vb{w}_\mu^\top/M\), these quantities can again be extracted by testing the resolvent against the vectors \(\vb{w}_\mu\) and \(\vb{v}_\mu\).

\paragraph{Choice of test vectors.}
Define
\begin{align}
  {L_\mu^\beta}^\top
  &:=
  \begin{pmatrix}
    \vb{w}_\mu^\top & 0 & 0 & 0 \\
    0 & \vb{w}_\mu^\top & 0 & 0 \\
    0 & 0 & \vb{w}_\mu^\top & 0 \\
    0 & 0 & 0 & \vb{w}_\mu^\top
  \end{pmatrix}
  \in\mathbb R^{4\times 4D},
  \\
  R_\mu^\beta
  &:=
  \begin{pmatrix}
    0\\
    0\\
    0\\
    \vb{v}_\mu
  \end{pmatrix}
  \in\mathbb R^{4D\times 1}.
\end{align}
Then
\begin{align}
  {L_\mu^\beta}^\top \mathbf{G} R_\mu^\beta
  =
  \begin{pmatrix}
    \vb{w}_\mu^\top \mathbf{G}_{14}\vb{v}_\mu \\
    \vb{w}_\mu^\top \mathbf{G}_{24}\vb{v}_\mu \\
    \vb{w}_\mu^\top \mathbf{G}_{34}\vb{v}_\mu \\
    \vb{w}_\mu^\top \mathbf{G}_{44}\vb{v}_\mu
  \end{pmatrix},
\end{align}
so that, after normalization and averaging over \(\mu\), the limiting vector is precisely
\(
(\beta_{14},\beta_{24},\beta_{34},\beta_{44})^\top
\).

Applying the Woodbury formula between \({L_\mu^\beta}^\top\) and \(R_\mu^\beta\), we obtain
\begin{align}
  {L_\mu^\beta}^\top \mathbf{G} R_\mu^\beta
  =
  {L_\mu^\beta}^\top \mathbf{G}^{[\mu]} R_\mu^\beta
  -
  \bigl({L_\mu^\beta}^\top \mathbf{G}^{[\mu]}U_\mu\bigr)
  (\mathbf{I}_4+\mathbf{J}_\mu)^{-1}
  \bigl(V_\mu^\top \mathbf{G}^{[\mu]}R_\mu^\beta\bigr).
  \label{eq:appendix_beta_master}
\end{align}

\paragraph{Evaluation of the three factors.}
We now analyze each term on the right-hand side of \eqref{eq:appendix_beta_master}.

First, by the same quadratic-form concentration
\eqref{eq:appendix_wv_quadratic_forms} and the leave-one-out argument as before,
\begin{align}
  \frac1D {L_\mu^\beta}^\top \mathbf{G}^{[\mu]} R_\mu^\beta
  \;\asymp\;
  \begin{pmatrix}
    m_{14}\\
    m_{24}\\
    m_{34}\\
    m_{44}
  \end{pmatrix}.
  \label{eq:appendix_beta_first}
\end{align}
Indeed, the \(\mu\)-th task is independent of \(\mathbf{G}^{[\mu]}\), and replacing quadratic forms by their normalized traces yields the corresponding block traces of the leave-one-out resolvent.

Next, using the explicit form of \(U_\mu\) in \eqref{eq:appendix_UV_mu}, we obtain
\begin{align}
  \frac1D {L_\mu^\beta}^\top \mathbf{G}^{[\mu]}U_\mu
  \;\asymp\;
  \frac1{\sqrt M}
  \begin{pmatrix}
    -\sqrt u\,m_{11} & m_{12} & -m_{14} & m_{14} \\
    -\sqrt u\,m_{21} & m_{22} & -m_{24} & m_{24} \\
    0 & 0 & -m_{34} & m_{34} \\
    0 & 0 & -m_{44} & m_{44}
  \end{pmatrix}.
  \label{eq:appendix_beta_second}
\end{align}
Here the first two columns come from the first two block columns of \(U_\mu\), which contain \(\vb{w}_\mu\) and \(\vb{v}_\mu\), while the last two columns reflect the two opposite contributions in the fourth block row of \(U_\mu\).

Similarly, using the explicit form of \(V_\mu\), we find
\begin{align}
  \frac1D V_\mu^\top \mathbf{G}^{[\mu]}R_\mu^\beta
  \;\asymp\;
  \frac1{\sqrt M}
  \begin{pmatrix}
    \sqrt u\,c\,m_{24}\\
    c\,m_{24}\\
    m_{34}\\
    c\,m_{44}
  \end{pmatrix}.
  \label{eq:appendix_beta_third}
\end{align}
The factor \(c\) appears whenever \(\vb{v}_\mu\) is contracted with itself, reflecting the asymptotic norm
\(
D^{-1}\|\vb{v}_\mu\|^2\to c
\).

\paragraph{Reduction of \((\mathbf{I}_4+\bar J)^{-1}\).}
To simplify the cavity correction in \eqref{eq:appendix_beta_master}, we use the block upper-triangular structure of \(\bar J\).
Introduce the shorthand
\begin{align}
  D_1 &:= \tau-u\,m_{21}+c\,m_{22},
  \\
  D_2 &:= \tau-m_{34}+c\,m_{44}.
  \label{eq:appendix_D2_def}
\end{align}
These quantities are the determinants of the relevant \(2\times2\) blocks of \(\mathbf{I}_4+\bar J\), up to the common factor \(\tau^{-1}\).
More precisely, the upper-left block of \(\mathbf{I}_4+\bar J\) has determinant \(D_1/\tau\), while the lower-right block has determinant \(D_2/\tau\).

Because the vector \eqref{eq:appendix_beta_third} has nonzero entries in both the first two and the last two components, both sectors contribute here.
A direct but straightforward matrix multiplication using \eqref{eq:appendix_beta_second}, \eqref{eq:appendix_beta_third}, and the explicit inverse of \(\mathbf{I}_4+\bar J\) yields the cavity correction.
After simplifying the resulting expressions, we obtain the following closed equations.

\paragraph{Final equations.}
In the high-dimensional limit,
\begin{align}
  \beta_{14}
  &=
  m_{14}
  -
  \frac{c\,m_{24}(m_{12}-u\,m_{11})}{D_1D_2/\tau}
  -
  \frac{m_{14}(c\,m_{44}-m_{34})}{D_2}
  \notag\\
  &=
  m_{14}
  -
  \frac{\tau\,c\,m_{24}(m_{12}-u\,m_{11})}{D_1D_2}
  -
  \frac{m_{14}(c\,m_{44}-m_{34})}{D_2},
  \label{eq:appendix_beta14}
  \\
  \beta_{24}
  &=
  m_{24}
  -
  \frac{\tau\,c\,m_{24}(m_{22}-u\,m_{21})}{D_1D_2}
  -
  \frac{m_{24}(c\,m_{44}-m_{34})}{D_2},
  \label{eq:appendix_beta24}
  \\
  \beta_{34}
  &=
  m_{34}
  -
  \frac{m_{34}(c\,m_{44}-m_{34})}{D_2},
  \label{eq:appendix_beta34}
  \\
  \beta_{44}
  &=
  m_{44}
  -
  \frac{m_{44}(c\,m_{44}-m_{34})}{D_2}.
  \label{eq:appendix_beta44}
\end{align}

Equivalently, writing out the denominators explicitly,
\begin{align}
  \beta_{14}
  &=
  m_{14}
  -
  \frac{c\,m_{24}(m_{12}-u\,m_{11})}
       {\tau\left(1-\frac{u}{\tau}m_{21}+\frac{c}{\tau}m_{22}\right)
             \left(1-\frac{1}{\tau}m_{34}+\frac{c}{\tau}m_{44}\right)}
  -
  \frac{m_{14}(c\,m_{44}-m_{34})}
       {\tau\left(1-\frac{1}{\tau}m_{34}+\frac{c}{\tau}m_{44}\right)},
  \\
  \beta_{24}
  &=
  m_{24}
  -
  \frac{c\,m_{24}(m_{22}-u\,m_{21})}
       {\tau\left(1-\frac{u}{\tau}m_{21}+\frac{c}{\tau}m_{22}\right)
             \left(1-\frac{1}{\tau}m_{34}+\frac{c}{\tau}m_{44}\right)}
  -
  \frac{m_{24}(c\,m_{44}-m_{34})}
       {\tau\left(1-\frac{1}{\tau}m_{34}+\frac{c}{\tau}m_{44}\right)},
  \\
  \beta_{34}
  &=
  m_{34}
  -
  \frac{m_{34}(c\,m_{44}-m_{34})}
       {\tau\left(1-\frac{1}{\tau}m_{34}+\frac{c}{\tau}m_{44}\right)},
  \\
  \beta_{44}
  &=
  m_{44}
  -
  \frac{m_{44}(c\,m_{44}-m_{34})}
       {\tau\left(1-\frac{1}{\tau}m_{34}+\frac{c}{\tau}m_{44}\right)}.
\end{align}

These relations close the \(\mathbf{B}^\top\)-dependent order parameters in terms of the primary traces \(m_{ab}\).
Together with the equations for \(s\) and \(b\), they determine the contributions coming from the blocks involving \(\mathbf{S}\), \(\mathbf{B}\), and \(\mathbf{B}^\top\).

\subsection{Cavity equations for the \(C\)-dependent order parameters}
\label{appendix:main_result_cavity_c}

We finally derive closed equations for the quantities
\(
c_{12},c_{14},c_{22},c_{24},c_{34},c_{44}
\),
which contain the matrix
\[
\mathbf{C}=\frac1M\sum_{\mu=1}^M \vb{v}_\mu \vb{v}_\mu^\top+\frac{\lambda}{2}\mathbf{I}.
\]
This subsection completes the closure of the order-parameter system.
As in the previous two subsections, we use a leave-one-task-out decomposition with respect to the pre-training index \(\mu\).
The derivation closely parallels the one for the \(\mathbf{B}^\top\)-dependent quantities, but here both left and right test vectors involve \(\vb{v}_\mu\), so the asymptotic norm factor \(c\) enters more systematically.

\paragraph{Setup.}
We keep the leave-one-task-out notation introduced above.
For each fixed \(\mu\), recall that
\begin{align}
  \mathcal M(u,v)=\mathcal M^{[\mu]}(u,v)+U_\mu V_\mu^\top,
\end{align}
with \(U_\mu\), \(V_\mu\), and \(\mathbf{G}^{[\mu]}=\bigl(\mathcal M^{[\mu]}(u,v)\bigr)^{-1}\) defined in \eqref{eq:appendix_UV_mu}.
The associated Woodbury formula is
\begin{align}
  \mathbf{G}
  =
  \mathbf{G}^{[\mu]}
  -
  \mathbf{G}^{[\mu]}U_\mu
  \bigl(\mathbf{I}_4+\mathbf{J}_\mu\bigr)^{-1}
  V_\mu^\top \mathbf{G}^{[\mu]},
\end{align}
where \(\mathbf{J}_\mu=V_\mu^\top \mathbf{G}^{[\mu]}U_\mu\), and in the high-dimensional limit we may replace \(\mathbf{J}_\mu\) by the deterministic matrix \(\bar J\) given in \eqref{eq:appendix_J_bar}.

Our aim is to evaluate the mixed traces involving \(C\).
Since
\[
\mathbf{C}=\frac1M\sum_{\mu=1}^M \vb{v}_\mu \vb{v}_\mu^\top+\frac{\lambda}{2}\mathbf{I},
\]
each term \(\Tr(\mathbf{G}_{ab}\mathbf{C})\) splits into a contribution from the empirical covariance of the \(\vb{v}_\mu\)'s and a deterministic ridge contribution \(\frac{\lambda}{2}\Tr(\mathbf{G}_{ab})\).
For this reason, the cavity equations naturally determine the combinations
\[
c_{ab}-\frac{\lambda}{2}m_{ab}.
\]

\paragraph{Choice of test vectors.}
To extract the \(C\)-dependent quantities, define
\begin{align}
  {L_\mu^c}^\top
  &:=
  \begin{pmatrix}
    \vb{v}_\mu^\top & 0 & 0 & 0 \\
    0 & \vb{v}_\mu^\top & 0 & 0 \\
    0 & 0 & \vb{v}_\mu^\top & 0 \\
    0 & 0 & 0 & \vb{v}_\mu^\top
  \end{pmatrix}
  \in \mathbb R^{4\times 4D},
  \\
  R_\mu^c
  &:=
  \begin{pmatrix}
    0 & 0 \\
    \vb{v}_\mu & 0 \\
    0 & 0 \\
    0 & \vb{v}_\mu
  \end{pmatrix}
  \in \mathbb R^{4D\times 2}.
\end{align}
Then
\begin{align}
  {L_\mu^c}^\top \mathbf{G} R_\mu^c
  =
  \begin{pmatrix}
    \vb{v}_\mu^\top \mathbf{G}_{12}\vb{v}_\mu &
    \vb{v}_\mu^\top \mathbf{G}_{14}\vb{v}_\mu \\
    \vb{v}_\mu^\top \mathbf{G}_{22}\vb{v}_\mu &
    \vb{v}_\mu^\top \mathbf{G}_{24}\vb{v}_\mu \\
    0 &
    \vb{v}_\mu^\top \mathbf{G}_{34}\vb{v}_\mu \\
    0 &
    \vb{v}_\mu^\top \mathbf{G}_{44}\vb{v}_\mu
  \end{pmatrix}.
\end{align}
Averaging over \(\mu\), the empirical covariance part becomes
\[
\frac1M\sum_{\mu=1}^M \vb{v}_\mu \vb{v}_\mu^\top
=
\mathbf{C}-\frac{\lambda}{2}\mathbf{I},
\]
and therefore
\begin{align}
  \frac1D\frac1M\sum_{\mu=1}^M {L_\mu^c}^\top \mathbf{G} R_\mu^c
  =
  \begin{pmatrix}
    c_{12}-\frac{\lambda}{2}m_{12} &
    c_{14}-\frac{\lambda}{2}m_{14} \\
    c_{22}-\frac{\lambda}{2}m_{22} &
    c_{24}-\frac{\lambda}{2}m_{24} \\
    0 &
    c_{34}-\frac{\lambda}{2}m_{34} \\
    0 &
    c_{44}-\frac{\lambda}{2}m_{44}
  \end{pmatrix}.
  \label{eq:appendix_c_lhs}
\end{align}

Applying the Woodbury formula between \({L_\mu^c}^\top\) and \(R_\mu^c\), we obtain
\begin{align}
  {L_\mu^c}^\top \mathbf{G} R_\mu^c
  =
  {L_\mu^c}^\top \mathbf{G}^{[\mu]} R_\mu^c
  -
  \bigl({L_\mu^c}^\top \mathbf{G}^{[\mu]}U_\mu\bigr)
  (\mathbf{I}_4+\mathbf{J}_\mu)^{-1}
  \bigl(V_\mu^\top \mathbf{G}^{[\mu]}R_\mu^c\bigr).
  \label{eq:appendix_c_master}
\end{align}

\paragraph{Evaluation of the three factors.}
We now evaluate each term on the right-hand side of \eqref{eq:appendix_c_master}.

First, since \(\vb{v}_\mu\) is independent of the leave-one-out resolvent \(\mathbf{G}^{[\mu]}\), concentration of quadratic forms gives
\begin{align}
  \frac1D {L_\mu^c}^\top \mathbf{G}^{[\mu]} R_\mu^c
  \;\asymp\;
  \begin{pmatrix}
    c\,m_{12} & c\,m_{14} \\
    c\,m_{22} & c\,m_{24} \\
    0 & c\,m_{34} \\
    0 & c\,m_{44}
  \end{pmatrix}.
  \label{eq:appendix_c_first}
\end{align}
Here each factor \(c\) comes from the asymptotic relation
\(
D^{-1}\|\vb{v}_\mu\|^2\to c
\).

Next, using the explicit form of \(U_\mu\) in \eqref{eq:appendix_UV_mu}, we find
\begin{align}
  \frac1D {L_\mu^c}^\top \mathbf{G}^{[\mu]}U_\mu
  \;\asymp\;
  \frac1{\sqrt M}
  \begin{pmatrix}
    -\sqrt u\,m_{11} & c\,m_{12} & -c\,m_{14} & c\,m_{14} \\
    -\sqrt u\,m_{21} & c\,m_{22} & -c\,m_{24} & c\,m_{24} \\
    0 & 0 & -c\,m_{34} & c\,m_{34} \\
    0 & 0 & -c\,m_{44} & c\,m_{44}
  \end{pmatrix}.
  \label{eq:appendix_c_second}
\end{align}
Similarly, using the explicit form of \(V_\mu\), we obtain
\begin{align}
  \frac1D V_\mu^\top \mathbf{G}^{[\mu]}R_\mu^c
  \;\asymp\;
  \frac1{\sqrt M}
  \begin{pmatrix}
    \sqrt u\,c\,m_{22} & \sqrt u\,c\,m_{24} \\
    c\,m_{22} & c\,m_{24} \\
    0 & m_{34} \\
    0 & c\,m_{44}
  \end{pmatrix}.
  \label{eq:appendix_c_third}
\end{align}

As in the previous subsection, the inverse of \(\mathbf{I}_4+\bar J\) is naturally expressed in terms of the two scalar combinations
\begin{align}
  D_1:=\tau-u\,m_{21}+c\,m_{22},
  \qquad
  D_2:=\tau-m_{34}+c\,m_{44}.
\end{align}
The upper-left and lower-right \(2\times2\) blocks of \(\mathbf{I}_4+\bar J\) have determinants \(D_1/\tau\) and \(D_2/\tau\), respectively.
Using this structure and multiplying the matrices in \eqref{eq:appendix_c_second} and \eqref{eq:appendix_c_third}, one obtains the cavity correction term explicitly.

\paragraph{Final equations.}
Substituting the resulting expression into \eqref{eq:appendix_c_master}, averaging over \(\mu\), and using \eqref{eq:appendix_c_lhs}, we arrive at the following closed system:
\begin{align}
  c_{12}-\frac{\lambda}{2}m_{12}
  &=
  c\,m_{12}
  -
  \frac{1}{\tau}
  \frac{c\,m_{22}(c\,m_{12}-u\,m_{11})}
       {1-\frac{u}{\tau}m_{21}+\frac{c}{\tau}m_{22}},
  \label{eq:appendix_c12}
  \\
  c_{22}-\frac{\lambda}{2}m_{22}
  &=
  c\,m_{22}
  -
  \frac{1}{\tau}
  \frac{c\,m_{22}(c\,m_{22}-u\,m_{21})}
       {1-\frac{u}{\tau}m_{21}+\frac{c}{\tau}m_{22}},
  \label{eq:appendix_c22}
  \\
  c_{14}-\frac{\lambda}{2}m_{14}
  &=
  c\,m_{14}
  -
  \frac{1}{\tau}
  \left[
    \frac{c\,m_{24}(c\,m_{12}-u\,m_{11})}
         {\left(1-\frac{u}{\tau}m_{21}+\frac{c}{\tau}m_{22}\right)
          \left(1-\frac{1}{\tau}m_{34}+\frac{c}{\tau}m_{44}\right)}
    +
    \frac{c\,m_{14}(c\,m_{44}-m_{34})}
         {1-\frac{1}{\tau}m_{34}+\frac{c}{\tau}m_{44}}
  \right],
  \label{eq:appendix_c14}
  \\
  c_{24}-\frac{\lambda}{2}m_{24}
  &=
  c\,m_{24}
  -
  \frac{1}{\tau}
  \left[
    \frac{c\,m_{24}(c\,m_{22}-u\,m_{21})}
         {\left(1-\frac{u}{\tau}m_{21}+\frac{c}{\tau}m_{22}\right)
          \left(1-\frac{1}{\tau}m_{34}+\frac{c}{\tau}m_{44}\right)}
    +
    \frac{c\,m_{24}(c\,m_{44}-m_{34})}
         {1-\frac{1}{\tau}m_{34}+\frac{c}{\tau}m_{44}}
  \right],
  \label{eq:appendix_c24}
  \\
  c_{34}-\frac{\lambda}{2}m_{34}
  &=
  c\,m_{34}
  -
  \frac{1}{\tau}
  \frac{c\,m_{34}(c\,m_{44}-m_{34})}
       {1-\frac{1}{\tau}m_{34}+\frac{c}{\tau}m_{44}},
  \label{eq:appendix_c34}
  \\
  c_{44}-\frac{\lambda}{2}m_{44}
  &=
  c\,m_{44}
  -
  \frac{1}{\tau}
  \frac{c\,m_{44}(c\,m_{44}-m_{34})}
       {1-\frac{1}{\tau}m_{34}+\frac{c}{\tau}m_{44}}.
  \label{eq:appendix_c44}
\end{align}

Equivalently, using \(D_1=\tau-u\,m_{21}+c\,m_{22}\) and \(D_2=\tau-m_{34}+c\,m_{44}\), the same equations can be written in the slightly shorter form
\begin{align}
  c_{12}-\frac{\lambda}{2}m_{12}
  &=
  c\,m_{12}
  -
  \frac{c\,m_{22}(c\,m_{12}-u\,m_{11})}{D_1},
  \\
  c_{22}-\frac{\lambda}{2}m_{22}
  &=
  c\,m_{22}
  -
  \frac{c\,m_{22}(c\,m_{22}-u\,m_{21})}{D_1},
  \\
  c_{14}-\frac{\lambda}{2}m_{14}
  &=
  c\,m_{14}
  -
  \frac{\tau\,c\,m_{24}(c\,m_{12}-u\,m_{11})}{D_1D_2}
  -
  \frac{c\,m_{14}(c\,m_{44}-m_{34})}{D_2},
  \\
  c_{24}-\frac{\lambda}{2}m_{24}
  &=
  c\,m_{24}
  -
  \frac{\tau\,c\,m_{24}(c\,m_{22}-u\,m_{21})}{D_1D_2}
  -
  \frac{c\,m_{24}(c\,m_{44}-m_{34})}{D_2},
  \\
  c_{34}-\frac{\lambda}{2}m_{34}
  &=
  c\,m_{34}
  -
  \frac{c\,m_{34}(c\,m_{44}-m_{34})}{D_2},
  \\
  c_{44}-\frac{\lambda}{2}m_{44}
  &=
  c\,m_{44}
  -
  \frac{c\,m_{44}(c\,m_{44}-m_{34})}{D_2}.
\end{align}

This completes the cavity analysis.
At this point, the basic identities from \(G\mathcal M=I\), together with the cavity equations for \(s\), \(b\), \(\beta\), and \(c\), form a closed deterministic system for the primary traces \(m_{ab}\).

\subsection{Closed self-consistent equations for the primary order parameters}
\label{appendix:main_result_closed_system}

We now combine the exact identities from \(G\mathcal M(u,v)=I\) with the cavity equations derived in the previous subsections.
The goal of this subsection is to eliminate the auxiliary quantities
\[
s_{ab},\qquad b_{ab},\qquad \beta_{ab},\qquad c_{ab},
\]
and obtain a closed system for the primary traces
\[
m_{11},m_{12},m_{21},m_{22},\qquad
m_{33},m_{34},m_{43},m_{44},\qquad
m_{13},m_{14},m_{23},m_{24}.
\]
A useful structural point is that this system is triangular.
Indeed, the first group \((m_{11},m_{12},m_{21},m_{22})\) depends only on \(u\), the second group \((m_{33},m_{34},m_{43},m_{44})\) depends only on \(v\), and the mixed group \((m_{13},m_{14},m_{23},m_{24})\) becomes linear once the first two groups are known.
This structure is what later makes the explicit solution possible.

\paragraph{The \(u\)-sector.}
We begin with the equations involving
\(
m_{11},m_{12},m_{21},m_{22}
\).
From the basic identities
\eqref{eq:appendix_basic_111},
\eqref{eq:appendix_basic_112},
\eqref{eq:appendix_basic_221},
and
\eqref{eq:appendix_basic_222},
we have
\begin{align}
  m_{11}-s_{12} &= 1,
  \label{eq:appendix_u_basic_1}\\
  -u\,b_{11}+c_{12} &= 0,
  \label{eq:appendix_u_basic_2}\\
  m_{21}-s_{22} &= 0,
  \label{eq:appendix_u_basic_3}\\
  -u\,b_{21}+c_{22} &= 1.
  \label{eq:appendix_u_basic_4}
\end{align}
We first substitute the \(S\)-cavity relations
\eqref{eq:appendix_s12} and \eqref{eq:appendix_s22} into
\eqref{eq:appendix_u_basic_1} and \eqref{eq:appendix_u_basic_3}.
This gives
\begin{align}
  m_{11} &= 1+\alpha\,\frac{m_{12}}{1-m_{12}},
  \label{eq:appendix_m11_closed}\\
  m_{21} &= \alpha\,\frac{m_{22}}{1-m_{12}}.
  \label{eq:appendix_m21_closed}
\end{align}

Next we use the cavity equations for \(b\) and \(c\).
From \eqref{eq:appendix_b11} and \eqref{eq:appendix_c12}, we obtain
\begin{align}
  b_{11}
  &=
  m_{11}
  -
  \frac{m_{21}(c\,m_{12}-u\,m_{11})}
       {\tau-u\,m_{21}+c\,m_{22}},
  \\
  c_{12}
  &=
  \frac{\lambda}{2}m_{12}
  +
  c\,m_{12}
  -
  \frac{c\,m_{22}(c\,m_{12}-u\,m_{11})}
       {\tau-u\,m_{21}+c\,m_{22}}.
\end{align}
Substituting these into \eqref{eq:appendix_u_basic_2} and collecting terms yields
\begin{align}
  -u\,m_{11}
  +\left(c+\frac{\lambda}{2}\right)m_{12}
  +\frac{(u\,m_{21}-c\,m_{22})(c\,m_{12}-u\,m_{11})}
        {\tau-u\,m_{21}+c\,m_{22}}
  =0.
  \label{eq:appendix_m12_closed}
\end{align}
Similarly, substituting \eqref{eq:appendix_b21} and \eqref{eq:appendix_c22} into \eqref{eq:appendix_u_basic_4}, we obtain
\begin{align}
  -u\,m_{21}
  +\left(c+\frac{\lambda}{2}\right)m_{22}
  +\frac{(u\,m_{21}-c\,m_{22})(c\,m_{22}-u\,m_{21})}
        {\tau-u\,m_{21}+c\,m_{22}}
  =1.
  \label{eq:appendix_m22_closed}
\end{align}
Thus the \(u\)-sector is closed by
\eqref{eq:appendix_m11_closed}--\eqref{eq:appendix_m22_closed}.

\paragraph{The \(v\)-sector.}
We next consider
\(
m_{33},m_{34},m_{43},m_{44}
\).
From the basic identities
\eqref{eq:appendix_basic_333},
\eqref{eq:appendix_basic_334},
\eqref{eq:appendix_basic_443},
and
\eqref{eq:appendix_basic_444},
we have
\begin{align}
  m_{33}-\beta_{34} &= 1,
  \label{eq:appendix_v_basic_1}\\
  -v\,s_{33}+c_{34} &= 0,
  \label{eq:appendix_v_basic_2}\\
  m_{43}-\beta_{44} &= 0,
  \label{eq:appendix_v_basic_3}\\
  -v\,s_{43}+c_{44} &= 1.
  \label{eq:appendix_v_basic_4}
\end{align}
Using the cavity relations
\eqref{eq:appendix_beta34} and \eqref{eq:appendix_beta44}, we can rewrite
\eqref{eq:appendix_v_basic_1} and \eqref{eq:appendix_v_basic_3} as
\begin{align}
  m_{33}
  &=
  1+m_{34}
  -
  \frac{m_{34}(c\,m_{44}-m_{34})}
       {\tau-m_{34}+c\,m_{44}},
  \label{eq:appendix_m33_closed}\\
  m_{43}
  &=
  m_{44}
  -
  \frac{m_{44}(c\,m_{44}-m_{34})}
       {\tau-m_{34}+c\,m_{44}}.
  \label{eq:appendix_m43_closed}
\end{align}
Next, substituting the \(S\)-cavity relations
\eqref{eq:appendix_s33} and \eqref{eq:appendix_s43},
together with the \(C\)-cavity relations
\eqref{eq:appendix_c34} and \eqref{eq:appendix_c44},
into \eqref{eq:appendix_v_basic_2} and \eqref{eq:appendix_v_basic_4},
we obtain
\begin{align}
  -v\,\alpha\,\frac{m_{33}}{1-v\,m_{43}}
  +\left(c+\frac{\lambda}{2}\right)m_{34}
  -\frac{c\,m_{34}(c\,m_{44}-m_{34})}
        {\tau-m_{34}+c\,m_{44}}
  &= 0,
  \label{eq:appendix_m34_closed}\\
  -v\,\alpha\,\frac{m_{43}}{1-v\,m_{43}}
  +\left(c+\frac{\lambda}{2}\right)m_{44}
  -\frac{c\,m_{44}(c\,m_{44}-m_{34})}
        {\tau-m_{34}+c\,m_{44}}
  &= 1.
  \label{eq:appendix_m44_closed}
\end{align}
Hence the \(v\)-sector is closed by
\eqref{eq:appendix_m33_closed}--\eqref{eq:appendix_m44_closed}.

\paragraph{The mixed sector.}
We finally derive the equations for
\(
m_{13},m_{14},m_{23},m_{24}
\).
At this stage the \(u\)-sector and \(v\)-sector are already closed, so the mixed equations will be linear in the unknown mixed variables.

From the basic identities
\eqref{eq:appendix_basic_113},
\eqref{eq:appendix_basic_114},
\eqref{eq:appendix_basic_223},
and
\eqref{eq:appendix_basic_224},
we have
\begin{align}
  -m_{11}+m_{13}-\beta_{14} &= 0,
  \label{eq:appendix_mixed_basic_1}\\
  -v\,s_{13}+c_{14} &= 0,
  \label{eq:appendix_mixed_basic_2}\\
  -m_{21}+m_{23}-\beta_{24} &= 0,
  \label{eq:appendix_mixed_basic_3}\\
  -v\,s_{23}+c_{24} &= 0.
  \label{eq:appendix_mixed_basic_4}
\end{align}
For convenience, define
\begin{align}
  D_1 &:= \tau-u\,m_{21}+c\,m_{22},
  \\
  D_2 &:= \tau-m_{34}+c\,m_{44}.
  \label{eq:appendix_D1D2_again}
\end{align}
Substituting \eqref{eq:appendix_beta14} and \eqref{eq:appendix_beta24} into
\eqref{eq:appendix_mixed_basic_1} and \eqref{eq:appendix_mixed_basic_3}, we obtain
\begin{align}
  m_{13}-m_{11}
  -\frac{\tau}{D_2}
  \left(
    m_{14}
    -
    c\,m_{24}\frac{m_{12}-u\,m_{11}}{D_1}
  \right)
  &=0,
  \label{eq:appendix_m13_closed}\\
  m_{23}-m_{21}
  -\frac{\tau}{D_2}
  \left(
    m_{24}
    -
    c\,m_{24}\frac{m_{22}-u\,m_{21}}{D_1}
  \right)
  &=0.
  \label{eq:appendix_m23_closed}
\end{align}
Similarly, substituting the \(S\)-cavity relations
\eqref{eq:appendix_s13} and \eqref{eq:appendix_s23},
together with the \(C\)-cavity relations
\eqref{eq:appendix_c14} and \eqref{eq:appendix_c24},
into \eqref{eq:appendix_mixed_basic_2} and \eqref{eq:appendix_mixed_basic_4}, we obtain
\begin{align}
  v\alpha\frac{m_{13}}{(1-m_{12})(1-v\,m_{43})}
  -\frac{\lambda}{2}m_{14}
  -\frac{c\tau}{D_2}
  \left(
    m_{14}
    -
    m_{24}\frac{c\,m_{12}-u\,m_{11}}{D_1}
  \right)
  &=0,
  \label{eq:appendix_m14_closed}\\
  v\alpha
  \left(
    \frac{m_{23}}{1-v\,m_{43}}
    +
    \frac{m_{22}m_{13}}{(1-m_{12})(1-v\,m_{43})}
  \right)
  -\frac{\lambda}{2}m_{24}
  -\frac{c\tau}{D_2}
  \left(
    m_{24}
    -
    m_{24}\frac{c\,m_{22}-u\,m_{21}}{D_1}
  \right)
  &=0.
  \label{eq:appendix_m24_closed}
\end{align}

Equations
\eqref{eq:appendix_m13_closed}--\eqref{eq:appendix_m24_closed}
close the mixed sector once the \(u\)- and \(v\)-sectors are known.

\paragraph{Summary of the closed system.}
Collecting the results above, the primary order parameters are determined by the following self-consistent system.

The \(u\)-sector satisfies
\begin{align}
  &m_{11}=1+\alpha\,\frac{m_{12}}{1-m_{12}},
  \notag\\
  &m_{21}=\alpha\,\frac{m_{22}}{1-m_{12}},
  \notag\\
  &-u\,m_{11}
  +\left(c+\frac{\lambda}{2}\right)m_{12}
  +\frac{(u\,m_{21}-c\,m_{22})(c\,m_{12}-u\,m_{11})}
        {\tau-u\,m_{21}+c\,m_{22}}
  =0,
  \notag\\
  &-u\,m_{21}
  +\left(c+\frac{\lambda}{2}\right)m_{22}
  +\frac{(u\,m_{21}-c\,m_{22})(c\,m_{22}-u\,m_{21})}
        {\tau-u\,m_{21}+c\,m_{22}}
  =1,
  \label{eq:appendix_closed_u_summary}
\end{align}
where
\begin{align}
  c=1+\frac{1+\sigma^2}{\alpha}.
  \label{eq:appendix_c_definition}
\end{align}

The \(v\)-sector satisfies
\begin{align}
  &m_{33}=1+m_{34}
  -\frac{m_{34}(c\,m_{44}-m_{34})}
        {\tau-m_{34}+c\,m_{44}},
  \notag\\
  &-v\,\alpha\,\frac{m_{33}}{1-v\,m_{43}}
  +\left(c+\frac{\lambda}{2}\right)m_{34}
  -\frac{c\,m_{34}(c\,m_{44}-m_{34})}
        {\tau-m_{34}+c\,m_{44}}
  =0,
  \notag\\
  &m_{43}=m_{44}
  -\frac{m_{44}(c\,m_{44}-m_{34})}
        {\tau-m_{34}+c\,m_{44}},
  \notag\\
  &-v\,\alpha\,\frac{m_{43}}{1-v\,m_{43}}
  +\left(c+\frac{\lambda}{2}\right)m_{44}
  -\frac{c\,m_{44}(c\,m_{44}-m_{34})}
        {\tau-m_{34}+c\,m_{44}}
  =1.
  \label{eq:appendix_closed_v_summary}
\end{align}

Finally, the mixed sector satisfies
\begin{align}
  &m_{13}-m_{11}
  -\frac{\tau}{D_2}
  \left(
    m_{14}
    -
    c\,m_{24}\frac{m_{12}-u\,m_{11}}{D_1}
  \right)
  =0,
  \notag\\
  &m_{23}-m_{21}
  -\frac{\tau}{D_2}
  \left(
    m_{24}
    -
    c\,m_{24}\frac{m_{22}-u\,m_{21}}{D_1}
  \right)
  =0,
  \notag\\
  &v\alpha\frac{m_{13}}{(1-m_{12})(1-v\,m_{43})}
  -\frac{\lambda}{2}m_{14}
  -\frac{c\tau}{D_2}
  \left(
    m_{14}
    -
    m_{24}\frac{c\,m_{12}-u\,m_{11}}{D_1}
  \right)
  =0,
  \notag\\
  &v\alpha
  \left(
    \frac{m_{23}}{1-v\,m_{43}}
    +
    \frac{m_{22}m_{13}}{(1-m_{12})(1-v\,m_{43})}
  \right)
  -\frac{\lambda}{2}m_{24}
  -\frac{c\tau}{D_2}
  \left(
    m_{24}
    -
    m_{24}\frac{c\,m_{22}-u\,m_{21}}{D_1}
  \right)
  =0,
  \label{eq:appendix_closed_mixed_summary}
\end{align}
with
\begin{align}
  D_1=\tau-u\,m_{21}+c\,m_{22},
  \qquad
  D_2=\tau-m_{34}+c\,m_{44}.
\end{align}

This is the closed deterministic system announced above.
Its structure is now transparent: first one solves
\eqref{eq:appendix_closed_u_summary} for the \(u\)-sector, then
\eqref{eq:appendix_closed_v_summary} for the \(v\)-sector, and finally
\eqref{eq:appendix_closed_mixed_summary} for \(m_{13},m_{14},m_{23},m_{24}\).
In the next subsection, we simplify these equations in the limit \(\lambda\to0\) and derive explicit formulas for the order parameters, in particular for \(m_{13}(u,v)\), which is the quantity needed for the evaluation of the generalization error.

\subsection{Explicit solution in \(\lambda = 0\)}
\label{appendix:main_result_lambda_zero}

We now solve the closed system obtained in the previous subsection in the unregularized case \(\lambda = 0\).
In this regime, the self-consistent equations simplify substantially and admit an explicit solution.
The main reason is that the \(u\)-sector and \(v\)-sector reduce to the same scalar equation, and the mixed sector then becomes a linear system with rational coefficients.
As a result, all primary order parameters can be written in terms of a single scalar resolvent function.

\paragraph{Solution of the \(u\)-sector.}
We begin with the equations
\eqref{eq:appendix_closed_u_summary}.
A convenient way to solve them is to eliminate \(m_{11}\), \(m_{21}\), and \(m_{22}\) in favor of \(m_{12}\), and then rewrite the resulting equation in terms of a single scalar function.
To this end, define
\begin{align}
  g(u):=m_{11}(u).
\end{align}
Then, from \eqref{eq:appendix_m11_closed},
\begin{align}
  g(u)=1+\alpha\,\frac{m_{12}(u)}{1-m_{12}(u)}.
\end{align}
Solving for \(m_{12}\), we obtain
\begin{align}
  m_{12}(u)
  =
  \frac{g(u)-1}{\alpha+g(u)-1}.
\end{align}
However, it is more convenient to parametrize the solution in the equivalent form
\begin{align}
  \frac{m_{12}(u)}{u}=\frac{g(u)}{c}.
  \label{eq:appendix_m12_over_u}
\end{align}
Using this ansatz together with the remaining equations in the \(u\)-sector, one finds that
\begin{align}
  u\,m_{21}(u)
  =
  \frac{\tau}{\tau-1}\bigl(g(u)-1\bigr),
  \label{eq:appendix_um21_solution}
\end{align}
and
\begin{align}
  m_{22}(u)
  =
  \frac{\tau}{\tau-1}\frac{g(u)}{c}.
  \label{eq:appendix_m22_solution}
\end{align}
Substituting these expressions back into \eqref{eq:appendix_closed_u_summary}, we find that \(g(u)\) must satisfy the scalar equation
\begin{align}
  u\,g(u)^2-\bigl(c+u(1-\alpha)\bigr)g(u)+c=0.
  \label{eq:appendix_g_quadratic}
\end{align}
The solution that is regular at \(u=0\) is the resolvent function of the Wishart matrix, which is given by
\begin{align}
  g(u)
  =
  \frac{c+u(1-\alpha)-\sqrt{\bigl(c+u(1-\alpha)\bigr)^2-4cu}}{2u}.
  \label{eq:appendix_g_definition}
\end{align}
Therefore the full \(u\)-sector is given by
\begin{align}
  m_{11}(u) &= g(u),
  \label{eq:appendix_u_sector_explicit_1}\\
  \frac{m_{12}(u)}{u} &= \frac{g(u)}{c},
  \label{eq:appendix_u_sector_explicit_2}\\
  u\,m_{21}(u) &= \frac{\tau}{\tau-1}\bigl(g(u)-1\bigr),
  \label{eq:appendix_u_sector_explicit_3}\\
  m_{22}(u) &= \frac{\tau}{\tau-1}\frac{g(u)}{c}.
  \label{eq:appendix_u_sector_explicit_4}
\end{align}

\paragraph{Solution of the \(v\)-sector.}
The \(v\)-sector is completely analogous.
Indeed, comparing
\eqref{eq:appendix_closed_u_summary}
and
\eqref{eq:appendix_closed_v_summary},
we see that the two systems are related by the substitutions
\begin{align}
  (m_{11},m_{12},u\,m_{21},m_{22})
  \longleftrightarrow
  (m_{33},v\,m_{43},m_{34},m_{44}),
\end{align}
together with \(u\leftrightarrow v\).
More precisely, the explicit solution is
\begin{align}
  m_{33}(v) &= g(v),
  \label{eq:appendix_v_sector_explicit_1}\\
  m_{43}(v) &= \frac{g(v)}{c},
  \label{eq:appendix_v_sector_explicit_2}\\
  m_{34}(v) &= \frac{\tau}{\tau-1}\bigl(g(v)-1\bigr),
  \label{eq:appendix_v_sector_explicit_3}\\
  m_{44}(v) &= \frac{\tau}{\tau-1}\frac{g(v)}{c}.
  \label{eq:appendix_v_sector_explicit_4}
\end{align}

\paragraph{Simplification of \(D_1\) and \(D_2\).}
Before solving the mixed sector, it is useful to simplify the denominators \(D_1\) and \(D_2\).
Using \eqref{eq:appendix_u_sector_explicit_3} and \eqref{eq:appendix_u_sector_explicit_4}, we obtain
\begin{align}
  D_1
  &=
  \tau-u\,m_{21}+c\,m_{22}
  \notag\\
  &=
  \tau-\frac{\tau}{\tau-1}\bigl(g(u)-1\bigr)
  +\frac{\tau}{\tau-1}g(u)
  =
  \frac{\tau(\tau+\alpha+g(u)-1)}{\tau-1}
\end{align}
in an intermediate parametrization.
However, for the subsequent algebra it is more convenient to keep the compact symbolic notation
\begin{align}
  D_1=\tau-u\,m_{21}+c\,m_{22},
  \qquad
  D_2=\tau-m_{34}+c\,m_{44},
\end{align}
and substitute the explicit formulas only at the end.
Doing so leads to cleaner rational expressions in \(g(u)\) and \(g(v)\).

\paragraph{Solution of the mixed sector.}
We now turn to
\(
m_{13},m_{14},m_{23},m_{24}
\),
which satisfy the linear system
\eqref{eq:appendix_closed_mixed_summary}.
Substituting the explicit \(u\)- and \(v\)-sector solutions
\eqref{eq:appendix_u_sector_explicit_1}--\eqref{eq:appendix_u_sector_explicit_4}
and
\eqref{eq:appendix_v_sector_explicit_1}--\eqref{eq:appendix_v_sector_explicit_4}
into
\eqref{eq:appendix_closed_mixed_summary},
and solving the resulting linear equations, we obtain
\begin{align}
  m_{13}(u,v)
  &=
  \frac{\alpha\,g(u)\,g(v)}{\Delta(u,v)},
  \label{eq:appendix_m13_explicit}\\
  m_{14}(u,v)
  &=
  \frac{\tau\,g(u)\bigl(g(v)-1\bigr)\bigl(\alpha+g(u)-1\bigr)}
       {(\tau-1)\,\Delta(u,v)},
  \label{eq:appendix_m14_explicit}\\
  m_{23}(u,v)
  &=
  \frac{\tau\,g(u)\,g(v)\bigl(\alpha+g(u)-1\bigr)\bigl(\alpha+g(v)-1\bigr)}
       {c(\tau-1)\,\Delta(u,v)},
  \label{eq:appendix_m23_explicit}\\
  m_{24}(u,v)
  &=
  \frac{\tau^{3}\,g(u)\bigl(g(v)-1\bigr)\bigl(\alpha+g(u)-1\bigr)
        \bigl(\alpha+g(u)+g(v)-1\bigr)}
       {c(\tau-1)^{3}\,\Delta(u,v)},
  \label{eq:appendix_m24_explicit}
\end{align}
where
\begin{align}
  \Delta(u,v)
  &:=
  \alpha
  -
  \bigl(g(u)-1\bigr)\bigl(g(v)-1\bigr)
  \left[
    1+\frac{c-1}{\tau-1}\bigl(\alpha+g(u)+g(v)-1\bigr)
  \right].
  \label{eq:appendix_Delta_definition}
\end{align}

At this stage, the resolvent block relevant for the generalization error has been obtained explicitly, since
\begin{equation}
  F(u,v)=m_{13}(u,v)
  \label{eq:appendix_F_limit_m13}
\end{equation}
in the high-dimensional limit.
Substituting \eqref{eq:appendix_m13_explicit} into
\eqref{eq:appendix_generalization_from_F}
already yields an explicit representation of the asymptotic generalization error.

In conclusion, combining
\eqref{eq:appendix_generalization_from_F}
with
\eqref{eq:appendix_m13_explicit},
we conclude that the asymptotic generalization error is given by
\begin{align}
  \mathcal E_t
  =
  \sum_{p=0}^t\sum_{q=0}^t
  \binom{t}{p}\binom{t}{q}
  \left(-\frac1\alpha\right)^{p+q}
  \frac{1}{p!\,q!}
  \partial_u^p\partial_v^q m_{13}(u,v)\Big|_{u=v=0},
  \label{eq:appendix_Et_m13_derivatives}
\end{align}
which is exactly the statement of Result~\ref{result:main}.

This completes the derivation, up to the routine verification that the branch of \(g(u)\) chosen in \eqref{eq:appendix_g_definition} is the one analytic at the origin and therefore compatible with the Taylor expansion used in \eqref{eq:appendix_generalization_from_F}.

\section{Numerical evaluation of the theoretical prediction}
\label{appendix:numerical_evaluation}

We numerically evaluate the theoretical formula in Result~\ref{result:main} through coefficient extraction from a bivariate generating function.  
A direct implementation of \eqref{eq:appendix_Et_m13_derivatives} is numerically unstable for moderate or large \(t\), because it involves substantial cancellation among many terms.  
Instead, we compute the same quantity by introducing
\begin{align}
  B(x,y)
  :=
  \frac{1}{(1-x)(1-y)}
  \,m_{13}\!\left(
    -\frac{x}{\alpha(1-x)},
    -\frac{y}{\alpha(1-y)}
  \right),
  \label{eq:appendix_num_B_def}
\end{align}
for which
\begin{align}
  \mathcal E_t = [x^t y^t]\,B(x,y).
  \label{eq:appendix_num_Et_diag_coeff}
\end{align}
Therefore, the problem reduces to computing the diagonal coefficients of \(B(x,y)\).

To do so, we first expand
\begin{align}
  h(x)
  :=
  g\!\left(-\frac{x}{\alpha(1-x)}\right)
  =
  \sum_{n\ge0} h_n x^n,
\end{align}
where \(g(u)\) is the resolvent function appearing in Result~\ref{result:main} (see \eqref{eq:appendix_g_definition}).  
Substituting \(u=-x/(\alpha(1-x))\) into the quadratic relation \eqref{eq:appendix_g_quadratic} for \(g(u)\), i.e.\ into
\begin{align}
  u g(u)^2 - \bigl(c+(1-\alpha)u\bigr)g(u) + c = 0
\end{align}
gives
\begin{align}
  x h(x)^2 + \ab(\alpha c-\ab(2+\sigma^2) x)h(x) - \alpha c (1-x)=0,
\end{align}
with \(c\) as in \eqref{eq:appendix_c_definition}.
Writing \(h(x)=\sum_{n\ge0} h_n x^n\) and matching coefficients yields the recurrence
\begin{align}
  h_0 &= 1,\\
  h_n &=
  \frac{
    \ab(2+\sigma^2) h_{n-1}
    -
    \sum_{k=0}^{n-1} h_k h_{n-1-k}
    -
    \alpha c\,\mathbf{1}_{\{n=1\}}
  }{\alpha c},
  \qquad n\ge1.
\end{align}
This allows us to compute the coefficients of \(h(x)\) stably up to any finite order.

For \(\lambda=0\), the closed form for \(m_{13}(u,v)\) is \eqref{eq:appendix_m13_explicit}, i.e.\ \(\alpha\,g(u)g(v)/\Delta(u,v)\) with \(\Delta(u,v)\) defined in \eqref{eq:appendix_Delta_definition}.
Equivalently, Result~\ref{result:main} records the same expression as
\begin{align}
  m_{13}(u,v)
  =
  \frac{\alpha\,g(u)\,g(v)}
  {\alpha-(g(u)-1)(g(v)-1)
    \left[
      1+\frac{1+\sigma^2}{\alpha(\tau-1)}
      \bigl(\alpha+g(u)+g(v)-1\bigr)
    \right]}.
\end{align}
After the change of variables above, \(B(x,y)\) satisfies
\begin{align}
  D(x,y)\,B(x,y)
  =
  \frac{\alpha\,h(x)h(y)}{(1-x)(1-y)},
\end{align}
where
\begin{align}
  D(x,y)
  =
  \alpha
  -(h(x)-1)(h(y)-1)
  \left[
    1+\frac{c-1}{\tau-1}
    \bigl(\alpha+h(x)+h(y)-1\bigr)
  \right].
\end{align}

We then expand
\begin{align}
  B(x,y)=\sum_{m,n\ge0} b_{m,n}x^m y^n,
  \qquad
  D(x,y)=\sum_{i,j\ge0} d_{i,j}x^i y^j.
\end{align}
Since the constant term satisfies \(d_{0,0}=\alpha\neq0\), the coefficients \(b_{m,n}\) are determined recursively from
\begin{align}
  \sum_{i=0}^m\sum_{j=0}^n d_{i,j}\,b_{m-i,n-j}
  =
  \alpha
  \left(\sum_{k=0}^m h_k\right)
  \left(\sum_{\ell=0}^n h_\ell\right).
\end{align}
Equivalently,
\begin{align}
  b_{m,n}
  =
  \frac{1}{\alpha}
  \left(
    \alpha
    \left(\sum_{k=0}^m h_k\right)
    \left(\sum_{\ell=0}^n h_\ell\right)
    -
    \sum_{\substack{0\le i\le m,\ 0\le j\le n\\(i,j)\neq(0,0)}}
    d_{i,j}\,b_{m-i,n-j}
  \right).
\end{align}
We evaluate this recursion in increasing total degree \(m+n\), so that each coefficient depends only on previously computed lower-degree coefficients.

Finally, the theoretical prediction is obtained from the diagonal coefficients in \eqref{eq:appendix_num_Et_diag_coeff},
\begin{align}
  \mathcal E_t = b_{t,t}.
\end{align}
In practice, to compute \(\mathcal E_t\) for \(t=0,\dots,t_{\max}\), we generate the coefficients up to degree \(t_{\max}\) in each variable and then read off the diagonal entries.  
Since the coefficients can vary substantially in magnitude for large \(t\), we use high-precision arithmetic.  
This coefficient-based procedure is substantially more stable than directly evaluating the derivative formula \eqref{eq:appendix_Et_m13_derivatives} (Result~\ref{result:main}).

\section{Theoretical and Experimental Validation}
\label{appendix:theory_vs_experiment}

In this appendix, we empirically validate the theoretical prediction derived from random matrix theory for the weight prediction model with linear attention. 
Figure~\ref{fig:theory_vs_experiment} compares the asymptotic theoretical curve under the ridgeless limit with numerical experiments performed at finite feature dimension $D$ across several choices of $(\alpha,\tau)$. 
These parameter settings are chosen to cover qualitatively different regimes, including the saturation, overthinking, polynomial decay, and exponential decay regimes.

\begin{figure}[h]
  \centering
  \includegraphics[width=1.0\textwidth]{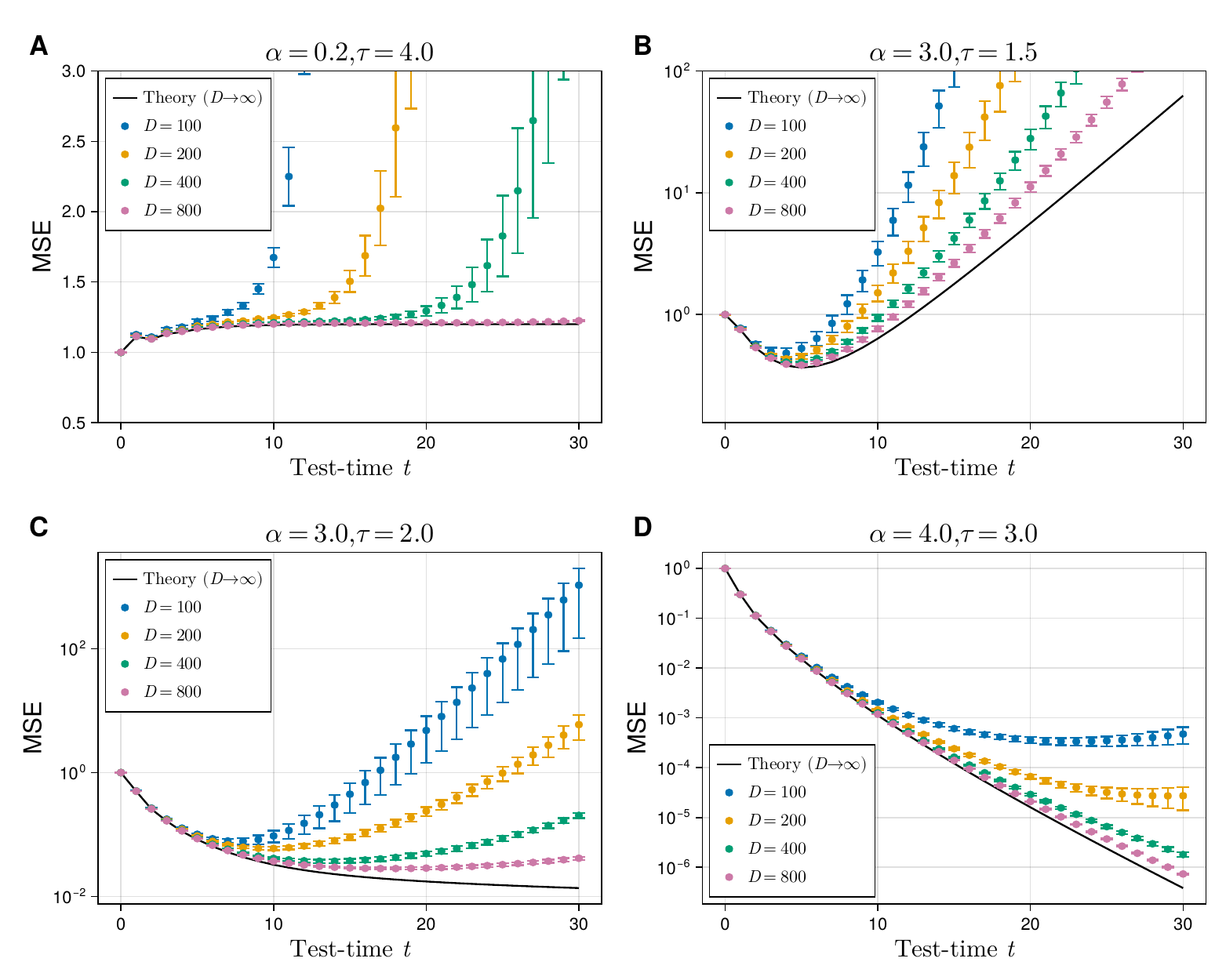}
  \caption{
    \textbf{Comparison between the theoretical prediction and numerical experiments.}
    The solid lines show the theoretical prediction in the $D \to \infty$ limit, while scatter points show numerical results at finite $D$.
    Parameters: (A-D) $\lambda=10^{-5}$, $\sigma = 0.01$. 
    Error bars represent the standard error of the mean over $5$ trials per point.
  }
  \label{fig:theory_vs_experiment}
\end{figure}

Overall, the numerical results show clear agreement with the theoretical prediction. 
In all parameter settings we tested, the finite-$D$ experimental curves approach the $D \to \infty$ theoretical prediction as $D$ increases, indicating that the asymptotic theory accurately captures the large-dimensional behavior of the model even when compared against finite-size simulations.

The agreement is particularly close at small test-time depth $t$, where finite-size effects remain limited. 
As $t$ becomes larger, discrepancies between theory and experiment gradually accumulate, and the empirical MSE tends to become larger than the theoretical prediction. 
This deviation is consistent with finite-size corrections that are not captured by the leading-order asymptotic analysis. 
Importantly, however, these discrepancies diminish as $D$ increases, and the experiments continue to converge toward the theoretical curve, supporting the validity of the theoretical characterization.


\section{Derivation of Result~\ref{result:scaling_analysis}}
\label{appendix:scaling_analysis}

In this appendix, we derive the switching criterion stated in Result~\ref{result:scaling_analysis}.
The starting point is the diagonal coefficient representation of the quantity
\(\mathcal E_t^{\mathrm{new}}\), which allows us to study its large-\(t\) behavior through the singularities of a two-variable generating function.
Our goal is to identify the singular point that controls the diagonal asymptotics and to determine when its modulus crosses the unit circle.
This yields an explicit critical value \(\tau_c(\alpha,\sigma^2)\) at which the exponential rate changes from growth to decay.

We begin by restating the result.
\resultscalinganalysis*
We now derive this result in detail.

First, from the coefficient formula,
\begin{align}
  \mathcal E_t^{\mathrm{new}}
  =
  [x^t y^t]\,B(x,y),
\end{align}
where
\begin{align}
  B(x,y)
  :=
  \frac{1}{(1-x)(1-y)}
  \,m_{13}\!\left(
    -\frac{x}{\alpha(1-x)},
    -\frac{y}{\alpha(1-y)}
  \right).
\end{align}
The function \(m_{13}(u,v)\) is meromorphic away from the branch locus of the scalar Wishart resolvent \(g\), and its singularities are determined by the vanishing of its denominator
\begin{align}
  D(u,v)
  :=
  \alpha-(g(u)-1)(g(v)-1)
  \left[
    1+\frac{1+\sigma^2}{\alpha(\tau-1)}
    \bigl(\alpha+g(u)+g(v)-1\bigr)
  \right].
\end{align}
Accordingly, the singular variety relevant for the diagonal coefficient is
\begin{align}
  \mathcal V
  :=
  \{(x,y): D(u(x),v(y))=0\},
\end{align}
where
\begin{align}
  u(x):=-\frac{x}{\alpha(1-x)},
  \qquad
  v(y):=-\frac{y}{\alpha(1-y)}.
\end{align}

Because both the singular variety and the diagonal direction \((t,t)\) are invariant under the exchange \(x\leftrightarrow y\), the natural candidate for the dominant critical point is a symmetric point of the form
\begin{align}
  (x,y)=(\xi,\xi)\in\mathcal V.
\end{align}
Throughout this appendix, we assume that this symmetric point is indeed the unique dominant critical point and that no other singularity of smaller modulus contributes to the diagonal asymptotics.
Under this assumption, the problem reduces to a one-parameter analysis along the symmetric branch, which we now describe.

On the symmetric branch, \(u=v\) and hence \(g(u)=g(v)\).
The scalar resolvent \(g(u)\) satisfies the quadratic equation
\begin{align}
  u\,g(u)^2-\bigl(c+(1-\alpha)u\bigr)g(u)+c=0,
\end{align}
where
\begin{align}
  c=1+\frac{1+\sigma^2}{\alpha}.
\end{align}
Solving this equation for \(u\), we obtain
\begin{align}
  u
  =
  -\frac{c(1-g(u))}
  {g(u)(\alpha-1+g(u))}.
\end{align}
Combining this with
\begin{align}
  u(x)=-\frac{x}{\alpha(1-x)},
\end{align}
we obtain
\begin{align}
  x
  =
  \frac{\alpha c(1-g(u))}
  {g(u)(\alpha-1+g(u))+\alpha c(1-g(u))}
  =
  \frac{(\alpha+1+\sigma^2)(1-g(u))}
  {\alpha+\sigma^2(1-g(u))+(1-g(u))^2}.
  \label{eq:appendix_scaling_x_of_g}
\end{align}

Next, we rewrite the singularity condition on the diagonal.
The condition \(D(u,u)=0\) becomes
\begin{align}
  (1-g(u))^2
  \left[
    1+\frac{1+\sigma^2}{\alpha(\tau-1)}
    \bigl(\alpha-1+2g(u)\bigr)
  \right]
  =
  \alpha.
  \label{eq:appendix_scaling_diag_eq}
\end{align}
Therefore, the symmetric singular points are determined by \eqref{eq:appendix_scaling_diag_eq}, with their location in the \(x\)-plane given by \eqref{eq:appendix_scaling_x_of_g}.

Under the dominant symmetric critical-point assumption, the diagonal coefficient is governed by a smooth critical point \((\xi,\xi)\) on the singular variety.
Standard smooth-point asymptotics in two variables then gives
\begin{align}
  [x^t y^t]\,B(x,y)
  \sim
  K\,t^{-1/2}\,(\xi^2)^{-t}
\end{align}
for some nonzero constant \(K\).
Equivalently,
\begin{align}
  \mathcal E_t^{\mathrm{new}}
  \sim
  K\,t^{-1/2}\,\Lambda(\alpha,\tau,\sigma^2)^t,
  \qquad
  \Lambda(\alpha,\tau,\sigma^2):=|\xi|^{-2}.
  \label{eq:appendix_scaling_asymptotic}
\end{align}
Thus the long-time behavior is determined by the position of \(\xi\) relative to the unit circle:
\begin{align}
  |\xi|<1
  &\Longrightarrow
  \Lambda>1
  \qquad\text{(exponential growth)},\\
  |\xi|=1
  &\Longrightarrow
  \Lambda=1
  \qquad\text{(critical polynomial behavior)},\\
  |\xi|>1
  &\Longrightarrow
  \Lambda<1
  \qquad\text{(exponential damping)}.
\end{align}

We now determine the parameter value at which this crossing occurs.
Along the relevant symmetric branch, the transition takes place when the dominant point reaches the unit circle.
The branch connected to the large-\(\tau\) regime crosses at $\xi=-1$.
Substituting \(x=-1\) into \eqref{eq:appendix_scaling_x_of_g}, we obtain
\begin{align}
  (1-g(u))^2
  +(\alpha+1+2\sigma^2)(1-g(u))
  +\alpha
  =
  0.
\end{align}
Define
\begin{align}
  \Delta
  :=
  (\alpha+1+2\sigma^2)^2-4\alpha.
\end{align}
The relevant root on the dominant symmetric branch is
\begin{align}
  1-g(u)
  =
  -\frac{\alpha+1+2\sigma^2-\sqrt{\Delta}}{2}.
\end{align}
Substituting this root into \eqref{eq:appendix_scaling_diag_eq} and solving for \(\tau\), we obtain
\begin{align}
  \tau_c(\alpha,\sigma^2)
  =
  1+
  \frac{(1+\sigma^2)(\alpha+1+2\sigma^2-\sqrt{\Delta})
  (2\alpha+2+2\sigma^2-\sqrt{\Delta})}
  {2\alpha\sqrt{\Delta}}.
  \label{eq:appendix_scaling_tau_c}
\end{align}
This is exactly the expression stated in Result~\ref{result:scaling_analysis}.

\subsection{Inference-time label noise}
\label{appendix:noisy_inference}

The phase transition above was derived with noiseless inference labels.
We now add the same label noise used in pretraining and show that the stability phase diagram is unchanged.
Let
\begin{align}
  \vb{y}
  =
  \mathbf{X}^{\top}\vb{w}+\vb{\xi},
  \qquad
  \vb{\xi}\sim\mathcal{N}(\vb{0},\sigma^2\mathbf{I}_L),
\end{align}
where \(\vb{\xi}\) is independent of \(\vb{w}\), \(\mathbf{X}\), and the pretraining data.
Define \(\mathbf{R}:=\mathbf{I}-\alpha^{-1}\mathbf{A}\mathbf{S}\).
The error recursion becomes
\begin{align}
  \vb{e}_{t+1}
  =
  \mathbf{R}\vb{e}_t
  -
  \frac{1}{\alpha}\mathbf{A}\mathbf{X}\vb{\xi},
  \qquad
  \vb{e}_0=\vb{w},
\end{align}
and hence
\begin{align}
  \vb{e}_t
  =
  \mathbf{R}^t\vb{w}
  -
  (\mathbf{I}-\mathbf{R}^t)\mathbf{S}^{\dagger}\mathbf{X}\vb{\xi},
\end{align}
where \(\mathbf{S}^{\dagger}\) is the Moore--Penrose inverse of \(\mathbf{S}\).
Averaging over \(\vb{w}\) and \(\vb{\xi}\) gives
\begin{align}
  \mathcal{E}_t^{\mathrm{noisy}}
  =
  \mathcal{E}_t
  +
  \frac{\sigma^2}{D}
  \E\Tr\left[
    (\mathbf{I}-\mathbf{R}^t)\mathbf{S}^{\dagger}
    (\mathbf{I}-\mathbf{R}^t)^{\top}
  \right].
\end{align}
The inference noise changes only the additive term and leaves \(\mathbf{R}\) unchanged.
Therefore, the stability condition for \(\mathbf{R}^t\) and the phase boundary \(\tau_c(\alpha,\sigma^2)\) in Eq.~\eqref{eq:appendix_scaling_tau_c} are unchanged.
For \(\alpha>1\) in the stable regime, \(\mathbf{R}^t\to\mathbf{0}\), and the additional term converges in the high-dimensional limit to
\begin{align}
  \frac{\sigma^2}{D}\E\Tr[\mathbf{S}^{-1}]
  \longrightarrow
  \frac{\sigma^2}{\alpha-1}.
\end{align}
Thus, inference noise raises the error floor and can reduce the optimal reasoning depth, but it does not move the phase boundary.

\section{Derivation of Theorem~\ref{theorem:exponential_decay}}
\label{appendix:exponential_decay}

In this appendix, we derive the asymptotic scaling law stated in Theorem~\ref{theorem:exponential_decay}.
We begin by restating the result.
\theoremexponentialdecay*

\begin{proof}
    We work in the population-risk regime \(\tau\to\infty\) under ridgeless learning \(\lambda=0\), and assume \(\alpha>1\).
    In this regime, the test-time generalization error can be written as
    \begin{align}
      \mathcal E_t
      =
      \int
      \left(1-\frac{s}{\alpha c}\right)^{2t}\mu_\alpha(ds),
      \label{eq:appendix_exponential_decay_integral}
    \end{align}
    where $c=1+(1+\sigma^2)/\alpha$ and \(\mu_\alpha\) is the Marchenko--Pastur law with aspect ratio \(\alpha\). 
    Here, \(\mu_\alpha\) is the limiting empirical eigenvalue distribution of \(\mathbf{S}\).
    Since \(\alpha>1\), its mass at zero is zero. Hence \(\mu_\alpha\) consists only of the following density:
    \begin{align}
      \mu_\alpha(ds)
      =
      \frac{\sqrt{(s_+-s)(s-s_-)}}{2\pi s}\,
      \mathbf 1_{[s_-,s_+]}(s)\,ds,
      \qquad
      s_\pm=(1\pm\sqrt{\alpha})^2.
      \label{eq:appendix_exponential_decay_mp}
    \end{align}
    
    Fix a small \(\varepsilon>0\), and split \eqref{eq:appendix_exponential_decay_integral} as
    \begin{align}
      \mathcal E_t
      =
      \int_{s_-}^{s_-+\varepsilon}
      \left(1-\frac{s}{\alpha c}\right)^{2t}\mu_\alpha(ds)
      +
      \int_{s_-+\varepsilon}^{s_+}
      \left(1-\frac{s}{\alpha c}\right)^{2t}\mu_\alpha(ds).
      \label{eq:appendix_exponential_decay_split}
    \end{align}
    
    We begin with contribution of the first term in Equation~\eqref{eq:appendix_exponential_decay_split}.
    Write
    \begin{align}
      s=s_-+x,
      \qquad
      x\in[0,\varepsilon].
    \end{align}
    Since
    \begin{align}
      s-s_-=x,
      \qquad
      s_+-s=(s_+-s_-)-x,
    \end{align}
    and \(s_-= (\sqrt{\alpha}-1)^2>0\), the density satisfies
    \begin{align}
      \frac{\sqrt{(s_+-s)(s-s_-)}}{2\pi s}
      \asymp x^{1/2}
      \qquad
      (x\downarrow0).
      \label{eq:appendix_exponential_decay_density}
    \end{align}
    On the other hand, we have
    \begin{align}
      \left(1-\frac{s}{\alpha c}\right)^{2t}
      =
      r_\alpha^{2t}
      \left(1-\frac{x}{\alpha c\,r_\alpha}\right)^{2t},
    \end{align}
    where
    \begin{align}
     r_\alpha:=1-\frac{s_-}{\alpha c}=\frac{2\sqrt{\alpha}+\sigma^2}{\alpha+1+\sigma^2}\in(0,1).
    \end{align}
    For \(x\in[0,\varepsilon]\) with \(\varepsilon\) sufficiently small, the factor in parentheses is bounded above and below by exponentials, namely
    \begin{align}
      e^{-C_2tx}
      \le
      \left(1-\frac{x}{\alpha c\,r_\alpha}\right)^{2t}
      \le
      e^{-C_1tx}
    \end{align} 
    for some constants \(C_1,C_2>0\) independent of \(t\).
    Therefore,
    \begin{align}
      \int_{s_-}^{s_-+\varepsilon}
      \left(1-\frac{s}{\alpha c}\right)^{2t}\mu_\alpha(ds)
      &\asymp
      r_\alpha^{2t}
      \int_0^\varepsilon e^{-Ctx}x^{1/2}\,dx \\
      &=r_\alpha^{2t}t^{-3/2}\int_0^{\varepsilon t} e^{-Cu}u^{1/2}\,du \\
      &\asymp
      r_\alpha^{2t} t^{-3/2}.
      \label{eq:appendix_exponential_decay_edge_final}
    \end{align}
    
    Next, we consider the contribution of the second term in Equation~\eqref{eq:appendix_exponential_decay_split}.
    For $\sigma^2>0$ and sufficiently small \(\varepsilon>0\), continuity and convexity
    give an \(\eta>0\) such that
    \begin{align}
      \left|1-\frac{s}{\alpha c}\right|
      \le r_\alpha-\eta
      \qquad (s\in[s_-+\varepsilon,s_+]).
    \end{align}
    Since \(\mu_\alpha\) is a finite measure, this yields
    \begin{align}
      \int_{s_-+\varepsilon}^{s_+}
      \left(1-\frac{s}{\alpha c}\right)^{2t}\mu_\alpha(ds)
      \le
      C(r_\alpha-\eta)^{2t}.
      \label{eq:appendix_exponential_decay_bulk}
    \end{align}
    This term is negligible compared with \eqref{eq:appendix_exponential_decay_edge_final}.

    If \(\sigma^2=0\), 
    the density in \eqref{eq:appendix_exponential_decay_mp} is
    \(\asymp(s_+-s)^{1/2}\) near \(s_+\). Writing \(x=s_+-s\), we have
    \(1-s/(\alpha c)=-r_\alpha+x/(\alpha c)\). Hence the same
    edge calculation gives
    \begin{align}
      \int_{s_+-\varepsilon}^{s_+}
      \left(1-\frac{s}{\alpha c}\right)^{2t}\mu_\alpha(ds)
      \asymp t^{-3/2}r_\alpha^{2t}.
    \end{align}
    On \([s_-+\varepsilon,s_+-\varepsilon]\), the absolute value of
    \(1-s/(\alpha c)\) is at most \(r_\alpha-\eta\) for some \(\eta>0\).
    Its contribution is therefore exponentially smaller than either
    edge contribution.
    
    Combining \eqref{eq:appendix_exponential_decay_split} and
    \eqref{eq:appendix_exponential_decay_edge_final} with the bounds
    above, we conclude that
    \begin{align}
      \mathcal E_t
      \asymp
      t^{-3/2}r_\alpha^{2t}
      =
      t^{-3/2}
      \left(
        \frac{2\sqrt{\alpha}+\sigma^2}{\alpha+1+\sigma^2}
      \right)^{2t}.
    \end{align}
    This proves the first claim.
    
    Finally, when \(\alpha\to\infty\) with
    \(\sigma^2=o(\sqrt{\alpha})\), the exponential base satisfies
    \begin{align}
      \frac{2\sqrt{\alpha}+\sigma^2}{\alpha+1+\sigma^2}
      =
      \frac{2}{\sqrt{\alpha}}\bigl(1+o(1)\bigr),
    \end{align}
    Hence, equivalently,
    \begin{align}
      \mathcal{E}_t \asymp
      t^{-3/2}\left(\frac{4}{\alpha}\right)^t.
    \end{align}
\end{proof}


\section{Derivation of Result~\ref{result:optimal_depth}}
\label{appendix:optimal_depth}

In this appendix, we derive the optimal test-time depth and minimum error in the overthinking regime stated in Result~\ref{result:optimal_depth}.

We begin by restating the result.
\resultoptimaldepth*

We derive Result~\ref{result:optimal_depth} from Result~\ref{result:scaling_analysis}.
Let
\begin{align}
  \delta:=\tau_c(\alpha,\sigma^2)-\tau>0.
  \label{eq:appendix_optimal_depth_delta}
\end{align}
In the overthinking regime, we have \(\tau<\tau_c(\alpha,\sigma^2)\), and hence
Result~\ref{result:scaling_analysis} gives
\begin{align}
  \mathcal E_t
  \sim
  K(\alpha,\tau,\sigma^2)\,t^{-1/2}\Lambda(\alpha,\tau,\sigma^2)^t
  \qquad (t\to\infty),
  \label{eq:appendix_optimal_depth_asymptotic}
\end{align}
since \(\mathcal E_\infty=0\) in this regime.

Assume, as in the derivation of Result~\ref{result:scaling_analysis}, that the dominant symmetric critical point remains unique near the transition and crosses the unit circle transversely at \(\tau=\tau_c\).
Then \(\Lambda(\alpha,\tau,\sigma^2)\) is smooth near \(\tau_c\), and we have
\begin{align}
  \log \Lambda(\alpha,\tau,\sigma^2)
  =
  \kappa(\alpha,\sigma^2)\,(\tau_c(\alpha,\sigma^2)-\tau)
  +o(\tau_c-\tau),
  \label{eq:appendix_optimal_depth_log_lambda}
\end{align}
for some \(\kappa(\alpha,\sigma^2)>0\).
Moreover,
\begin{align}
  K(\alpha,\tau,\sigma^2)=K_c(\alpha,\sigma^2)+o(1),
  \label{eq:appendix_optimal_depth_prefactor}
\end{align}
with \(K_c(\alpha,\sigma^2)>0\).
Substituting \eqref{eq:appendix_optimal_depth_delta}, \eqref{eq:appendix_optimal_depth_log_lambda}, and \eqref{eq:appendix_optimal_depth_prefactor} into \eqref{eq:appendix_optimal_depth_asymptotic}, we obtain
\begin{align}
  \mathcal E_t
  \sim
  K_c(\alpha,\sigma^2)\,
  t^{-1/2}\exp\ab(\kappa(\alpha,\sigma^2)\delta\, t).
  \label{eq:appendix_optimal_depth_reduced}
\end{align}

We now minimize the continuous proxy suggested by \eqref{eq:appendix_optimal_depth_reduced},
\begin{align}
  \Phi_\delta(t):=t^{-1/2}e^{\kappa\delta t}.
  \label{eq:appendix_optimal_depth_proxy}
\end{align}
One can check that the unique minimizer of \eqref{eq:appendix_optimal_depth_proxy} is
\begin{align}
  t_{\mathrm{cont}}
  =
  \frac{1}{2\kappa(\alpha,\sigma^2)\delta}.
  \label{eq:appendix_optimal_depth_tcont}
\end{align}
As \(\delta\downarrow0\), this tends to \(+\infty\), so the large-\(t\) asymptotics is self-consistent.
Passing back to the discrete minimizer \(t^*\in\mathbb Z_{\ge0}\), we obtain
\begin{align}
  t^*
  =
  \frac{C_t(\alpha,\sigma^2)}{\tau_c(\alpha,\sigma^2)-\tau}
  \bigl(1+o(1)\bigr),
  \label{eq:appendix_optimal_depth_tstar}
\end{align}
where
\begin{align}
  C_t(\alpha,\sigma^2)
  =
  \frac{1}{2\kappa(\alpha,\sigma^2)} > 0.
\end{align}

Finally, substituting \eqref{eq:appendix_optimal_depth_tstar} into \eqref{eq:appendix_optimal_depth_reduced} gives
\begin{align}
  \mathcal E_{t^*}
  &\sim
  K_c(\alpha,\sigma^2)\,(t^*)^{-1/2}
  \exp\ab(\kappa(\alpha,\sigma^2)\delta\, t^*) \\
  &=
  K_c(\alpha,\sigma^2)e^{1/2}(t^*)^{-1/2}\ab(1+o(1)).
  \label{eq:appendix_optimal_depth_estar}
\end{align}
Hence
\begin{align}
  \mathcal E_{t^*}
  =
  C_E(\alpha,\sigma^2)\,(t^*)^{-1/2}\ab(1+o(1)),
  \label{eq:appendix_optimal_depth_estar_claim}
\end{align}
where
\begin{align}
  C_E(\alpha,\sigma^2):=K_c(\alpha,\sigma^2)e^{1/2}>0.
\end{align}
Equation \eqref{eq:appendix_optimal_depth_estar_claim} is exactly the claimed form \eqref{eq:Estar_via_tstar}, which proves Result~\ref{result:optimal_depth}.

\section{Experimental details}
\label{appendix:experimental_details}

In this appendix, we summarize the experimental details for the fully learned
linear-attention and softmax-attention experiments in
Section~\ref{section:experiments}.
All experiments used synthetic linear-regression tasks with ambient dimension
\(D=50\).
Each task was defined by a ground-truth parameter
\(\vb{w}_\mu \sim \mathcal{N}(\vb{0}, I_D)\), and each in-context example was
generated as
\[
  \vb{x}_{\mu,l} \sim \mathcal{N}(\vb{0}, I_D),
  \qquad
  y_{\mu,l} = \vb{w}_\mu^\top \vb{x}_{\mu,l}.
\]
Thus, the experiments were noiseless.
We used the same prompt structure as in Section~\ref{section:model}, but
initialized the first parameter-estimate token with a random Gaussian vector.

\subsection{Model architecture}

Both models used a single residual self-attention layer acting on tokens in
\(\mathbb{R}^{2D+2}\).
For \(D=50\), the token dimension was therefore \(2D+2=102\).
To compute the next estimate from reasoning depth \(t\), the prompt contained
\(L\) example tokens followed by \(t+1\) parameter-estimate tokens,
corresponding to
\[
  [\hat{\vb{w}}_0,\ldots,\hat{\vb{w}}_t].
\]
The total number of tokens was therefore
\[
  T = L+t+1.
\]
The prediction was always read from the parameter block of the final token,
which we denote by \(\vb{z}_T\).

In both models, only the final token was updated according to
\begin{align}
  \vb{z}_T
  \mapsto
  \vb{z}_T + W_V \sum_{s=1}^{T} a_s \vb{z}_s,
\end{align}
where \(W_K, W_Q, W_V \in \mathbb{R}^{(2D+2)\times(2D+2)}\) denote the key,
query, and value matrices.

In the fully learned linear-attention model, we used
\begin{align}
  a_s
  =
  m_s^{(t)}
  \frac{1}{L}\,
  \bigl(W_K \vb{z}_s\bigr)^\top
  \bigl(W_Q \vb{z}_T\bigr),
\end{align}
where \(m_s^{(t)}\) is a mask over source positions that determines which
source tokens are visible to the final token.

In the softmax-attention model, we used the same bilinear score, but normalized
it with a softmax over the visible source positions:
\begin{align}
  a_s
  =
  \frac{
    \exp\ab(
      \frac{\ab(W_K \vb{z}_s)^\top \ab(W_Q \vb{z}_T)}{\sqrt{2D+2}}
      + b_s^{(t)}
    )
  }{
    \sum_{r=1}^{T}
    \exp\ab(
      \frac{\ab(W_K \vb{z}_r)^\top \ab(W_Q \vb{z}_T)}{\sqrt{2D+2}}
      + b_r^{(t)}
    )
  },
\end{align}
where \(b_s^{(t)}\) is an additive attention mask.

For both the fully learned linear-attention and softmax-attention experiments,
we used a full-history visibility pattern for the final parameter-estimate
token.
That is, the final token was allowed to attend to all \(L\) example tokens and
all previously generated estimate tokens, but not to itself.
For the fully learned linear-attention model, this mask was
\begin{align}
  m_s^{(t)}
  =
  \begin{cases}
    1, & 1 \le s \le L+t,\\
    0, & s=T.
  \end{cases}
\end{align}
For the softmax-attention model, the corresponding additive mask was
\begin{align}
  b_s^{(t)}
  =
  \begin{cases}
    0, & 1 \le s \le L+t,\\
    -\infty, & s=T.
  \end{cases}
\end{align}
Thus, in both models, the visible source tokens consisted of the \(L\) example
tokens and the previous estimate tokens
\([\hat{\vb{w}}_0,\ldots,\hat{\vb{w}}_{t-1}]\), while the final query token
\(\hat{\vb{w}}_t\) was masked from attending to itself.

\subsection{Learning protocol}

For each pair \((M,L)\), we first sampled a fixed training set of \(M\) tasks,
each with \(L\) in-context examples.
We then trained the model parameters on this fixed set by minimizing the MSE between the predicted parameter vector and the ground-truth task
parameter, with \(\ell_2\) regularization on all trainable matrices.

Training used a one-step objective.
For each sampled training task, we initialized the first parameter estimate
\(\hat{\vb{w}}_0\) as a random Gaussian vector, constructed a prompt containing
the \(L\) example tokens and this single estimate token, and trained the model
to predict the ground-truth parameter \(\vb{w}\) from the updated final token.

Optimization used Adam with learning rate \(10^{-3}\), minibatch size \(300\),
\(3000\) gradient steps, and regularization coefficient
\(\lambda = 10^{-5}\).

\subsection{Inference protocol}

After training, we evaluated each model on newly sampled test tasks from the
same noiseless distribution.
For each test task, we initialized the first parameter estimate
\(\hat{\vb{w}}_0\) as a random Gaussian vector and then iteratively reapplied
the same learned attention layer to generate a trajectory
\((\hat{\vb{w}}_t)_{t\geq 0}\).

At reasoning depth \(t\), the prompt contained the \(L\) example tokens and the
history of generated estimates
\[
  [\hat{\vb{w}}_0,\ldots,\hat{\vb{w}}_t].
\]
The model updated the final estimate token to produce the next estimate
\(\hat{\vb{w}}_{t+1}\).
Under the full-history mask, the final token could attend to the example
tokens and the previous estimate tokens
\([\hat{\vb{w}}_0,\ldots,\hat{\vb{w}}_{t-1}]\), but not to itself.
Thus, both the fully learned linear-attention and softmax-attention models had
access to the full generated history during inference, except for the current
final token itself as an attention source.

We evaluated the MSE of the parameter estimate
\begin{align}
  \mathcal{E}_t
  =
  \frac{1}{D}\,
  \mathbb{E}\bigl[\|\vb{w}-\hat{\vb{w}}_t\|^2\bigr]
\end{align}
for reasoning depth \(t\).
Evaluation used \(4096\) independently sampled test tasks for each trained
model.

\subsection{Computational resources}

All experiments were run in PyTorch.
The experiments summarized here were run on CPU using an AMD EPYC 9654P
96-Core processor.


\newpage
\section*{NeurIPS Paper Checklist}

\begin{enumerate}

\item {\bf Claims}
    \item[] Question: Do the main claims made in the abstract and introduction accurately reflect the paper's contributions and scope?
    \item[] Answer: \answerYes{} 
    \item[] Justification: We state the main contributions in Section~\ref{section:introduction}. 
    These claims are matched by the results in Section~\ref{section:result_asympt} and Section~\ref{section:result_phase_transition}, and their assumptions and scope are clarified in the corresponding statements.

\item {\bf Limitations}
    \item[] Question: Does the paper discuss the limitations of the work performed by the authors?
    \item[] Answer: \answerYes{} 
    \item[] Justification: We discuss the limitations of the work in the Section~\ref{section:conclusion}.

\item {\bf Theory assumptions and proofs}
    \item[] Question: For each theoretical result, does the paper provide the full set of assumptions and a complete (and correct) proof?
    \item[] Answer: \answerYes{} 
    \item[] Justification: We discuss the assumptions and proofs in the Appendices~\ref{appendix:scaling_analysis}, \ref{appendix:exponential_decay}, and \ref{appendix:optimal_depth}.

    \item {\bf Experimental result reproducibility}
    \item[] Question: Does the paper fully disclose all the information needed to reproduce the main experimental results of the paper to the extent that it affects the main claims and/or conclusions of the paper (regardless of whether the code and data are provided or not)?
    \item[] Answer: \answerYes{} 
    \item[] Justification: We provide detailed information to reproduce our experiments in Appendix~\ref{appendix:experimental_details}.

\item {\bf Open access to data and code}
    \item[] Question: Does the paper provide open access to the data and code, with sufficient instructions to faithfully reproduce the main experimental results, as described in supplemental material?
    \item[] Answer: \answerYes{} 
    \item[] Justification: We provide the reproduction code in a GitHub repository in the last of Section~\ref{section:introduction}.

\item {\bf Experimental setting/details}
    \item[] Question: Does the paper specify all the training and test details (e.g., data splits, hyperparameters, how they were chosen, type of optimizer) necessary to understand the results?
    \item[] Answer: \answerYes{} 
    \item[] Justification: We discuss the experimental setting in the Appendix~\ref{appendix:experimental_details}.

\item {\bf Experiment statistical significance}
    \item[] Question: Does the paper report error bars suitably and correctly defined or other appropriate information about the statistical significance of the experiments?
    \item[] Answer: \answerYes{} 
    \item[] Justification: We report error bars in the caption of the each figure.

\item {\bf Experiments compute resources}
    \item[] Question: For each experiment, does the paper provide sufficient information on the computer resources (type of compute workers, memory, time of execution) needed to reproduce the experiments?
    \item[] Answer: \answerYes{} 
    \item[] Justification: We report the compute resources in the Appendix~\ref{appendix:experimental_details}.
    
\item {\bf Code of ethics}
    \item[] Question: Does the research conducted in the paper conform, in every respect, with the NeurIPS Code of Ethics \url{https://neurips.cc/public/EthicsGuidelines}?
    \item[] Answer: \answerYes{} 
    \item[] Justification: We conform to the NeurIPS Code of Ethics.

\item {\bf Broader impacts}
    \item[] Question: Does the paper discuss both potential positive societal impacts and negative societal impacts of the work performed?
    \item[] Answer: \answerYes{} 
    \item[] Justification: We provide a discussion of potential impacts in a Impact Statement subsection
    at the end of Section~\ref{section:introduction}.
    
\item {\bf Safeguards}
    \item[] Question: Does the paper describe safeguards that have been put in place for responsible release of data or models that have a high risk for misuse (e.g., pre-trained language models, image generators, or scraped datasets)?
    \item[] Answer: \answerNA{} 
    \item[] Justification:  We do not release any datasets or models that pose a
    significant risk of misuse, so no additional safeguards are required.

\item {\bf Licenses for existing assets}
    \item[] Question: Are the creators or original owners of assets (e.g., code, data, models), used in the paper, properly credited and are the license and terms of use explicitly mentioned and properly respected?
    \item[] Answer: \answerYes{} 
    \item[] Justification: We have credited all third-party assets and fully complied with their license
    terms.

\item {\bf New assets}
    \item[] Question: Are new assets introduced in the paper well documented and is the documentation provided alongside the assets?
    \item[] Answer: \answerYes{} 
    \item[] Justification: We provide the reproduction code in a GitHub repository including
    documentation.

\item {\bf Crowdsourcing and research with human subjects}
    \item[] Question: For crowdsourcing experiments and research with human subjects, does the paper include the full text of instructions given to participants and screenshots, if applicable, as well as details about compensation (if any)? 
    \item[] Answer: \answerNA{} 
    \item[] Justification: We do not involve crowdsourcing nor research with human subjects.

\item {\bf Institutional review board (IRB) approvals or equivalent for research with human subjects}
    \item[] Question: Does the paper describe potential risks incurred by study participants, whether such risks were disclosed to the subjects, and whether Institutional Review Board (IRB) approvals (or an equivalent approval/review based on the requirements of your country or institution) were obtained?
    \item[] Answer: \answerNA{} 
    \item[] Justification: We do not involve crowdsourcing nor research with human subjects.

\item {\bf Declaration of LLM usage}
    \item[] Question: Does the paper describe the usage of LLMs if it is an important, original, or non-standard component of the core methods in this research? Note that if the LLM is used only for writing, editing, or formatting purposes and does \emph{not} impact the core methodology, scientific rigor, or originality of the research, declaration is not required.
    \item[] Answer: \answerNA{} 
    \item[] Justification: The core methodology of this research does not involve LLMs as an important, original, or non-standard component.

\end{enumerate}

\end{document}